\documentclass[numbers,webpdf,imaiai]{GT}%

\graphicspath{{Fig/}}  

\usepackage{cases}
\usepackage{bm}
\usepackage{subcaption}
\usepackage{booktabs}
\usepackage{makecell}  
\usepackage{titlesec}  
\usepackage{float}
\floatstyle{plaintop}
\restylefloat{table}
\usepackage{tabularx}

\usepackage{tikz}
\usetikzlibrary{arrows.meta, positioning}

\theoremstyle{thmstyletwo}%
\newtheorem{theorem}{Theorem} %
\newtheorem{lemma}{Lemma} %
\newtheorem{remark}{Remark}
\newtheorem{defi}{Definition}

\numberwithin{equation}{section}

\newcommand{\rank}{\operatorname{rank}}

\usepackage{mathptmx}
\DeclareMathAlphabet{\mathcal}{OMS}{cmsy}{m}{n}
\DeclareSymbolFont{largesymbols}{OMX}{cmex}{m}{n}

\newcommand{\B}{\mathcal{B}}
\newcommand{\D}{\mathcal{D}}
\newcommand{\E}{\mathcal{E}}

\newcommand{\G}{\mathcal{G}}
\newcommand{\I}{\mathcal{I}}

\newcommand{\mL}{\mathcal{L}}
\newcommand{\mO}{\mathcal{O}}
\newcommand{\U}{\mathcal{U}}
\newcommand{\V}{\mathcal{V}}

\newcommand{\mS}{\mathcal{S}}

\newcommand{\W}{\mathcal{W}}
\newcommand{\X}{\mathcal{X}}
\newcommand{\Y}{\mathcal{Y}}
\newcommand{\Z}{\mathcal{Z}}

\newcommand{\R}{\mathbb{R}}

\newcommand{\bB}{\mathbb{B}}
\newcommand{\bD}{\mathbb{D}}
\newcommand{\bC}{\mathbb{C}}

\newcommand{\mC}{\mathcal{C}}
\newcommand{\mH}{\mathcal{H}}

\usepackage[algo2e,ruled,linesnumbered]{algorithm2e}

\begin{document}

\DOI{DOI HERE}
\copyrightyear{2024}
\vol{00}
\pubyear{2021}
\access{Advance Access Publication Date: Day Month Year}
\appnotes{Paper}
\firstpage{1}


\title[Hyperspectral Anomaly Detection]{Spectral-Spatial Extraction through Layered Tensor Decomposition for Hyperspectral Anomaly Detection}

\author{Quan Yu\ORCID{0000-0002-8051-7477}
\address{\orgdiv{School of Mathematics}, \orgname{Hunan University}, \orgaddress{\postcode{410082}, \state{Changsha}, \country{China}}}}
\author{Yu-Hong Dai\ORCID{0000-0002-6932-9512}
\address{\orgdiv{State Key Laboratory of Mathematical Sciences, Academy of Mathematics and Systems Science}, \orgname{Chinese Academy of Sciences}, \orgaddress{\postcode{100190}, \state{Beijing}, \country{China}}}
\address{\orgdiv{ School of Mathematical Sciences}, \orgname{ University of Chinese Academy of Sciences}, \orgaddress{\postcode{100049}, \state{Beijing}, \country{China}}}}
\author{Minru Bai*\ORCID{0000-0002-9960-6138}
	\address{\orgdiv{School of Mathematics}, \orgname{Hunan University}, \orgaddress{\postcode{410082}, \state{Changsha}, \country{China}}}}

\authormark{Author Name et al.}

\corresp[*]{Corresponding author: \href{email:email-id.com}{minru-bai@hnu.edu.cn}}

\received{Date}{0}{Year}
\revised{Date}{0}{Year}
\accepted{Date}{0}{Year}

\abstract{Low rank tensor representation (LRTR) methods are very useful for hyperspectral anomaly detection (HAD). To overcome the limitations that they often overlook spectral anomaly and rely on large-scale matrix singular value decomposition, we first apply non-negative matrix factorization (NMF) to alleviate spectral dimensionality redundancy and extract spectral anomaly and then employ LRTR to extract spatial anomaly while mitigating spatial redundancy, yielding a highly efffcient  layered tensor decomposition (LTD) framework for HAD. An iterative algorithm based on proximal alternating minimization is developed to solve the proposed LTD model, with convergence guarantees provided. Moreover, we introduce a rank reduction strategy with validation mechanism that adaptively reduces data size while preventing excessive reduction. Theoretically, we rigorously establish the equivalence between the tensor tubal rank and tensor group sparsity regularization (TGSR) and, under mild conditions, demonstrate that the relaxed formulation of TGSR shares the same global minimizers and optimal values as its original counterpart. Experimental results on the Airport-Beach-Urban and MVTec datasets demonstrate that our approach outperforms state-of-the-art methods in the HAD task.}

 

\maketitle

\section{Introduction}
Hyperspectral image (HSI) integrates 2D spatial information with detailed spectral data, capturing the fundamental characteristics of objects. This integration offers substantial advantages in applications such as denoising \cite{TLZ23}, classification \cite{HGY21}, image fusion \cite{HYXS20,BCE18}, and anomaly detection \cite{YB24}. Recently, hyperspectral anomaly detection (HAD) has garnered significant attention due to its critical role in public safety and defense applications. The objective of HAD is to detect and isolate anomalous objects from their background. To address this problem, a variety of approaches have been developed, including statistic based methods, deep learning based methods and low rank representation based methods.


One of the most commonly used statistical approaches for anomaly detection is the Reed-Xiaoli (RX) detector \cite{RY90}, which models background pixels using a multivariate Gaussian distribution. The RX detector calculates the mean and covariance matrix of the background and then determines the Mahalanobis distance between a test pixel and the background mean to identify anomalies based on a predefined threshold. Although the RX detector is simple and effective, its exclusive reliance on the covariance matrix limits its ability to capture finer relationships within HSI data. To address this limitation, several variants of the RX detector have been introduced, such as kernel RX \cite{KN05,ZKAE16}, subspace RX \cite{Sch04}, and local RX \cite{MGGP13}. However, despite these enhancements, these methods still rely on manually crafted distribution models, which may struggle to capture the complexity and variability of real-world backgrounds.

Recent studies have highlighted deep learning based methods for effective HAD, which utilize deep networks to mine and interpret higher order information contained in HSI. These methods can be categorized into unsupervised, supervised, and self-supervised learning. Unsupervised learning methods, such as autoencoder (AE) detector \cite{BCKA15} and robust graph AE (RGAE) detector \cite{FMM22}, learn background representations or anomalies by analyzing the inherent structure of the data without requiring labeled data. Supervised learning methods, such as convolutional neural network based detection (CNND) \cite{LWD17}, use labeled data to classify each pixel as background or anomalous. Self-supervised learning methods, like blind-spot self-supervised learning network (BS$^3$LNet) \cite{GWZ23} and pixel-shuffle downsampling blind-spot reconstruction network (PDBSNet) \cite{WZG23}, generate their own labels to distinguish between similar and dissimilar data pairs, thereby enhancing anomaly detection. However, these methods often demand extensive training times and involve intricate tuning of network architectures and training parameters.

Beyond the statistical and deep learning methods mentioned above, recent research has focused on the low rank representation (LRR) of the background and the sparsity of anomaly. With the maturation of matrix theory, many researchers first convert HSI into a two-dimensional matrix and then use LRR for anomaly detection. A widely adopted approach is robust principal component analysis (RPCA), which decomposes HSI into a low rank matrix and a sparse anomaly matrix \cite{CYKC13}. Subsequently, a series of improved methods have been proposed, such as those addressing noise \cite{SLL14} and incorporating Mahalanobis distance metrics \cite{ZDZW16}. The above methods struggle to distinguish weak anomalies. To address this limitation, a method based on LRR that uses a fixed background dictionary to construct multiple subspaces was proposed \cite{ZGZ21}. Additionally, \cite{MZH22} enhanced the robustness and detection performance by employing a learned background dictionary. 
HSI can be treated as a third order tensor with one spectral and two spatial dimensions. Matrix based methods often overlook this multi-dimensional structure, potentially degrading detection performance. Consequently, methods based on low rank tensor representation (LRTR) of the background and anomaly sparsity have rapidly developed for HAD. Chen et al. \cite{CYW18} extended matrix RPCA to tensors, proposing tensor principal component analysis (TPCA) to effectively distinguish between the background and targets. Similarly, many methods construct tensor background dictionary to enhance representation quality \cite{WWH23,YB24}. Additionally, some regularizations, such as total variation regularized \cite{FCY23}, are considered to improve anomaly detection quality.

Until now, the LRTR based HAD method has faced two main drawbacks. First, due to the redundancy in the spectral dimension of HSI, some algorithms perform orthogonal transformations such as TPCA and maximum noise fraction (MNF) transformation \cite{WWH23,HWL23} to reduce the spectral dimension and computational load, followed by anomaly detection on the preprocessed data. This approach results in the loss of spectral domain anomaly and divides the original data processing into two stages, significantly compromising effectiveness. Second, extracting spatial domain anomaly often involves the singular value decomposition (SVD) of large matrices, leading to low computational efficiency.

To address the aforementioned limitations, we propose an HAD model based on layered tensor decomposition, incorporating non-negative matrix factorization (NMF) and low rank tensor representation (LRTR), named LTD. LTD efficiently detects both spectral and spatial anomaly within a unified framework, as illustrated in Figure \ref{fig:flow}. First, we use NMF to reduce the spectral dimension and computational load. Notably, this process not only yields the coefficient tensor but also identifies spectral anomaly, which are often overlooked by other models. Then, we apply LRTR to capture the spatial low rank nature of the coefficient tensor, during which spatial anomaly is detected. Finally, we integrate the spectral and spatial anomalies to achieve a superior anomaly detection result. In the LRTR process, we demonstrate the equivalence between tensor tubal rank and tensor group sparsity regularization (TGSR). Additionally, we show that TGSR and its relaxed formulation share the same global minimizers and optimal values under certain mild conditions. Leveraging these insights, the relaxed formulation of TGSR is employed to characterize the rank of the coefficient tensor. This regularization decomposes the coefficient tensor into two smaller tensors, whose sizes are directly related to the rank of the coefficient tensor. Building on the aforementioned decompositions and the established theory related to group sparse support sets, we developed a rank reduction strategy with validation mechanism.
This strategy adaptively and dynamically reduces the size of these two tensors, thereby improving computational efficiency and eliminating the reliance on large-scale matrix SVD. Additionally, it ensures the reduction process is appropriate, preventing excessively low ranks that could compromise computational performance.  In summary, our main contributions include:

\begin{itemize}
	\item[(1)] We develop a layered tensor decomposition (LTD) framework that synergistically integrates NMF and LRTR to reduce the spectral and spatial dimensions of HSI data, respectively. This framework decomposes large-scale HSI data into the product of multiple smaller tensors, thereby enabling a holistic decomposition that substantially enhances algorithmic performance.
	
	\item[(2)] We propose a unified framework that harnesses LTD to simultaneously capture both spectral and spatial domain anomalies. To our knowledge, this is the first LRR based approach that integrates these anomaly types within a single framework, yielding exceptional detection performance.
	
	\item[(3)] We demonstrate the equivalence between tensor tubal rank and tensor group sparsity regularization (TGSR). Additionally, we show that TGSR and its relaxed formulation share the same global minimizers and optimal values under certain mild conditions. Furthermore, we develop theories related to group sparse support sets. 

	\item[(4)] We propose a rank reduction strategy with validation mechanism that adaptively and dynamically reduces tensor size. This approach ensures the reduction process is appropriate, correcting any erroneous reductions to prevent excessively low ranks that could compromise performance.
	
	\item[(5)] We propose a proximal alternating minimization algorithm to solve the proposed model and provide a comprehensive convergence analysis.
\end{itemize}
Experimental results on Airport-Beach-Urban and MVTec databases demonstrate the advantages of our method.

\section{Preliminaries}
In this section, we present the notations and preliminaries utilized throughout this article. The sets of real numbers and complex numbers are denoted by $\mathbb{R}$ and $\mathbb{C}$, respectively. Meanwhile, some basic notations are summarized in Table \ref{tab:Notation}.
\begin{table}[htbp]
	\centering
	\caption{Notations.}
	\begin{tabularx}{\textwidth}{p{3cm} X}
		\hline Notation           &       Description \\
		\hline   $[n]$            &  the set   $\{1,2, \ldots, n\}$ \\
		$x$, $\bm{x}$, $X$, $\mathcal{X}$ & scalar, vector, matrix, tensor \\
		$\X(:,:,k)$ or $X^{(k)}$ & $k$-th frontal slice of a third order tensor $\X$ \\
		$\mathcal{X}\left(i, j,k\right)$ or $\mathcal{X}_{ijk}$ & the $(i, j, k)$-th index value of a third order tensor $\X$ \\
		$\|\X\|$ & the Frobenius norm of $\X$, which is defined as $\sqrt{\sum_{i j k}|\X_{i j k}|^2}$ \\
		$\left\langle \X, \Y \right\rangle $ & the inner product of two same-sized third order tensors $\mathcal{X}$ and $\mathcal{Y}$, which is defined as $\sum_{i j k}\X_{i j k}\Y_{i j k}$ \\
		$X_{(3)}$ & the mode-$3$ unfolding of a third order tensor $\mathcal{X}$ (Definition \ref{def:mode}) \\
		$\times_3$ & the mode-$3$ tensor-matrix product (Definition \ref{def:tmp})\\
		$\X^T$ & the conjugate transpose of a third order tensor $\mathcal{X}$ (Definition \ref{def:ct}) \\
		$\X*\Y$ & tensor product between two third order tensors $\mathcal{X}$ and $\mathcal{Y}$ (Definition \ref{def:t-product}) \\
		$\left\|X\right\|_2$ & the spectral norm of a matrix $X$, which is the largest singular value \\
		$\delta_{\mathbb X}\left(\X\right)$  & the indicator function of a subset $\mathbb X$, which is $0$ if $\X\in {\mathbb X}$ and $\infty$ if $\X\notin {\mathbb X}$ \\
		$\X \ge 0$ & the tensor $\X$ with all elements being non-negative \\
		$\I$  & the identity tensor (Definition \ref{def:it}) \\
		 $\Gamma(\X)$   & the group sparse support set of $\X$, which is defined as  $\{ j \mid\|\X(:,j,:)\| \neq 0\}$\\
		\hline
	\end{tabularx}
	\label{tab:Notation}%
\end{table}%

 Next, we introduce some essential definitions related to tensor operations. Consider a third order tensor $\X \in \R^{n_1\times n_2 \times n_3}$. The Discrete Fourier transform (DFT) along its third mode, denoted as $\bar{\X}$, is computed in MATLAB using $\bar{\mathcal{X}}=\operatorname{fft}(\mathcal{X},[~], 3)$. To reverse this transformation, the inverse DFT is applied using $\mathcal{X}=\operatorname{ifft}(\bar{\mathcal{X}}, [~], 3)$.
By defining the block circulant matrix \( \operatorname{bcirc}(\mathcal{X}) \) as
$$
\operatorname{bcirc}(\mathcal{X})=\left[\begin{array}{cccc}
	X^{(1)} & X^{\left(n_3\right)} & \cdots & X^{(2)} \\
	X^{(2)} & X^{(1)} & \cdots & X^{(3)} \\
	\vdots & \vdots & \ddots & \vdots \\
	X^{\left(n_3\right)} & X^{\left(n_3-1\right)} & \cdots & X^{(1)}
\end{array}\right],
$$
we can introduce the t-product between two third order tensors. This, in turn, allows us to propose a new tensor decomposition framework known as T-SVD.

\begin{defi}(t-product \cite{KM11}). \label{def:t-product}
	The t-product of $\mathcal{X} \in \mathbb{R}^{n_1 \times n_2 \times n_3}$ and $\mathcal{Y} \in \mathbb{R}^{n_2 \times n_4 \times n_3}$ is defined as:
	$$
	\mathcal{X} * \mathcal{Y} = \operatorname{fold}(\operatorname{bcirc}(\mathcal{X}) \cdot \operatorname{unfold}(\mathcal{Y})) \in \mathbb{R}^{n_1 \times n_4 \times n_3},
	$$
	where $\operatorname{unfold}(\mathcal{Y}) = [Y^{(1)}; Y^{(2)}; \ldots; Y^{(n_3)}] \in \mathbb{R}^{n_2 n_3 \times n_4}$. The inverse operation, $\operatorname{fold}$, is defined such that $\operatorname{fold}(\operatorname{unfold}(\mathcal{Y})) = \mathcal{Y}$.
\end{defi}

\begin{theorem}(T-SVD \cite{KM11}).
	 Let $\X$ be an $n_1\times n_2 \times n_3$ real-valued tensor. Then it can be factorized as
	 $ \X = \U * \mS * \V^T $,
	 where $\mathcal{U} \in \mathbb{R}^{n_1 \times n_1 \times n_3}$ and $\mathcal{V} \in \mathbb{R}^{n_2 \times n_2 \times n_3}$ are orthogonal tensors, $\mathcal{S} \in \mathbb{R}^{n_1 \times n_2 \times n_3}$ is an f-diagonal tensor.
\end{theorem}

\begin{defi}(tensor tubal rank \cite{KBHH13}).
	Suppose that $\X \in \mathbb{R}^{n_1 \times n_2 \times n_3}$ with T-SVD $\X = \U * \mS * \V^T$. The tensor tubal rank is defined as $\rank_t\left(\X\right)=\#\{i: \mathcal{S}(i, i,:) \neq 0\}$.
\end{defi}
Some related concepts, including the $f$-diagonal tensor, the conjugate transpose, among others, are elucidated in Appendix \ref{app:pre}.

\begin{defi}(mode-3 unfolding \cite{KB09}).\label{def:mode}
Given a third order tensor $\mathcal{X} \in \mathbb{R}^{n_1 \times n_2 \times n_3}$, its mode-3 unfolding $X_{(3)}$ is an $n_3 \times n_1 n_2$ matrix, which satisfies
$
X_{(3)}\left(k, (i-1)n_2 + j\right) = \mathcal{X}\left(i, j, k\right)
$
for $i \in [n_1]$, $j \in [n_2]$, and $k \in [n_3]$.
\end{defi}

\begin{defi}(mode-3 tensor-matrix product \cite{KB09}).\label{def:tmp}
	Given a third order tensor $\mathcal{X} \in \mathbb{R}^{n_1 \times n_2 \times n_3}$ and a matrix $M \in \mathbb{R}^{p \times n_3}$, the mode-3 product of $\mathcal{X}$ and $M$, denoted as $\mathcal{X} \times_3 M \in \R^{n_1 \times n_2 \times p}$, is defined by $(\mathcal{X} \times_3 M)_{ijk} = \sum_{l=1}^{n_3} \mathcal{X}_{ijl} M_{kl}$,
	where $i \in [n_1]$, $j \in [n_2]$, and $k \in [p]$.
\end{defi}

\section{A layered tensor decomposition method}
In this section, we propose a layered tensor decomposition (LTD) method for hyperspectral anomaly detection (HAD), which incorporates non-negative matrix factorization (NMF) and low rank tensor representation (LRTR). First, the spectral anomaly map is obtained using NMF. Next, the spatial anomaly map is extracted through LRTR. Finally, the two complementary detection maps are adaptively fused to highlight anomaly. Figure \ref{fig:flow} shows the main framework of proposed method.
\begin{figure}[htbp]
	\centering
	\includegraphics[width=0.8\linewidth]{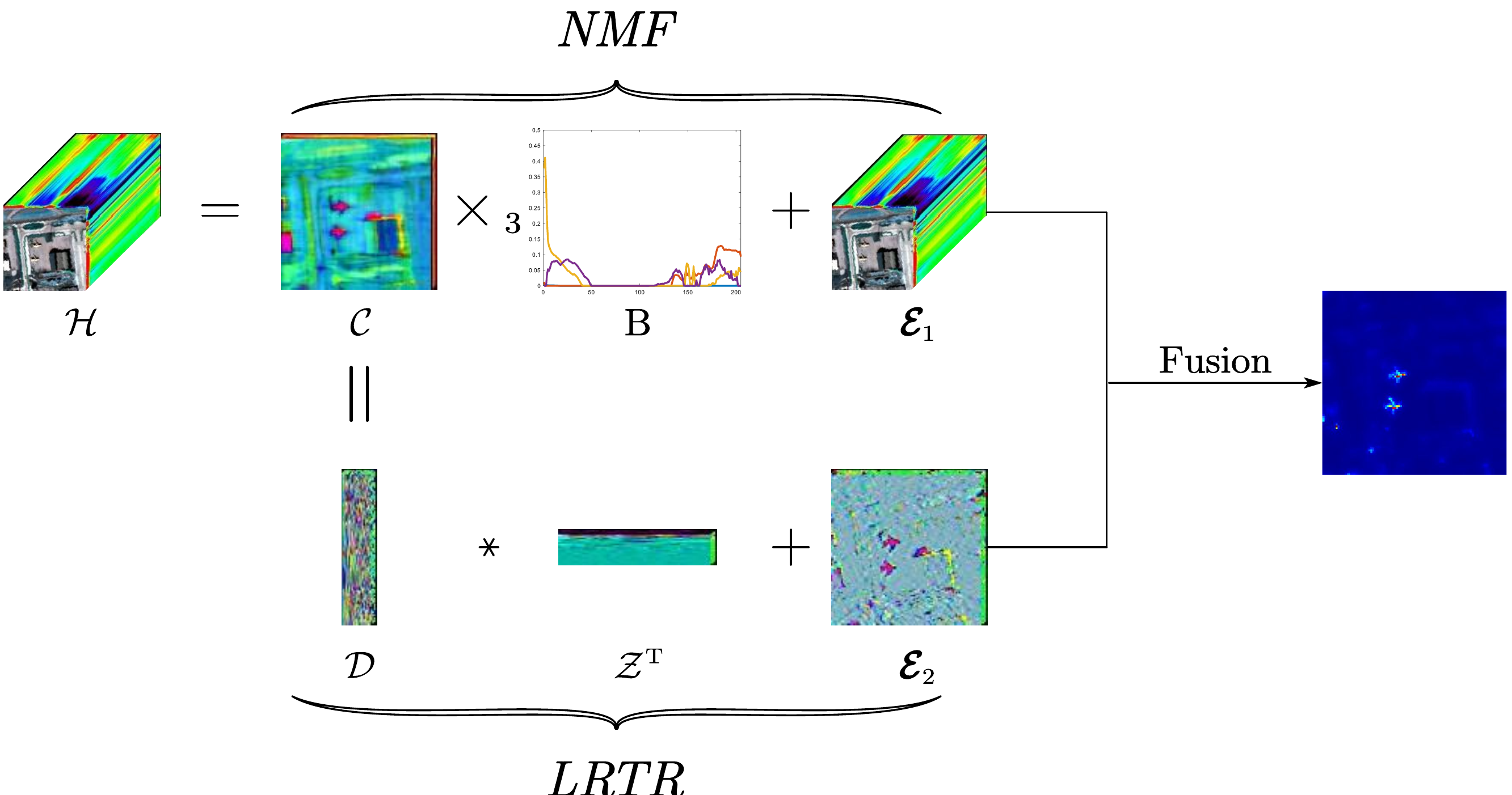}
	\caption{Flow chart of LTD.}
	\label{fig:flow}
\end{figure}

\subsection{Spectral extraction based NMF}\label{sec:NMF}
Hyperspectral image (HSI) data often contains a significant amount of redundancy due to the hundreds of highly correlated spectral bands, resulting in undesirable statistical and geometrical properties \cite{JL98,KMB07}. To overcome these drawbacks and reduce computational costs, dimensionality reduction (or band selection) in HSI data is typically applied first, followed by postprocessing tasks \cite{HWL23,WWH23}. Motivated by NMF, which aims to uncover latent and meaningful structures within the data, we decompose HSI data $\mH\in \R^{n_1\times n_2\times n_3}$ into the mode-$3$ multiplication of the coefficient tensor $\mC \in \R^{n_1\times n_2\times b}$ by the basis matrix $B\in \R^{n_3\times b}$:
\begin{equation*}\label{NMF}
	\mH=\mC\times_3 B, \quad \mC \ge 0,~ B \ge 0,
\end{equation*}
which is equivalent to the NMF: $H_{(3)}=BC_{(3)}$ \cite{KB09}. 

With the widespread application of NMF, numerous variants have been developed, including G-orthogonal NMF \cite{DLPP06}, Semi-NMF \cite{DLJ10} and Orth-NTF \cite{LGW23}. By imposing column unit constraints on each tube of the tensor $\mC$, we derive the one-side column unit NMF:
\begin{equation*}\label{NMF}
	\mH=\mC\times_3 B, \quad \mC \ge 0,~ B \ge 0,~\left\|\mC(i,j,:)\right\|=1,~\forall i\in[n_1],~ j\in [n_2].
\end{equation*}
Taking into account the one-side column unit NMF, we relax the nonnegative constraints on $\mC$. Moreover, considering that the above decomposition equality does not hold perfectly due to the presence of sparse anomaly in the spectral domain and Gaussian noise in HSI data, we propose the following optimization model:
\begin{equation}\label{Q:NMF}
	\begin{array}{cl}
		\min\limits_{\mC, B, \E_1}&  \frac{\lambda_1}{2}\left\|B\right\|^2+ \lambda_2\left\|\E_1\right\|_{11\phi}+\frac{\lambda_3}{2}\left\|\mH-\mC\times_3 B-\E_1\right\|^2   \\
		\mbox{\rm s.t.}& B \ge 0,~\left\|\mC(i,j,:)\right\|=1,~\forall i\in[n_1],~ j\in [n_2],
	\end{array}
\end{equation}
where $\|\E_1\|_{11\phi}: = \sum_{i=1}^{n_1}\sum_{j=1}^{n_2}\phi(\|\E_1(i,j,:)\|)$ with function $ \phi(\cdot):\R_+ \to \R_+ $. The parameter $\lambda_1>0$ constrains the size of the entries in $B$ to prevent excessively large values, which could induce instability in the results.
\begin{remark}
	In model \ref{Q:NMF}, we denote sparse anomaly in the spectral domain as $\E_1$ and employ the Frobenius norm to mitigate Gaussian noise in HSI data $\mH$.
\end{remark}
\begin{remark}
	In \cite{YB24}, it is demonstrated that $ \phi^{CapL1}(\cdot):=\min\{\cdot,1\} $ performs well in anomaly detection, both in terms of computational efficiency and effectiveness. Therefore, in this paper, we set $ \phi = \phi^{CapL1} $.
\end{remark}


\subsection{Spatial extraction based LRTR}
By applying NMF as described in Subsection~\ref{sec:NMF}, we relax the low rank constraint in the spectral dimension of HSI data $\mathcal{H}$, thereby generating the coefficient tensor $\mathcal{C}$. However, $\mC$ retains the low rank characteristics of the spatial dimensions of $\mH$. According to \cite[Lemma 6]{YZ22}, these low rank properties can be captured by the tensor tubal rank of $\mC$. Therefore, we propose the following optimization model:
\begin{equation}\label{Q:lr}
	\begin{array}{cl}
		\min\limits_{\mL, \E_2}& \lambda_4\rank_t\left(\mL\right) +\lambda_5\left\|\E_2\right\|_{11\phi}  \\
		\mbox{\rm s.t.}& \mC = \mL + \E_2,
	\end{array}
\end{equation}
where $ \E_2 $ represents sparse anomaly in the spatial domain.

Driven by the group sparsity structure evident in the coding of low rank matrices \cite{FDCU19,TQP22} and low rank tensors \cite{FDY22}, we represent $\mL$ within a subspace $\D$ and limit its rank by imposing group sparsity on its coefficients $\Z$. To justify the validity of this representation, we investigate the relationship between tensor tubal rank and group sparsity regularization.
\begin{theorem}\label{Thm:rank}
	For any tensor $\mL \in \mathbb{R}^{n_1 \times n_2 \times b}$, one has
	$$\rank_t\left(\mL\right)=\min \left\{\left\|\Z\right\|_{F,0}: \mL=\D*\Z^T, \D^T*\D=\I, \D\in\R^{n_1 \times r \times b}, \Z\in\R^{n_2 \times r \times b}\right\}, $$
	with $\left\|\Z\right\|_{F,0} = \sum_{j=1}^r\left\|\Z\left(:,j,:\right) \right\|^0$,
     for any $r$ that satisfies $ \rank_t\left(\mL\right) \le r \le \min\{n_1, n_2\} $.
\end{theorem}
\begin{proof}
	See Appendix \ref{app:thm_rank} for the proof.
\end{proof}

\begin{theorem}\label{Thm:cap}
	There exists a $\bar{\nu} > 0$. For any function $\psi$ that satisfies $ \psi(0)=0 $, $x/\nu \leq \psi(x) < 1$ for $x \in (0, \nu)$ and $\psi(x) = 1$ for $x \ge \nu$, where $0 < \nu < \bar{\nu}$, the following two problems share the same global minimizers and optimal values:
	$$
	\left\{\begin{array}{l}
		(P_0) ~ \min \left\{\left\|\Z\right\|_{F,0}: \mL=\D*\Z^T, \D^T*\D=\I\right\}; \\
		(P_\psi)~  \min \left\{\left\|\Z\right\|_{F,1}^\psi: \mL=\D*\Z^T, \D^T*\D=\I\right\}.
	\end{array}\right.
     $$
\end{theorem}
\begin{proof}
	See Appendix \ref{app:thm_cap} for the proof.
\end{proof}

\begin{remark}
The relationships between tensor tubal rank and group sparsity regularization are shown in Figure \ref{fig:relationship}.
	\begin{figure}[!htbp]
		\centering
		\includegraphics[width=1\linewidth]{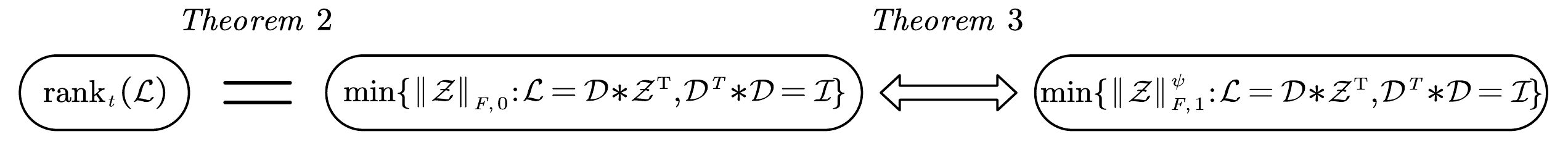}
		\caption{Relationships between tensor tubal rank and group sparsity regularization.}
		\label{fig:relationship}
	\end{figure}
Theorem \ref{Thm:rank} establishes that the tensor tubal rank $\rank_t\left({\mathcal L}\right)$ is equivalent to the optimal values of problem $(P_0)$.
However, due to the non-convex and non-smooth nature of $\|{\mathcal Z}\|_{F,0}$, algorithms based on it exhibit unstable numerical performance. Consequently, we demonstrate in Theorem \ref{Thm:cap} that problem $(P_0)$ and problem $(P_\psi)$ have the same global minimizers and optimal values. 

Optimizing $\|{\mathcal Z}\|_{F,1}^\psi$ is significantly more efficient than optimizing tensor nuclear norm, as it eliminates the need for large-scale SVD and the dimensionality of $\mathcal{Z}$ is considerably smaller than that of $\mathcal{L}$. Furthermore, in algorithms based on $\|{\mathcal Z}\|_{F,1}^\psi$, the size of $r$ can adaptively decrease until it approximates $\rank_t({\mathcal L})$, thereby obviating the need for additional techniques to adjust $r$. This is in contrast to tensor decomposition methods discussed in \cite{ZLLZ18,YZ22,YZH23}, which requires adjusting $r$.
By leveraging the ability to adaptively adjust $r$ and the fact that $\rank_t({\mathcal L})$ is much smaller, models based on $\|{\mathcal Z}\|_{F,1}^\psi$ are not only highly efficient but also exhibit superior performance.
\end{remark}

Motivated by Theorems \ref{Thm:rank} and \ref{Thm:cap}, we transform problem \eqref{Q:lr} into the following formulation:
\begin{equation}\label{Q:lrd}
	\begin{array}{cl}
		\min\limits_{\D, \Z, \E_2}& \lambda_4\left\|\Z\right\|_{F,1}^\psi +\lambda_5\left\|\E_2\right\|_{11\phi}  \\
		\mbox{\rm s.t.}& \mC = \D*\Z^T + \E_2,~\D^T*\D=\I,
	\end{array}
\end{equation}
where $\E_2\in \R^{n_1\times n_2 \times b}$, $\D\in \R^{n_1\times r \times b}$, $\Z\in \R^{n_2\times r \times b}$.

\begin{remark}
	In this paper, we choose $\psi$ as the CapLp function, specifically $\psi(x) = \min\{x^p/\nu^p, 1\}$ for some $0<p < 1$ and $\nu>0$.
\end{remark}

\subsection{The proposed LTD}
Taking into account the two subsections above, the proposed LTD can thus be formulated as:
\begin{equation}\label{Q:SU-LRTR}
	\begin{array}{cl}
		\min\limits_{\mC, B, \E_1, \D, \Z, \E_2}&  \frac{\lambda_1}{2}\left\|B\right\|^2+\lambda_2 \left\|\E_1\right\|_{11\phi}+\frac{\lambda_3}{2}\left\|\mH-\mC\times_3 B-\E_1\right\|^2+\lambda_4\left\|\Z\right\|_{F,1}^\psi +\lambda_5\left\|\E_2\right\|_{11\phi}   \\
		\mbox{\rm s.t.}& \mC = \D*\Z^T + \E_2,~B\in\bB,~\mC\in \bC,~\D\in\bD,
	\end{array}
\end{equation}
where $\bB = \left\lbrace B\mid B\ge0 \right\rbrace $, $\bC = \left\lbrace \mC\mid \left\|\mC(i,j,:)\right\|=1,~\forall i\in[n_1],~ j\in [n_2] \right\rbrace $, and $\bD = \left\lbrace \D\mid \D^T*\D=\I \right\rbrace $.

\begin{remark}
	From \eqref{Q:SU-LRTR}, we can see that the spectral anomaly $\E_1$ and the spatial anomaly $\E_2$ can be simultaneously learned in a unified framework, thereby achieving a comprehensive anomaly detection result by integrating these anomalies.
\end{remark}
\begin{remark}
	In Theorems \ref{Thm:rank} and \ref{Thm:cap}, we establish the relationship between tensor tubal rank and group sparsity regularization. Leveraging this, we employ group sparsity regularization in LRTR to characterize the spatial low-rankness of tensor $\mC$. Although SVD computation is required for $\D \in \R^{n_1\times r  \times b}$, it is limited to small-scale matrices, as our algorithm adaptively reduces the size of $\D$ (see Algorithm \ref{Alg:Rank}). Consequently, the final $r$ is much smaller than $\min\{n_1, n_2\}$ (see Section \ref{sec:ra}). Thus, the proposed group sparsity regularization is markedly more efficient than nuclear function regularization.
\end{remark}

\subsection{Spectral-spatial fusion}

To effectively highlight anomalous targets, we design an adaptive fusion architecture that integrates two complementary detection maps, as illustrated in Figure \ref{fig:fusion}. First, we compute the spectral anomaly detection map, given by $T_1=\sqrt{\sum_{k=1}^{n_3}|\mathcal{E}_1^*(:, :, k)|^2}$, and the spatial anomaly detection map, given by $T_2=\sqrt{\sum_{k=1}^{b}|\mathcal{E}_2^*(:, :, k)|^2}$. Then, we obtain the final detection result as follows:
\begin{equation}\label{opt:T}
	T = \operatorname{IGF}\left(T_1 \circ T_2\right)\quad \text{or} \quad  T = \operatorname{IGF}(\operatorname{IGF}(\operatorname{IGF}\left(T_1 \circ T_2\right), T_1), T_2)
\end{equation}
where $ \circ $ represents the Hadamard product operator, and $ \operatorname{IGF} $ is a guided image filter \cite{HST10}.
\begin{figure}[!htbp]
	\centering
	\includegraphics[width=0.7\linewidth]{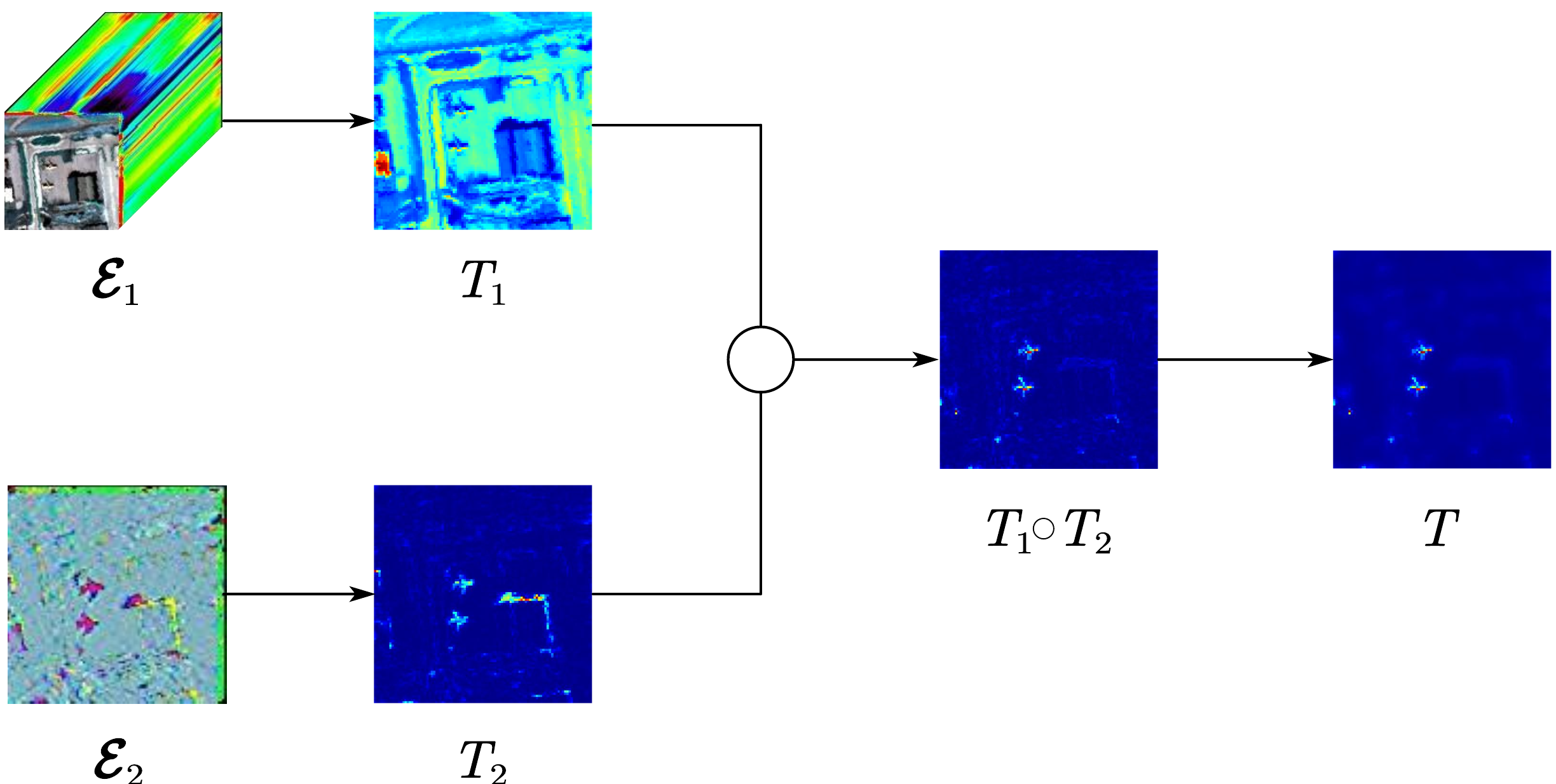}
	\caption{Flow chart of spectral-spatial fusion.}
	\label{fig:fusion}
\end{figure}

\begin{remark}
	From Figure \ref{fig:fusion}, it can be seen that $T_1$ and $T_2$ have complementary characteristics. Specifically, the noise points and anomalous targets are complementary in their positions in $T_1$ and $T_2$. Therefore, we take the Hadamard product of $T_1$ and $T_2$. Although $T_1 \circ T_2$ still contains slight noise, considering that the guided image filter can effectively remove noise and enhance details, we apply the guided image filter to $T_1 \circ T_2$, resulting in an image with reduced noise and clearer anomalous targets.
\end{remark}

\subsection{Optimization of LTD}
By employing the half-quadratic splitting technique \cite{KF09}, we convert the constrained problem \eqref{Q:SU-LRTR} into the subsequent unconstrained formulation:
\begin{equation}\label{Q-un:SU-LRTR}
	\begin{array}{cl}
		\min\limits_{\mC, B, \E_1, \D, \Z, \E_2}&  F\left(\mC, B, \E_1, \D, \Z, \E_2\right)=\frac{\lambda_1}{2}\left\|B\right\|^2+\lambda_2 \left\|\E_1\right\|_{11\phi}+\frac{\lambda_3}{2}\left\|\mH-\mC\times_3 B-\E_1\right\|^2+\lambda_4\left\|\Z\right\|_{F,1}^\psi \\
		&+\lambda_5\left\|\E_2\right\|_{11\phi}+\frac{\lambda_6}{2}\left\|\mC - \D*\Z^T - \E_2\right\|^2+\delta_{\bB}\left(B\right)+\delta_{\bC}\left(\mC\right)+\delta_{\bD}\left(\D\right).
	\end{array}
\end{equation}
Within the framework of the PAM algorithm \cite{ABRS10}, we alternatively update each variable:
\begin{equation*}
	\left\{\begin{array}{l}
		\mC^{t+1}=\arg\min_{\mC}~F\left(\mC, B^t, \E_1^t, \D^t, \Z^t, \E_2^t\right)+\frac{\rho_1}{2}\left\|\mC-\mC^t\right\|^2; \\
		
		B^{t+1}=\arg\min_{B}~F\left(\mC^{t+1}, B, \E_1^t, \D^t, \Z^t, \E_2^t\right)+\frac{\rho_2}{2}\left\|B-B^t\right\|^2; \\
		
		\E_1^{t+1}=\arg\min_{\E_1}~F\left(\mC^{t+1}, B^{t+1}, \E_1, \D^t, \Z^t, \E_2^t\right)+\frac{\rho_3}{2}\left\|\E_1-\E_1^t\right\|^2; \\
		
	\D^{t+1}=\arg\min_{\D}~F\left(\mC^{t+1}, B^{t+1}, \E_1^{t+1}, \D, \Z^t, \E_2^t\right)+\frac{\rho_4}{2}\left\|\D-\D^t\right\|^2; \\
	
	\Z^{t+1}=\arg\min_{\Z}~F\left(\mC^{t+1}, B^{t+1}, \E_1^{t+1}, \D^{t+1}, \Z, \E_2^t\right)+\frac{\rho_5}{2}\left\|\Z-\Z^t\right\|^2; \\
	
	\E_2^{t+1}=\arg\min_{\E_2}~F\left(\mC^{t+1}, B^{t+1}, \E_1^{t+1}, \D^{t+1}, \Z^{t+1}, \E_2\right)+\frac{\rho_6}{2}\left\|\E_2-\E_2^t\right\|^2,
	\end{array}\right.
\end{equation*}
where $\rho_i,~i\in[6]$ are six positive constants, and $t$ denotes the iteration number. Next, we provide details for updating the subproblems related to $\mC$, $B$, $\E_1$, $\D$, $\Z$, and $\E_2$, respectively.

\subsubsection{Computing $\mC^{t+1}$ and $B^{t+1}$} 
The objective function with respect to $\mC$ and $B$ can be written as
\begin{numcases}{}
\min\limits_{\mC} f\left(\mC, B^t\right)+\delta_{\bC}\left(\mC\right)+\frac{\rho_1}{2}\left\|\mC-\mC^t\right\|^2;  \label{O:C} \\
\min\limits_{B} f\left(\mC^{t+1}, B\right)+\delta_{\bB}\left(B\right)+\frac{\rho_2}{2}\left\|B-B^t\right\|^2, \label{O:B}
\end{numcases}
where $f\left(\mC, B\right)=\frac{\lambda_1}{2}\|B\|^2+\frac{\lambda_3}{2}\|\mH-\mC\times_3 B-\E_1^t\|^2+\frac{\lambda_6}{2}\|\mC - \D^t*\Z^{t^T} - \E_2^t\|^2$. However, \eqref{O:C} and \eqref{O:B} may have no closed-form solutions. Similarly to \cite{FZC20,YZCQ22}, we update $ \mC $ by solving the following subproblem
\begin{equation}\label{O:Ce}
	\begin{array}{rl}
		\mathop{\arg\min}\limits_{\mC}&   \delta_{\bC}\left(\mC\right)+\left\langle \nabla_{\mC}f\left(\mC^t, B^t\right), \mC-\mC^t \right\rangle + \frac{l_{\mC}^t}{2}\left\|\mC-\mC^t\right\|^2+\frac{\rho_1}{2}\left\|\mC-\mC^t\right\|^2\\
		=\mathop{\arg\min}\limits_{\mC}& \delta_{\bC}\left(\mC\right)+\frac{l_{\mC}^t+\rho_1}{2}\left\|\mC-\hat{\mC}^{t+1}\right\|^2\\
		=\mathop{\arg\max}\limits_{\mC}&\delta_{\bC}\left(\mC\right)+\left\langle\mC, \hat{\mC}^{t+1} \right\rangle,
	\end{array}
\end{equation}
where $\nabla_{\mC}f(\mC^t, B^t)=\lambda_3(\mC^t\times_3 B^t+\E_1^t-\mH)\times_3B^{t^T}+\lambda_6(\mC^t - \D^t*\Z^{t^T} - \E_2^t)$, $l_{\mC}^t=\lambda_3\|B^t\|_2^2+\lambda_6$, and 
$$
\hat{\mC}^{t+1}=\frac{l_{\mC}^t\mC^t+\rho_1\mC^t-\nabla_{\mC}f\left(\mC^t, B^t\right)}{l_{\mC}^t+\rho_1}.
$$
By applying the Cauchy-Schwarz inequality, the optimal solution of \eqref{O:Ce} is given by
\begin{equation}\label{opt:C}
	\mC^{t+1}(i,j,:) = \frac{\hat{\mC}^{t+1}(i,j,:)}{\left\|\hat{\mC}^{t+1}(i,j,:)\right\|}, \quad i\in [n_1],~j\in[n_2].
\end{equation}

Likewise, we can update $ B $ by solving the following subproblem
\begin{equation}\label{O:Be}
	\begin{array}{rl}
		\mathop{\arg\min}\limits_{B}& \delta_{\bB}\left(B\right)+\left\langle \nabla_{B}f\left(\mC^{t+1}, B^t\right), B-B^t \right\rangle + \frac{l_{B}^t}{2}\left\|B-B^t\right\|^2+\frac{\rho_2}{2}\left\|B-B^t\right\|^2
		\\
		=\mathop{\arg\min}\limits_{B}& \delta_{\bB}\left(B\right)+\frac{l_{B}^t+\rho_2}{2}\left\|B-\hat{B}^{t+1}\right\|^2,
	\end{array}
\end{equation}
where $\nabla_{B}f(\mC^{t+1}, B^t)=\lambda_1B^t + \lambda_3(B^tC_{(3)}^{t+1}+{E_1}^t_{(3)}-H_{(3)})C_{(3)}^{{t+1}^T}$, $l_{\mC}^t=\lambda_1 + \lambda_3\|C_{(3)}^{t+1}\|_2^2$, and 
$$
\hat{B}^{t+1}=\frac{l_{B}^tB^t+\rho_2B^t-\nabla_{B}f\left(\mC^{t+1}, B^t\right)}{l_{B}^t+\rho_2}.
$$
It is not difficult to see that the optimal solution of \eqref{O:Be} is given by
\begin{equation}\label{opt:B}
	B^{t+1}=\max\left\lbrace 0,  \hat{B}^{t+1}\right\rbrace.
\end{equation}

\subsubsection{Computing $\E_1^{t+1}$} 
The objective function with respect to $\E_1$ can be written as
\begin{equation}\label{O:E1}
	\begin{array}{rl}
		\mathop{\arg\min}\limits_{\E_1}&  \lambda_2 \left\|\E_1\right\|_{11\phi}+\frac{\lambda_3}{2}\left\|\mH-\mC^{t+1}\times_3 B^{t+1}-\E_1\right\|^2+\frac{\rho_3}{2}\left\|\E_1-\E_1^t\right\|^2\\
		=\mathop{\arg\min}\limits_{\E_1}& \hat{\lambda}_2 \left\|\E_1\right\|_{11\phi}+\frac{1}{2}\left\|\E_1-\hat{\E}_1^{t+1}\right\|^2,
	\end{array}
\end{equation}
where $\hat{\lambda}_2= \lambda_2/(\lambda_3+\rho_3)$ and
$$ \hat{\E}_1^{t+1}= \frac{\lambda_3\left(\mH-\mC^{t+1}\times_3 B^{t+1}\right)+\rho_3\E_1^t}{\lambda_3+\rho_3}. $$
By \cite[Appendix A.1]{PC21}, the optimal solution of \eqref{O:E1} is given by
\begin{equation}\label{opt:E1}
	\E_1^{t+1}\left(i, j,:\right)= \begin{cases}\max\left\lbrace 0,~ 1-\hat{\lambda}_2/e_1\right\rbrace \hat{\E}_1^{t+1}\left(i, j,:\right), & e_1 \leq 1+\hat{\lambda}_2; \\ \hat{\E}_1^{t+1}\left(i, j,:\right), & e_1>1+\hat{\lambda}_2,\end{cases} \quad i\in[n_1],~j\in[n_2],
\end{equation}
where $e_1 = \|\hat{\E}_1^{t+1}\left(i, j,:\right)\|$.

\subsubsection{Computing $\D^{t+1}$} 
The objective function with respect to $\D$ can be written as
\begin{equation}\label{O:D}
	\begin{array}{rl}
		\mathop{\arg\min}\limits_{\D}&  \frac{\lambda_6}{2}\left\|\mC^{t+1} - \D*\Z^{t^T} - \E_2^t\right\|^2+\frac{\rho_4}{2}\left\|\D-\D^t\right\|^2+\delta_{\bD}\left(\D\right)\\
		=\mathop{\arg\max}\limits_{\D}& \left\langle \D, \lambda_6\left(\mC^{t+1}-\E_2^t\right)*\Z^t+\rho_4\D^t\right\rangle+\delta_{\bD}\left(\D\right).
	\end{array}
\end{equation}
To solve the above problem, we introduce the following lemma.
\begin{lemma}\label{lem:or}
	Supposing T-SVD of $\G \in \R^{n_1\times n_2 \times b}$ is $\mathcal{U}*\mS*\mathcal{V}^T$, one has
	\begin{equation}
		\mathop{\arg\max}\limits_{\D^T*\D = \I} ~\left\langle \D,\G\right\rangle = \mathcal{U}*\mathcal{V}^T.
	\end{equation}
\end{lemma}
\begin{proof}
	See Appendix \ref{app:lem_or} for the proof.
\end{proof}
By using Lemma \ref{lem:or}, the optimal solution of \eqref{O:D} is given by
\begin{equation}\label{opt:D}
	\D^{t+1}=\U_\D*\V_\D^T,
\end{equation}
where $\U_\D$ and $\V_\D$ are obtained by the T-SVD decomposition: $\lambda_6(\mC^{t+1}-\E_2^t)*\Z^t+\rho_4\D^t = \U_\D*\mS_\D*\V_\D^T$.

\subsubsection{Computing $\Z^{t+1}$} 
The objective function with respect to $\Z$ can be written as
\begin{equation}\label{O:Z}
	\begin{array}{rl}
		\mathop{\arg\min}\limits_{\Z}&\lambda_4\left\|\Z\right\|_{F,1}^\psi +\frac{\lambda_6}{2}\left\|\mC^{t+1} - \D^{t+1}*\Z^T - \E_2^t\right\|^2+\frac{\rho_5}{2}\left\|\Z-\Z^t\right\|^2\\
		=\mathop{\arg\min}\limits_{\Z}&\hat{\lambda}_4\left\|\Z\right\|_{F,1}^\psi +\frac{1}{2}\left\|\Z-\hat{\Z}^{t+1}\right\|^2,
	\end{array}
\end{equation}
where $\hat{\lambda}_4 = \lambda_4/(\lambda_6+\rho_5)$ and
$$
\hat{\Z}^{t+1} = \frac{\lambda_6\left(\mC^{t+1} - \E_2^t\right)^T*\D^{t+1}+\rho_5\Z^t}{\lambda_6+\rho_5}.
$$
By \cite[Appendix A]{PC21} and \cite{MS12}, the optimal solution of \eqref{O:Z} is given by
\begin{equation}\label{opt:Z}
	\Z^{t+1}\left(:,j,:\right)= \begin{cases}\operatorname{Prox}_{\hat{\lambda}_4}\psi\left(\|\hat{\Z}^{t+1}\left(:,j,:\right)\|\right)\frac{\hat{\Z}^{t+1}\left(:,j,:\right)}{\|\hat{\Z}^{t+1}\left(:,j,:\right)\|}, & \|\hat{\Z}^{t+1}\left(:,j,:\right)\| \neq 0; \\ 0, & \|\hat{\Z}^{t+1}\left(:,j,:\right)\|=0,\end{cases}
	\quad j \in [r],
\end{equation}
where
$$
\operatorname{Prox}_{\hat{\lambda}_4}\psi\left(z\right) = \begin{cases}u_1, & \text {if}\; \hat{\lambda}_4\psi\left(u_1\right)+\frac{1}{2}\left(u_1-z\right)^2 \leq \hat{\lambda}_4\psi\left(u_2\right)+\frac{1}{2}\left(u_2-z\right)^2; \\ u_2, & \text{otherwise}.\end{cases}
$$
Here $u_1=\min\{\operatorname{prox}_{\hat{\lambda}_4}\|\cdot\|^p(z), v\}$, $u_2=\max \{z, v\}$, and
$$
\operatorname{prox}_{\hat{\lambda}_4}\|\cdot\|^p(z)= \begin{cases}0, & \text { if }z<\pi_2; \\ \left\{0, \pi_1\right\}, & \text { if }z=\pi_2; \\ \pi_{\star}, & \text { if }z>\pi_2,\end{cases}
$$
where $\pi_1=(2 \hat{\lambda}_4(1-p))^{1/(2-p)}$, $\pi_2=\pi_1+\hat{\lambda}_4 p \pi_1^{p-1}$, and $\pi_{\star} \in\left(\pi_1,z\right)$ is the solution of $g(\pi)=\pi+\hat{\lambda}_4 p \pi^{p-1}-z=0$ with $\pi>0$.
\subsubsection{Computing $\E_2^{t+1}$} 
The objective function with respect to $\E_2$ can be written as
\begin{equation}\label{O:E2}
	\begin{array}{rl}
		\mathop{\arg\min}\limits_{\E_2}&   \lambda_5\left\|\E_2\right\|_{11\phi}+\frac{\lambda_6}{2}\left\|\mC^{t+1} - \D^{t+1}*\Z^{{t+1}^T} - \E_2\right\|^2+\frac{\rho_6}{2}\left\|\E_2-\E_2^t\right\|^2\\
		=\mathop{\arg\min}\limits_{\E_2}&\hat{\lambda}_5\left\|\E_2\right\|_{11\phi}+\frac{1}{2}\left\|\E_2-\hat{\E}_2^{t+1}\right\|^2,
	\end{array}
\end{equation}
where $\hat{\lambda}_5=\lambda_5/(\lambda_6+\rho_6)$ and 
$$
\hat{\E}_2^{t+1}=\frac{\lambda_6\left(\mC^{t+1} - \D^{t+1}*\Z^{{t+1}^T}\right)+\rho_6\E_2^t}{\lambda_6+\rho_6}.
$$
By \cite[Appendix A.1]{PC21}, the optimal solution of \eqref{O:E1} is given by
\begin{equation}\label{opt:E2}
	\E_2^{t+1}\left(i, j,:\right)= \begin{cases}\max\left\lbrace 0,~ 1-\hat{\lambda}_5/e_2\right\rbrace \hat{\E}_2^{t+1}\left(i, j,:\right), & e_2 \leq 1+\hat{\lambda}_5; \\ \hat{\E}_2^{t+1}\left(i, j,:\right), & e_2>1+\hat{\lambda}_5,\end{cases} \quad i\in[n_1],~j\in[n_2],
\end{equation}
where $e_2 = \|\hat{\E}_2^{t+1}\left(i, j,:\right)\|$.

The entire procedure of the proposed method is summarized in Algorithm \ref{Alg:SPA}.
\begin{algorithm}[!htbp]
	\caption{The PAM-based solver for the proposed LTD model.}\label{Alg:SPA}
	\KwIn{HSI data $\mH$, regularization parameters $\left\lbrace\lambda_i\right\rbrace_{i=1}^6$ and proximal parameters $\left\lbrace\rho_i\right\rbrace_{i=1}^6$.}
	\KwOut{Detection map $T$.}
	$t\leftarrow 0$;
	
	\While{not converged}
	{
		Update $\mC^{t+1}$ via \eqref{opt:C};
		
		Update $B^{t+1}$ via \eqref{opt:B};
		
		Update $\E_1^{t+1}$ via \eqref{opt:E1};
		
		Update $\D^{t+1}$ via \eqref{opt:D};
		
		Update $\Z^{t+1}$ via \eqref{opt:Z};
		
		Update $\E_2^{t+1}$ via \eqref{opt:E2};
		
		
		$t\leftarrow t+1$\;
	}
	Obtain the optimal anomalies $\E_1^*$, $\E_2^*$;
	
	Compute the detection map $T$ by \eqref{opt:T}.
\end{algorithm}

\subsection{Convergence analysis}
In this subsection, we rigorously establish the global convergence guarantee of Algorithm \ref{Alg:SPA}. Prior to proving the convergence of the proposed method, we first introduce several pivotal lemmas that are essential for our subsequent analysis. 
\begin{lemma}(sufficient decrease condition)\label{lem:sdc}
Assume that $ \W^t:=(\mC^t, B^t, \E_1^t, \D^t, \Z^t, \E_2^t) $ is generated by Algorithm \ref{Alg:SPA}. Then, $ \W^t $ satisfies the following properties:
\begin{itemize}
	\item[(1)] $F(\W^{t+1})+\rho\|\W^{t+1}-\W^t\|^2 \le F(\W^t)$, where $ \rho = \min_{i\in [6]} \{\rho_i/2\} $.
	\item[(2)] $\lim_{t\to\infty}\|\W^{t+1}-\W^t\| = 0 $.
	\item[(3)] $\W^t$ is bounded if either $\Z^t$ or $\E_2^t$ is bounded.
\end{itemize}
\end{lemma}
\begin{proof}
$(1)$ From the Lipschitz continuity of $\nabla_{\mC}f(\mC, B^t)$ about $\mC$, it holds that
\begin{equation}\label{CA:C-1}
	f\left(\mC^{t+1}, B^t\right)\le f\left(\mC^t, B^t\right)+\left\langle \nabla_{\mC}f\left(\mC^t, B^t\right),  \mC^{t+1}-\mC^t\right\rangle + \frac{l_{\mC}^t}{2}\left\|\mC^{t+1}-\mC^t\right\|^2.
\end{equation}
Given that $ \mC^{t+1} $ is the minimizer of \eqref{O:Ce}, we derive
\begin{equation}\label{CA:C-2}
\delta_{\bC}\left(\mC^{t+1}\right)+\left\langle \nabla_{\mC}f\left(\mC^t, B^t\right), \mC^{t+1}-\mC^t \right\rangle + \frac{l_{\mC}^t}{2}\left\|\mC^{t+1}-\mC^t\right\|^2+\frac{\rho_1}{2}\left\|\mC^{t+1}-\mC^t\right\|^2
\le\delta_{\bC}\left(\mC^t\right).
\end{equation}
By summing \eqref{CA:C-1} and \eqref{CA:C-2}, one has
\begin{equation}\label{CA:C}
	F\left(\mC^{t+1}, B^t, \E_1^t, \D^t, \Z^t, \E_2^t\right)+\frac{\rho_1}{2}\left\|\mC^{t+1}-\mC^t\right\|^2\le F\left(\mC^t, B^t, \E_1^t, \D^t, \Z^t, \E_2^t\right).
\end{equation}
Similarly, for $B$, we have
\begin{equation}\label{CA:B}
	F\left(\mC^{t+1}, B^{t+1}, \E_1^t, \D^t, \Z^t, \E_2^t\right)+\frac{\rho_2}{2}\left\|B^{t+1}-B^t\right\|^2\le F\left(\mC^{t+1}, B^t, \E_1^t, \D^t, \Z^t, \E_2^t\right).
\end{equation}
Given that $ \E_1^{t+1} $, $ \D^{t+1} $, $ \Z^{t+1} $ and $ \E_2^{t+1} $ are the minimizers of \eqref{O:E1}, \eqref{O:D}, \eqref{O:Z}, and \eqref{O:E2}, respectively, we can deduce that
\begin{equation}\label{CA:O}
\begin{aligned}
	F\left(\mC^{t+1}, B^{t+1}, \E_1^{t+1}, \D^t, \Z^t, \E_2^t\right)+\frac{\rho_3}{2}\left\|\E_1^{t+1}-\E_1^t\right\|^2&\le F\left(\mC^{t+1}, B^{t+1}, \E_1^t, \D^t, \Z^t, \E_2^t\right),\\
	F\left(\mC^{t+1}, B^{t+1}, \E_1^{t+1}, \D^{t+1}, \Z^t, \E_2^t\right)+\frac{\rho_4}{2}\left\|\D^{t+1}-\D^t\right\|^2&\le F\left(\mC^{t+1}, B^{t+1}, \E_1^{t+1}, \D^t, \Z^t, \E_2^t\right),\\
	F\left(\mC^{t+1}, B^{t+1}, \E_1^{t+1}, \D^{t+1}, \Z^{t+1}, \E_2^t\right)+\frac{\rho_5}{2}\left\|\Z^{t+1}-\Z^t\right\|^2&\le F\left(\mC^{t+1}, B^{t+1}, \E_1^{t+1}, \D^{t+1}, \Z^t, \E_2^t\right),\\
	F\left(\mC^{t+1}, B^{t+1}, \E_1^{t+1}, \D^{t+1}, \Z^{t+1}, \E_2^{t+1}\right)+\frac{\rho_6}{2}\left\|\E_2^{t+1}-\E_2^t\right\|^2&\le F\left(\mC^{t+1}, B^{t+1}, \E_1^{t+1}, \D^{t+1}, \Z^{t+1}, \E_2^t\right).
\end{aligned}
\end{equation}
Combining \eqref{CA:C}, \eqref{CA:B} and \eqref{CA:O}, we have
\begin{equation}\label{CA:D}
	F\left(\W^{t+1}\right)+\rho\left\|\W^{t+1}-\W^t\right\|^2 \le F\left(\W^t\right),
\end{equation}
where $ \rho = \min_{i\in [6]} \{\rho_i/2\} $. Thus, we complete the proof of statement $(1)$.

$(2)$ Summing up \eqref{CA:D} over $t = 0, 1, \ldots, \infty$, it gives
\begin{equation}
	\rho\sum_{t=0}^{\infty}\left\|\W^{t+1}-\W^t\right\|^2 \le F\left(\W^0\right) - F\left(\W^{\infty}\right)<\infty,
\end{equation}
where the last inequality uses $F\left(\W^{\infty}\right)\ge0$. Thus, $\lim_{t\to\infty}\left\|\W^{t+1}-\W^t\right\| = 0 $.

$(3)$ According to \eqref{CA:D}, we have
\begin{equation*}
	F\left(\W^{t+1}\right) \le F\left(\W^t\right) \le \ldots \le F\left(\W^0\right)<\infty.
\end{equation*}
Thus, $ F(\W^t) $ is bounded. Combining this with
\begin{equation}\label{def:F}
	\begin{aligned}
		F\left(\W^t\right)=&\frac{\lambda_1}{2}\left\|B^t\right\|^2+\lambda_2 \left\|\E_1^t\right\|_{11\phi}+\frac{\lambda_3}{2}\left\|\mH-\mC^t\times_3 B^t-\E_1^t\right\|^2+\lambda_4\left\|\Z^t\right\|_{F,1}^\psi \\
		&+\lambda_5\left\|\E_2^t\right\|_{11\phi}+\frac{\lambda_6}{2}\left\|\mC^t - \D^t*{\Z^t}^T - \E_2^t\right\|^2+\delta_{\bB}\left(B^t\right)+\delta_{\bC}\left(\mC^t\right)+\delta_{\bD}\left(\D^t\right), 
	\end{aligned}
\end{equation}
we can deduce each term in \eqref{def:F} is bounded. From the boundedness of $\|B^t\|$, $\delta_{\bC}(\mC^t)$ and $\delta_{\bD}(\D^t)$, we can deduce $B^t$, $\mC^t$ and $\D^t$ are bounded. Combining this with the boundedness of $\|\mH - \mC^t \times_3 B^t - \E_1^t\|$ and $\|\mC^t - \D^t * {\Z^t}^T - \E_2^t\|$, and the fact that either $\Z^t$ or $\E_2^t$ is bounded, we obtain that $\Z^t$, $\E_1^t$, and $\E_2^t$ are bounded. Therefore, we can conclude that $\W^t$ is bounded.
\end{proof}

\begin{lemma}(relative error condition)\label{lem:rec}
	Let the sequence $\{\W^t\}$ be generated by Algorithm \ref{Alg:SPA}. Suppose that either $\Z^t$ or $\E_2^t$ is bounded. Then, there exists $\varpi>0$ such that
	\begin{equation}
		\left\| V^{t+1}\right\| \le \varpi\left\|\W^{t+1}-\W^t\right\|
	\end{equation}
	for any $ V^{t+1} \in \partial F(\W^{t+1})$.
\end{lemma}
\begin{proof}
	By the first-order necessary optimality conditions in \eqref{O:Ce}, \eqref{O:Be}, \eqref{O:E1}, \eqref{O:D}, \eqref{O:Z}, and \eqref{O:E2}, we have
	\begin{equation}\label{REC-1}
		\left\{\begin{array}{l}
			0 \in \partial \delta_{\bC}\left(\mC^{t+1}\right)+ \nabla_{\mC}G_1\left(\mC^t, B^t,\E_1^t,\D^t,\Z^t,\E_2^t\right)+\left(l_{\mC}^t+\rho_1\right)\left(\mC^{t+1}-\mC^t\right);\\
			0 \in  \partial \delta_{\bB}\left(\B^{t+1}\right)+ \nabla_{\B}G_1\left(\mC^{t+1}, B^t,\E_1^t,\D^t,\Z^t,\E_2^t\right)+\left(l_{\B}^t+\rho_2\right)\left(\B^{t+1}-B^t\right);\\
			0 \in \lambda_2\partial\left\|\E_1^{t+1}\right\|_{11\phi}+ \lambda_3\left(-\mH+\mC^{t+1}\times_3 B^{t+1}+\E_1^{t+1}\right)+\rho_3\left(\E_1^{t+1}-\E_1^t\right);\\
			0 \in \partial\delta_{\bD}\left(\D^{t+1}\right)+ \nabla_{\D}G_2\left(\D^{t+1},\Z^t,\E_2^t\right)+\rho_4\left(\D^{t+1}-\D^t\right);\\
			0 \in \lambda_4\partial\left\|\Z^{t+1}\right\|_{F,1} + \nabla_{\Z}G_2\left(\D^{t+1},\Z^{t+1},\E_2^t\right)+\rho_5\left(\Z^{t+1}-\Z^t\right); \\
			0 \in \lambda_5\partial\left\|\E_2^{t+1}\right\|_{11\phi} + \lambda_6\left(-\mC^{t+1} + \D^{t+1}*\Z^{{t+1}^T} + \E_2^{t+1} \right)+\rho_6\left(\E_2^{t+1}-\E_2^t\right),
		\end{array}\right.
	\end{equation}
	where $G_1(\mC, B,\E_1,\D,\Z,\E_2)=\frac{\lambda_1}{2}\|B\|^2+\frac{\lambda_3}{2}\|\mH-\mC\times_3 B-\E_1\|^2+\frac{\lambda_6}{2}\|\mC - \D*\Z^T - \E_2\|^2$ and $G_2(\D,\Z,\E_2)=\frac{\lambda_6}{2}\|\mC^{t+1} - \D*\Z^T - \E_2\|^2$. Denoting $V_1^{t+1}$, $V_2^{t+1}$, $V_3^{t+1}$, $V_4^{t+1}$, $V_5^{t+1}$ and $V_6^{t+1}$ as follows:
\begin{equation}\label{REC-2}
	\left\{\begin{array}{l}
		V_1^{t+1}= \nabla_{\mC}G_1\left(\W^{t+1}\right)-\nabla_{\mC}G_1\left(\W^t\right)-\left(l_{\mC}^t+\rho_1\right)\left(\mC^{t+1}-\mC^t\right);\\
		V_2^{t+1}=  \nabla_{\B}G_1\left(\W^{t+1}\right)- \nabla_{\B}G_1\left(\mC^{t+1}, B^t,\E_1^t,\D^t,\Z^t,\E_2^t\right)-\left(l_{\B}^t+\rho_2\right)\left(\B^{t+1}-B^t\right);\\
		V_3^{t+1}= -\rho_3\left(\E_1^{t+1}-\E_1^t\right);\\
		V_4^{t+1}= \nabla_{\D}G_2\left(\D^{t+1},\Z^{t+1},\E_2^{t+1}\right)-\nabla_{\D}G_2\left(\D^{t+1},\Z^t,\E_2^t\right)-\rho_4\left(\D^{t+1}-\D^t\right);\\
		V_5^{t+1}= \nabla_{\Z}G_2\left(\D^{t+1},\Z^{t+1},\E_2^{t+1}\right)-\nabla_{\Z}G_2\left(\D^{t+1},\Z^{t+1},\E_2^t\right)-\rho_5\left(\Z^{t+1}-\Z^t\right); \\
		V_6^{t+1}=-\rho_6\left(\E_2^{t+1}-\E_2^t\right).
	\end{array}\right.
\end{equation}
Combining \eqref{REC-1}, \eqref{REC-2} and recalling the definition of $F(\W)$, one has
$$
V^{t+1}=\left(V_1^{t+1},V_2^{t+1},V_3^{t+1},V_4^{t+1},V_5^{t+1},V_6^{t+1}\right)  \in \partial F\left(\W^{t+1}\right).
$$
It is evident that $G_1$ and $G_2$ are Lipschitz continuous on any bounded set. From Lemma \ref{lem:sdc}(3), we know that $\W^t$ is bounded. Hence, we obtain
\begin{equation}
	\left\| V^{t+1}\right\| \le \varpi\left\|\W^{t+1}-\W^t\right\|,
\end{equation}
where $ \varpi = \max_{i\in [6]}\{\rho_i+l_g\} $ and $ l_g $ is the maximum Lipschitz constant of $G_1$ and $G_2$.
\end{proof}

\begin{lemma}\label{lem:KL}
	The function $F(\W)$ is a Kurdyka-Łojasiewicz (KL) function.
\end{lemma}
\begin{proof}
	According to \cite{BST14}, the Frobenius norm functions $\frac{\lambda_1}{2}\|B\|^2$, $\frac{\lambda_3}{2}\|\mH-\mC\times_3 B-\E_1\|^2$, and $ \frac{\lambda_6}{2}\|\mC - \D*\Z^T - \E_2\|^2 $ are semi-algebraic functions. The terms $\lambda_2 \|\E_1\|_{11\phi}$ and $ \lambda_5\|\E_2\|_{11\phi}$ are also semi-algebraic functions because they are composed of semi-algebraic operations, including the Frobenius norm, minimum, and finite summation. Similarly, $\lambda_4\|\Z\|_{F,1}^\psi$, which involves only the Frobenius norm and finite summation, is semi-algebraic. The functions $\delta_{\bB}(B)$, $ \delta_{\bC}(\mC) $, and $\delta_{\bD}(\D)$ are semi-algebraic because they are indicator functions with semi-algebraic sets \cite{BST14}. 
	Hence, $F(\W)$ is semi-algebraic because it is a finite sum of semi-algebraic functions. Additionally, since $F(\W)$ is a proper continuous function, it follows from \cite[Theorem 3]{BST14} that $F(\W)$ is a KL function. This completes the proof.
\end{proof}

Finally, we provide a theoretical guarantee for the convergence of Algorithm \ref{Alg:SPA}.

\begin{theorem}
	Consider the sequence $\{\W^t\}$ obtained by Algorithm \ref{Alg:SPA}. Assuming that either $\Z^t$ or $\E_2^t$ is bounded, the sequence $\{\W^t\}$ converges to a critical point of $F(\W)$.
\end{theorem}
\begin{proof}
	We begin by noting that Lemma \ref{lem:sdc}(3) establishes the boundedness of the sequence $\{\W^t\}$. Given this boundedness, the Bolzano-Weierstrass theorem ensures the existence of a convergent subsequence. Furthermore, by exploiting the continuity of $F(\W)$, along with the results derived from Lemmas \ref{lem:sdc}, \ref{lem:rec}, \ref{lem:KL}, we can rigorously establish the desired conclusion as articulated in \cite[Theorem 2.9]{ABS13}.
\end{proof}

\subsection{Rank reduction strategy with validation mechanism}
In this subsection, we propose a rank reduction strategy to decrease $r$, which is related to the dimensions of $\D \in \mathbb{R}^{n_1 \times r \times b} $ and $\Z \in \mathbb{R}^{n_2 \times r \times b}$, thereby reducing the complexity of Algorithm \ref{Alg:SPA}.  Unlike most decomposition methods that set a small initial $r^0$, we set $r^0 = \min\{n_1, n_2\}$ and leverage the algorithm's inherent capability to adaptively reduce $r^t$. This approach is initially supported by Theorem \ref{Thm:rank}. Furthermore, we present the following theorem to substantiate our proposed rank reduction strategy.
\begin{theorem}\label{Thm:rc}
	There exists $t^\#\in\mathbb{N}$ such that $\Gamma(\Z^t)=\Gamma(\Z^{t+1}) $ for each $t\ge t^\#$, where $\Gamma(\Z)=\left\lbrace j \mid\|\Z(:,j,:)\| \neq 0, \, j=1, \ldots, r\right\rbrace$.
\end{theorem}
\begin{proof}
	By Lemma \ref{lem:sdc} that $\lim_{t\to\infty} \|\Z^{t+1}-\Z^t\| =0$, there exists $t^\#\in\mathbb{N}$ such that $\|\Z^{t+1}-\Z^t\| < \min\{(2\hat{\lambda}_4(1-p))^{1/(2-p)}, \nu\}$ for each $t\ge t^\#$. Proving by contradiction, we assume that there exists $t\ge t^\#$ such that $\Gamma(\Z^t)\neq\Gamma(\Z^{t+1}) $. Then there exists $j \in r$ such that (a) $\Z^t(:,j,:)\neq0$  and $\Z^{t+1}(:,j,:)=0$, or (b) $\Z^t(:,j,:)=0$  and $\Z^{t+1}(:,j,:)\neq0$. Hence by \eqref{opt:Z}, we have  for both cases that
	$$
	\left\|\Z^{t+1}-\Z^t\right\|\ge\left\|\Z^{t+1}(:,j,:)-\Z^t(:,j,:)\right\|\ge \min\{(2\hat{\lambda}_4(1-p))^{1/(2-p)}, \nu\},
	$$
	which yields a contradiction. The proof is complete.
\end{proof}

Based on the aforementioned theorem, when $t\ge t^\#$, we can remove the zero lateral slices of $\Z^t$ and the corresponding lateral slices in $\D^t$. Consequently, the dimensions of $\D^t$ and $\Z^t$ can be reduced to  $\mathbb{R}^{n_1 \times |\Gamma(\Z^\#)| \times b}$ and $\mathbb{R}^{n_2 \times |\Gamma(\Z^\#)| \times b}$, respectively, when $t\ge t^\#$. 
  However, when $t < t^\#$, the computational cost remains high due to the large value of $\min\{n_1, n_2\}$, and $t^\#$ is uncertain and potentially very large.

  Therefore, we initiate the reduction of $r$ from $t = 1$, as illustrated in Step 21 of Algorithm \ref{Alg:Rank}. However, this approach introduces a potential issue: when $\Z^{t+1}(:,j,:) = 0$, it does not necessarily imply that $\Z^{t+2}(:,j,:) = 0$ for $t+1 <t^\#$. Consequently, removing  $\Z^{t+1}(:,j,:)$ may be unjustified. To address this, we incorporate a verification step in Steps 16-20 of Algorithm \ref{Alg:Rank} to assess the validity of removing $\Z^{t+1}(:,j,:)$. If  $\Z_{sub}^{t+1}(:,j,:) = 0$, we deem it reasonable to remove $\Z^{t+1}(:,j,:)$; conversely, if  $\Z_{sub}^{t+1}(:,j,:) \neq 0$, we consider it unreasonable to remove $\Z^{t+1}(:,j,:)$ and reintegrate  $\Z_{sub}^{t+1}(:,j,:)$ back into $\Z^{t+1}$. This procedure is detailed in Algorithm \ref{Alg:Rank}.

\begin{algorithm}[!htbp]
	\caption{Rank reduction with validation mechanism strategy for Algorithm \ref{Alg:SPA}.}\label{Alg:Rank}
	\KwIn{HSI data $\mH$, regularization parameters $\left\lbrace\lambda_i\right\rbrace_{i=1}^6$ and proximal parameters $\left\lbrace\rho_i\right\rbrace_{i=1}^6$.}
	\KwOut{Detection map $T$.}
	$t\leftarrow 0$;
	
	\While{not converged}
	{
		Update $\mC^{t+1}$, $B^{t+1}$, $\E_1^{t+1}$, $\D^{t+1}$, $\Z^{t+1}$, and $\E_2^{t+1}$ using Algorithm \ref{Alg:SPA};
		
		\If{t $\ge$ 2 \& $\mathbb{S}\neq \emptyset$}{
		Compute
		$\Z_{sub}^{t+1}=\operatorname{Prox}_{\hat\lambda_4}\psi\left( \lambda_6\left(\mC^{t+1} - \E_2^{t+1}\right)^T*\D_{sub}^t(:,\mathbb{S},:)/(\lambda_6+\rho_5)\right)$;
		
		Identify the top 5 lateral slices of $\Z_{sub}^{t+1}$ with the highest Frobenius norms, or all if fewer than 5, and denote this set as $\mathbb{S}_1$;
		
		Let $\Z^{t+1}=[\Z^{t+1}, \Z_{sub}^{t+1}(:,\mathbb{S}_1,:)]$ and $\D^{t+1}=[\D^{t+1}, \D_{sub}^t(:,\mathbb{S}_1,:)]$;
		}
		
		Let $\mathbb{S} = \{j \mid \|\Z^{t+1}(:,j,:)\|=0\}$, and define $\D_{sub}^{t+1} = \D^{t+1}(:,\mathbb{S},:)$. Then, remove the lateral slices at positions in $\mathbb{S}$ from $\D^{t+1}$ and $\Z^{t+1}$;
		
		$t\leftarrow t+1$\;
	}
	Obtain the optimal anomalies $\E_1^*$, $\E_2^*$;
	
	Compute the detection map $T$ by \eqref{opt:T}.
\end{algorithm}

\begin{remark}
	In Algorithm \ref{Alg:Rank}:
	
	(1) Step 21 significantly reduces the value of $r$, thereby decreasing the algorithm's complexity. Meanwhile, Steps 16-20 ensure optimal solution quality by providing a mechanism to correct any erroneous reductions in $r$. This interplay between the steps enhances the algorithm's efficiency while maintaining high-quality results. 
	
	(2) In Step 21, if $ |\mathbb{S}| = size(\Z^{t+1}, 2) $, executing this step would result in $\D^{t+1}$ and $\Z^{t+1}$ becoming empty. We consider this to be unreasonable, and therefore, this step is not executed in such cases.
	
	(3) In Step 18, to prevent $r$ from increasing too rapidly, we impose a limit of a maximum increase of $5$ per step.
\end{remark}

\subsection{Complexity analysis}
Here we analyze the computational complexity of Algorithm \ref{Alg:Rank}. At each iteration, updating $\mC$ involves performing the tensor-matrix product and t-product, with a computational complexity of $ \mathcal{O}\left(n_1n_2(n_3+r^t)b\right) $ when the iterates $\D^t$ and $\Z^t$ contain $ r^t $ nonzero lateral slices. 
For updating $B$ and $\E_1$, the primary computational cost is associated with the matrix (tensor-matrix) product, it costs $\mathcal{O}\left(n_1n_2n_3b\right)$.
Updating $\D$ requires a computational cost of $\mathcal{O}\left((n_1+n_2+\log b)n_1r^tb\right)$ to perform the (inverse) DFT, SVD, and matrix product calculations.
The computational cost for updating $ \Z $ and $ \E_2 $ is $\mathcal{O}\left(n_1n_2r^tb+n_1(n_2+r^t)b\log b\right)$ and $\mathcal{O}\left(n_1n_2r^tb+(n_1+n_2)r^tb\log b\right)$, respectively.
In summary, the overall computational complexity of our proposed method is
$ \mathcal{O}\left(n_1n_2(n_3+r^t)b\right) $, assuming that $n_1\approx n_2$ in most cases. It is important to note that $r^t$ decreases with each iteration and eventually becomes much smaller than $\min\{n_1, n_2\}$. The computational complexity of a single SVD operation for the input HSI $\mH$ is $ \mathcal{O}\left(n_1n_2n_3\min\{n_1n_2, n_3\}\right) $ in matrix based methods. Therefore, the proposed method offers a lower computational cost compared to other low rank approximation based HAD methods.

\section{Experiments}
In this section, we present a series of experiments designed to assess the performance of the proposed algorithm. These experiments were conducted on a server equipped with 16 logical CPU cores and 16 GB of memory. All algorithms were implemented using MATLAB 2022a, and no preprocessing was applied to maintain fairness.

For comparison, we employed several established methods in our experiments. These included RX \cite{RY90}, which is statistic based; RPCA \cite{SLL14} and LRASR \cite{XWL16}, which are matrix based; PTA \cite{LLQ22}, TPCA \cite{CYW18}, and TLRSR \cite{WWH23}, which are tensor based; and RGAE \cite{FMM22} and GAED \cite{XAJZ22}, which are deep learning based. To evaluate detection performance, we utilized three widely recognized metrics in addition to visually observing the resulting anomaly maps: the receiver operating characteristic (ROC) curve \cite{Ker08}, the area under the ROC curve (AUC) value \cite{KHSM11}, and the separability map between anomalies and the background.

\subsection{Datasets description}
In this section, the proposed method is evaluated using the Airport-Beach-Urban (ABU) database, a real hyperspectral database, and the MVTec database, a natural image database, as detailed below.

1) Airport-Beach-Urban Dataset \footnote{http://xudongkang.weebly.com/data-sets.html.}: The Airport-Beach-Urban (ABU) dataset includes three distinct scenes: Airport, Beach, and Urban. The Airport and Beach scenes each contain four images, while the Urban scene includes five images. All images were captured using the Airborne Visible/Infrared Imaging Spectrometer (AVIRIS) sensor, producing images with approximately $200$ spectral bands, except for one Beach scene, which was captured by the Reflective Optics System Imaging Spectrometer (ROSIS-03) sensor, producing images with approximately $100$ spectral bands. Detailed acquisition procedures are described in \cite{KZL17}. The sample images, sized $100\times 100$ or $150\times 150$ pixels, include corresponding reference maps. These reference maps were manually labeled using the Environment for Visualizing Images (ENVI) software. Due to space limitations, we selected two images from the Airport and Urban scene datasets for the experiments, as illustrated in Figure \ref{fig:HSIs}, which shows the selected image scenes and their corresponding ground truth maps.

2) MVTec Dataset \footnote{https://www.mvtec.com/company/research/datasets/mvtec-ad.}:
The MVTec dataset \cite{BFSS19}, developed by MVTec Software GmbH, is specifically designed for anomaly detection in a real industrial scenario. This dataset encompasses $15$ different industrial products, which are categorized into texture and object classes. The texture class includes five categories: carpet, grid, leather, tile, and wood. The object class comprises ten categories: bottle, cable, capsule, hazelnut, metal nut, pill, screw, toothbrush, transistor, and zipper. Each product category contains multiple defect types, resulting in a total of $73$ distinct defects within the dataset. The resolution of the images ranges from $700 \times 700$ pixels to $1024 \times 1024$ pixels. Due to space constraints, we selected images of hazelnuts exhibiting all defect types (crack, cut, hole, and print) for our experiments and resized them to $512 \times 512$ pixels. Figure \ref{fig:MVTec} presents the selected images along with their corresponding ground truth maps.


\subsection{Discussion}

\subsubsection{Effects of rank reduction strategy with validation mechanism}\label{sec:ra}

In our model, the large tensor $\mC\in\R^{n_1\times n_2\times b}$ is decomposed into two smaller tensors, $\D\in\R^{n_1\times r\times b}$ and $\Z\in\R^{n_2\times r\times b}$, utilizing the t-product, with the dimensions of $\D$ and $\Z$ determined by the parameter $r$. By leveraging the principle of group sparsity, we propose a rank reduction strategy with validation mechanism in Algorithm \ref{Alg:Rank} to adaptively adjust the value of $r$.

Initially, we depict the variation in $r$ values as a function of iteration count for the MVTec dataset in Figure \ref{fig:r_adj}. From Figure \ref{fig:r_adj}, it is evident that $r$ significantly decreases during the initial stages of the algorithm, iterating at a lower value to effectively find a suboptimal solution. However, as iterations progress, the smaller $r$ becomes insufficient to achieve optimal results. Consequently, the algorithm dynamically adjusts $r$, incrementally restoring some of the previously reduced values to compensate for the excessive reduction. This adjustment gradually increases $r$ and eventually stabilizes in the later stages, consistent with the description in Theorem \ref{Thm:rc}. This process demonstrates the algorithm's effectiveness in adaptively adjusting $r$, thereby achieving a good balance between computational efficiency and decomposition accuracy.
\begin{figure}[htbp]
	\centering
	\includegraphics[width=1\linewidth]{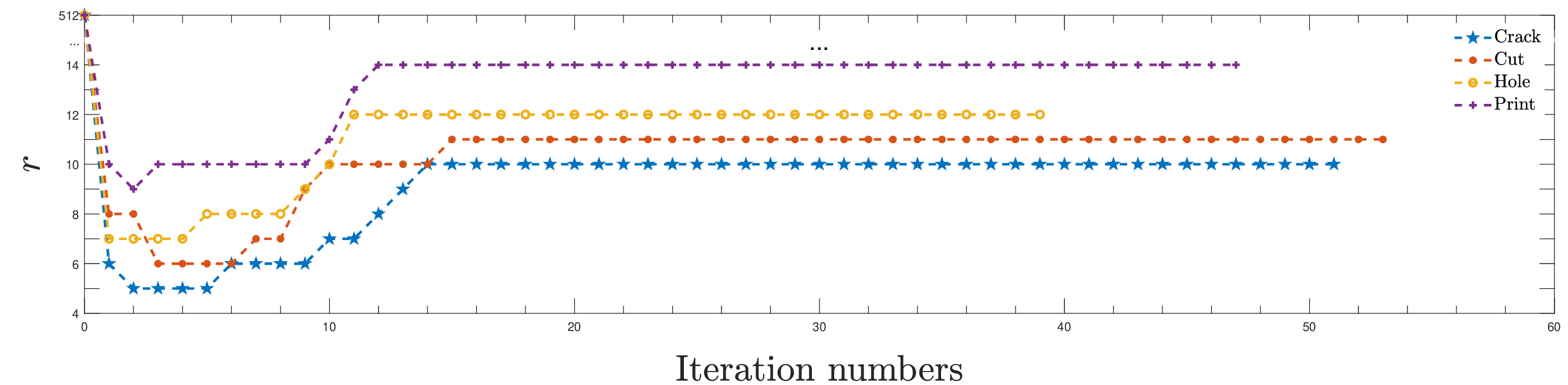}
	\caption{Variation of $r$ values across iterations for Algorithm \ref{Alg:Rank}.}
	\label{fig:r_adj}
\end{figure}

Moreover, Figure \ref{fig:rv} presents a detailed evaluation of AUC, computational time, and $r$ metrics for Algorithm \ref{Alg:SPA}, Algorithm \ref{Alg:Rank}, and Algorithm \ref{Alg:Rank} without validation mechanism (i.e., excluding steps 16-20), across a range of $\lambda_4$ values. From the detection results, the AUC value produced by Algorithm \ref{Alg:Rank} without validation mechanism exhibits significant sensitivity to the parameter $\lambda_4$. This sensitivity arises because an excessively large $\lambda_4$ results in an overly small $r$, leading to suboptimal performance. This issue is prevalent among most algorithms that employ rank reduction operations. In contrast, our proposed Algorithm \ref{Alg:Rank} incorporates a validation mechanism within its rank reduction strategy. This mechanism evaluates the reasonableness of the reduction in $r$. By leveraging this validation mechanism, it mitigates the issue of an excessively large $\lambda_4$ causing $r$ to become too small, thereby achieving more stable results and reducing sensitivity to $\lambda_4$. Algorithm \ref{Alg:SPA}, which lacks a rank reduction operation, also demonstrates minimal sensitivity to the parameter $\lambda_4$. Regarding time consumption, Algorithm \ref{Alg:SPA}, due to the absence of a rank reduction operation, maintains $r$ at a relatively large value of $512$, resulting in substantial time consumption. Conversely, both Algorithm \ref{Alg:Rank} and Algorithm \ref{Alg:Rank} without validation mechanism incorporate a rank reduction operation, which keeps $r$ at a smaller value, thereby significantly reducing time consumption compared to Algorithm \ref{Alg:SPA}. In summary, our proposed Algorithm \ref{Alg:Rank}, which integrates a validation mechanism within its rank reduction strategy, not only achieves high AUC values and exhibits reduced sensitivity to the parameter $\lambda_4$, but also significantly lowers time consumption.
\begin{figure}[htbp]
	\centering
	\begin{subfigure}[b]{1\linewidth}
		\begin{subfigure}[b]{0.33\linewidth}
			\centering
			\includegraphics[width=\linewidth]{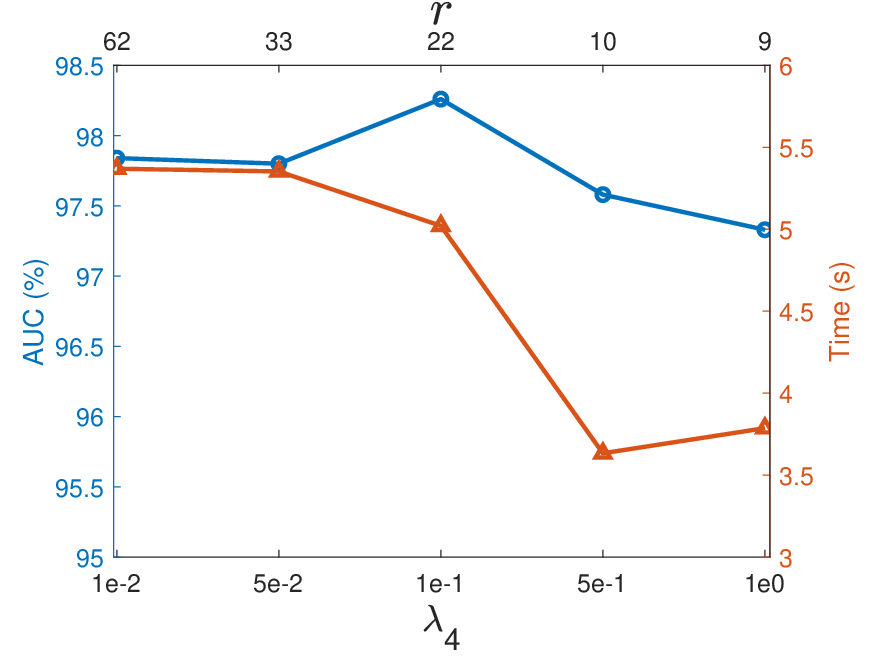}
			\caption{Algorithm \ref{Alg:Rank}}
		\end{subfigure}   	
		\begin{subfigure}[b]{0.33\linewidth}
			\centering
			\includegraphics[width=\linewidth]{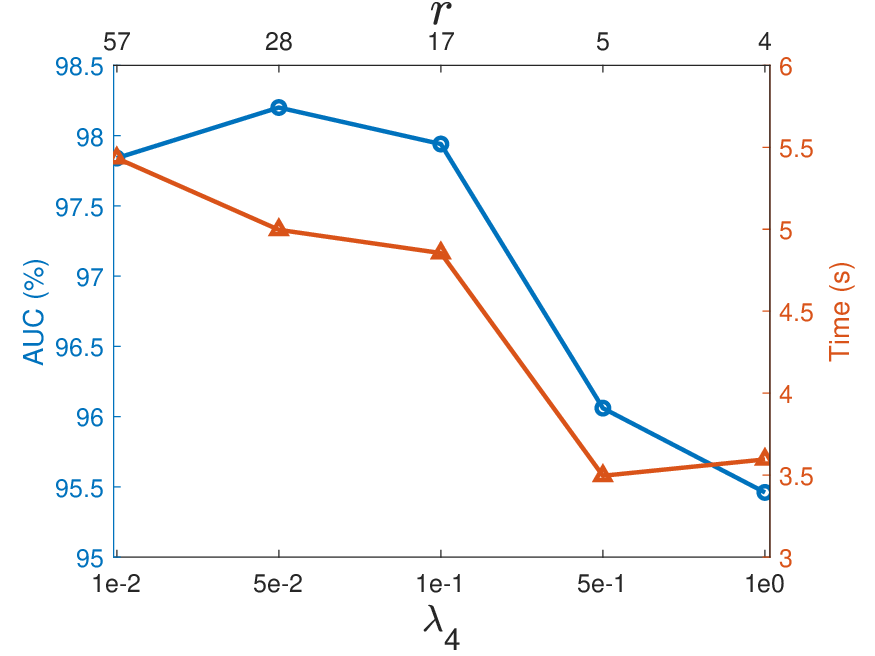}
			\caption{Algorithm \ref{Alg:Rank} without validation mechanism}
		\end{subfigure}
		\begin{subfigure}[b]{0.33\linewidth}
			\centering
			\includegraphics[width=\linewidth]{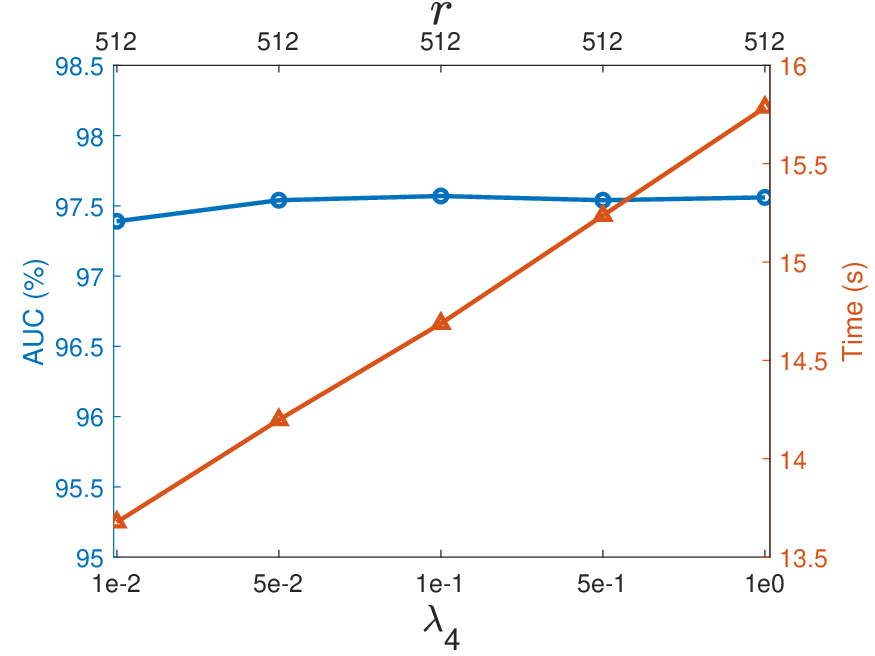}
			\caption{Algorithm \ref{Alg:SPA}}
		\end{subfigure}
	\end{subfigure}
	\vfill
	\caption{Performance analysis of Algorithm \ref{Alg:SPA}, Algorithm \ref{Alg:Rank}, and Algorithm \ref{Alg:Rank} without validation mechanism on AUC values (\%), running time (s), and $r$ metrics across different various values of $\lambda_4$ on the Crack dataset.}
	\label{fig:rv}
\end{figure}

\subsubsection{Effects of spectral-spatial fusion}
Figure \ref{fig:fusion} demonstrates that the anomalies in $T_1$ and $T_2$ exhibit complementary characteristics, suggesting that their fusion can produce a more effective anomaly detection map. To more intuitively illustrate the efficacy of fusing spectral anomaly detection map $T_1$ and spatial anomaly detection map $T_2$, Table \ref{tab:SSF} presents the AUC values for various anomaly detection maps across the ABU dataset. The table indicates that, irrespective of the dataset, the AUC values for the spectral-spatial fusion anomaly detection map $T_1\circ T_2$ consistently surpass those of the spectral anomaly detection map $T_1$ and the spatial anomaly detection map $T_2$. Moreover, the results $T$ derived from the guided image filter further enhance the performance of $T_1\circ T_2$.

\begin{table}[htbp]
	\centering
	\caption{Comparison of AUC values (\%) of different anomaly detection map for ABU dataset.}
	\begin{tabular}{ccccc}
		\toprule
		& Airport1 & Airport2 & Urban1 & Urban2 \\
		\midrule
		$T_1$            & 42.97  & 92.83  & 99.06  & 94.93  \\
		\midrule
		$T_2$            & 96.82  & 99.91  & 99.63  & 94.30  \\
		\midrule
		$T_1\circ T_2$   & 97.58  & 99.95  & 99.92  & 98.29  \\
		\midrule
		$T$              & 99.73  & 99.95  & 99.93  & 98.30  \\
		\bottomrule
	\end{tabular}%
	\label{tab:SSF}%
\end{table}%

\subsubsection{Parameter analysis}

For the proposed model LTD, three main types of parameters, the proximal parameters $ \rho_i, i\in[6] $, the parameter $b$, and the regularization parameters $ \lambda_i, i\in[6] $, affect the performance. Following \cite{LZJ22}, we set the proximal parameters $ \rho_i=1e\text{-}2, i\in[6] $ in all experiments.

The parameter $b$ influences the dimensions of tensor $\mC$ and matrix $B$. We explored values of $b$ ranging from 2 to 38, with increments of 4. Figure \ref{fig:CM} illustrates the AUC values and corresponding running times of LTD for these different values of $b$. As shown in the figure, the AUC values reach their optimum at very small values of $b$. As $b$ increases, the AUC values tend to stabilize. Regarding running times, as $b$ increases, the sizes of tensor $\mC$ and matrix $B$ also increase, resulting in a gradual increase in running times in most cases. Therefore, in our experiments, we set $b$ within the range of $[2,6]$ for all cases.
\begin{figure}[!htbp]
	\centering
	\begin{subfigure}[b]{1\linewidth}
		\begin{subfigure}[b]{0.245\linewidth}
			\centering
			\includegraphics[width=\linewidth]{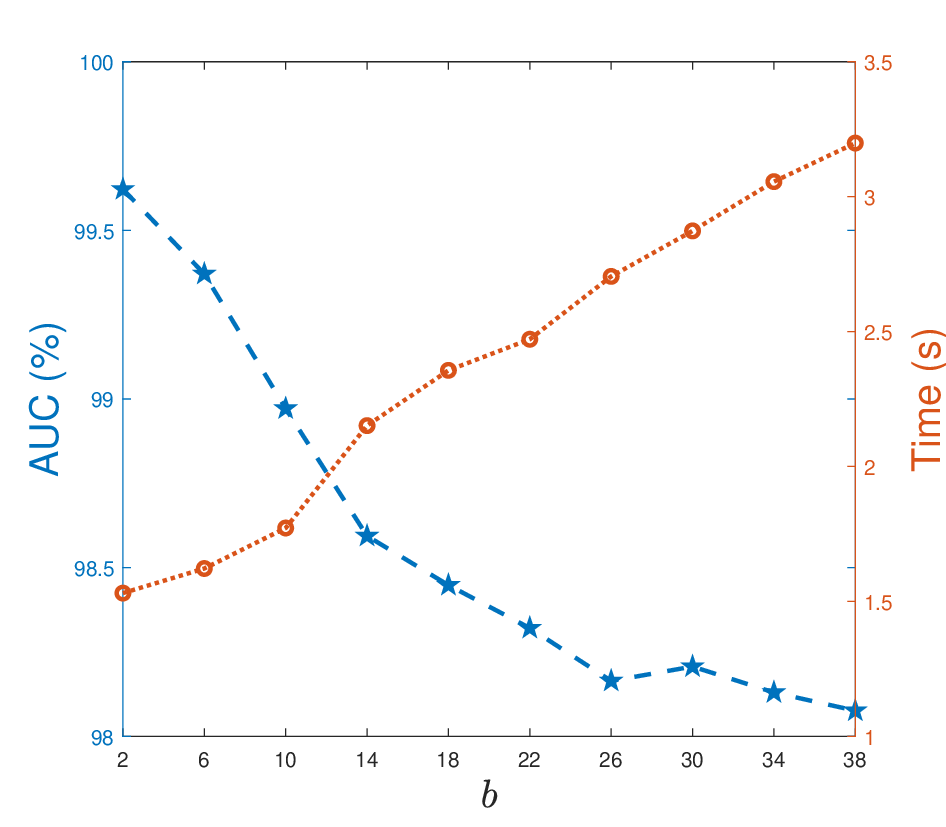}
			\caption{Airport1}
		\end{subfigure}   	
		\begin{subfigure}[b]{0.245\linewidth}
			\centering
			\includegraphics[width=\linewidth]{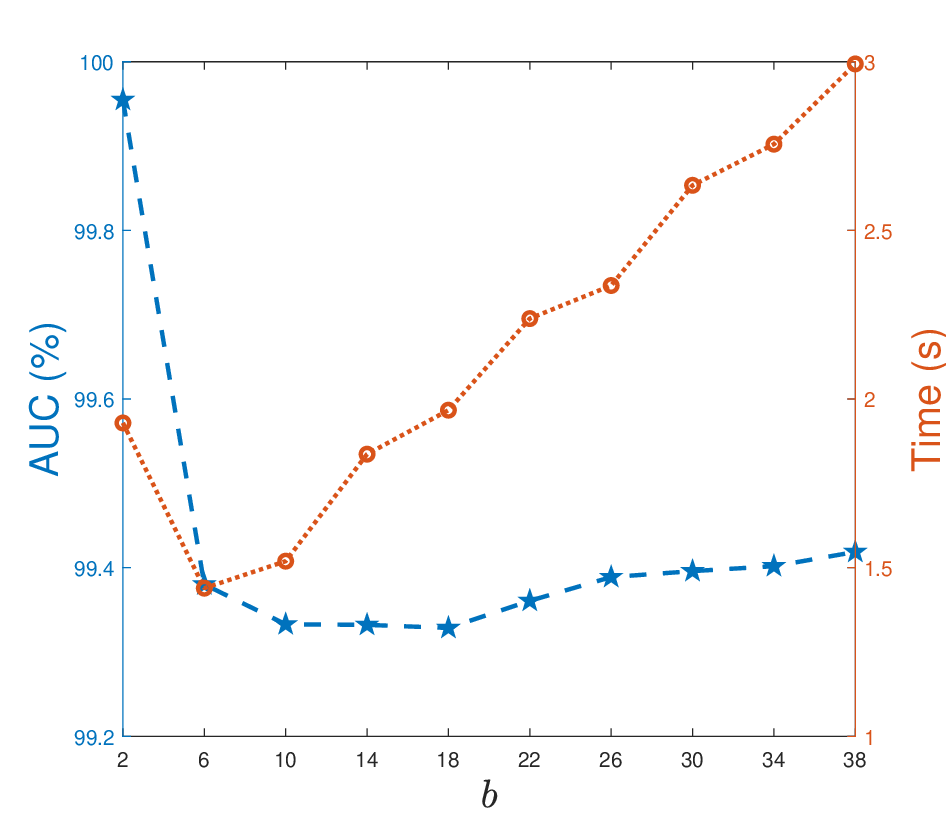}
			\caption{Airport2}
		\end{subfigure}
		\begin{subfigure}[b]{0.245\linewidth}
			\centering
			\includegraphics[width=\linewidth]{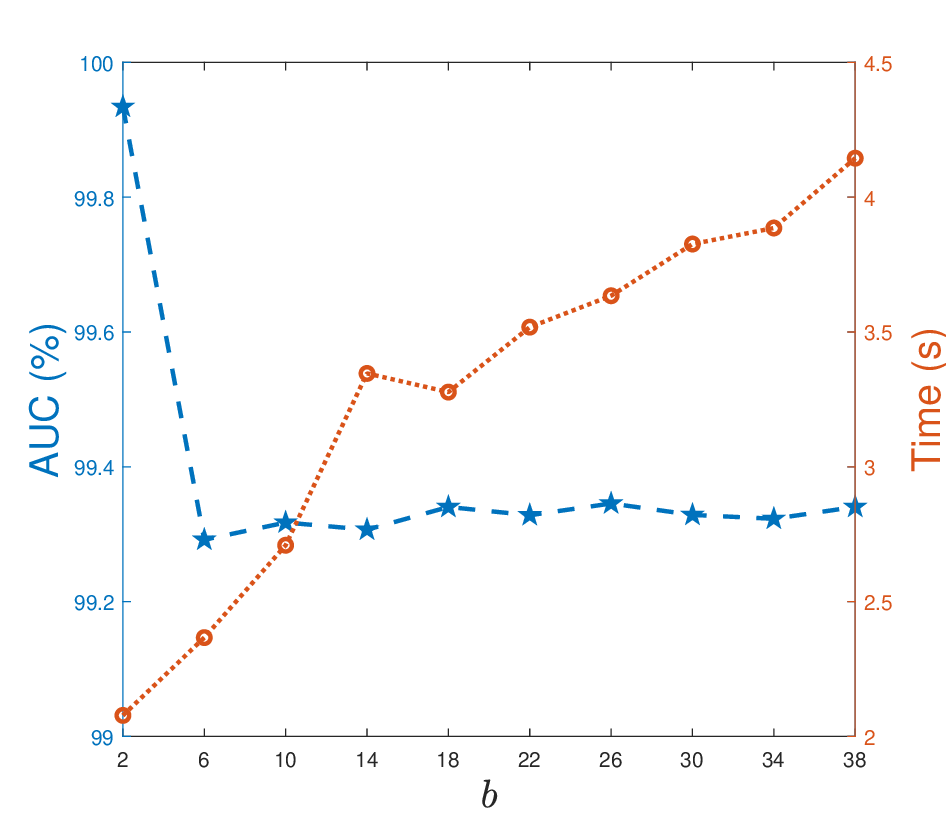}
			\caption{Urban1}
		\end{subfigure}
		\begin{subfigure}[b]{0.245\linewidth}
			\centering
			\includegraphics[width=\linewidth]{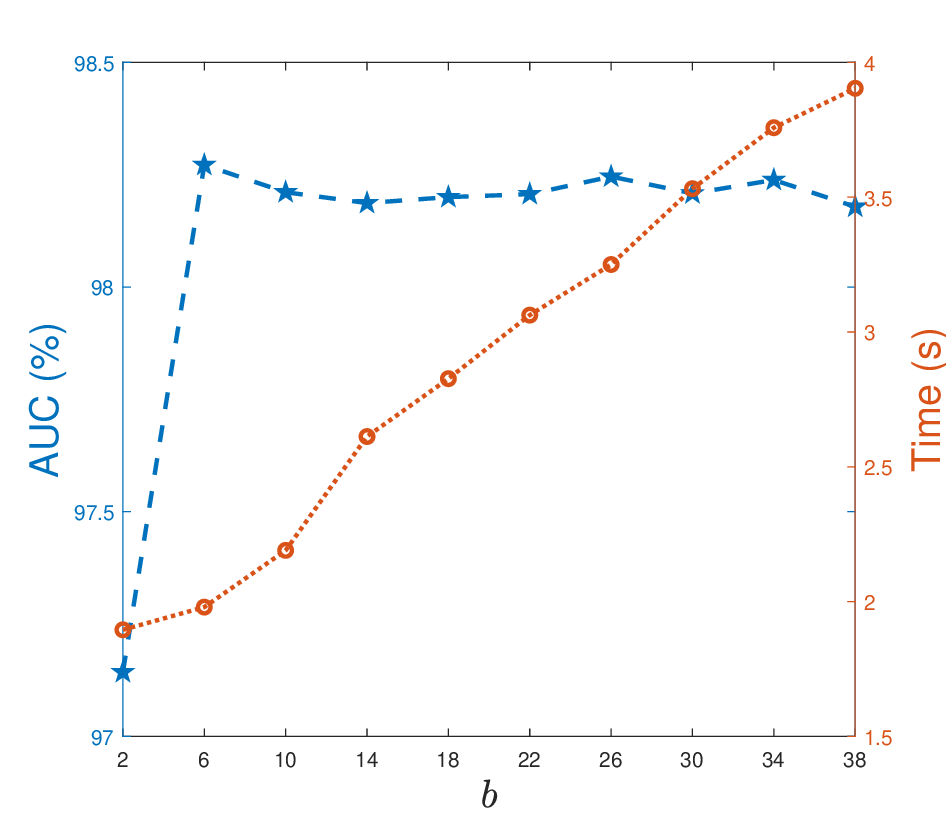}
			\caption{Urban2}
		\end{subfigure}
	\end{subfigure}
	\vfill
	\caption{AUC values (\%) and corresponding running time (s)  of LTD with different $b$ for ABU dataset.}
	\label{fig:CM}
\end{figure}

From Figure \ref{fig:Lambda}, we observe that LTD is not sensitive to the parameter $\lambda_1$ when it is less than or equal to 1. Its primary function is to prevent $B$ from becoming too large; therefore, we set $\lambda_1 = 1e\text{-}2$ in our experiments. For $\lambda_2$, we set it to 5, and for $\lambda_5$, we set it to $1e\text{-}1$, as these values yielded optimal results across all datasets.
Regarding $\lambda_3$, we select the best values from $\{1e\text{-}2, 1e\text{-}1, 5e\text{-}1, 1\}$ based on the AUC values for different datasets, and then fine-tune them to achieve the optimal results. With respect to $\lambda_4$, we found that LTD's performance is not significantly affected when it is set to $1e\text{-}1$ or higher. Therefore, we set $\lambda_4 = 5e\text{-}1$ in our experiments. For the ABU dataset, we observed that the optimal AUC values for $\lambda_3$ and $\lambda_6$ exhibit a multiplicative relationship, with $\lambda_6$ being $\lambda_3/10$. Therefore, we set $\lambda_6 = \lambda_3/10$. Similarly, for the MVTec dataset, we set $\lambda_6 = \lambda_3/100$ to achieve the optimal AUC values.

\begin{figure}[htbp]
	\centering
	\begin{subfigure}[b]{1\linewidth}
		\begin{subfigure}[b]{0.322\linewidth}
			\centering
			\includegraphics[width=\linewidth]{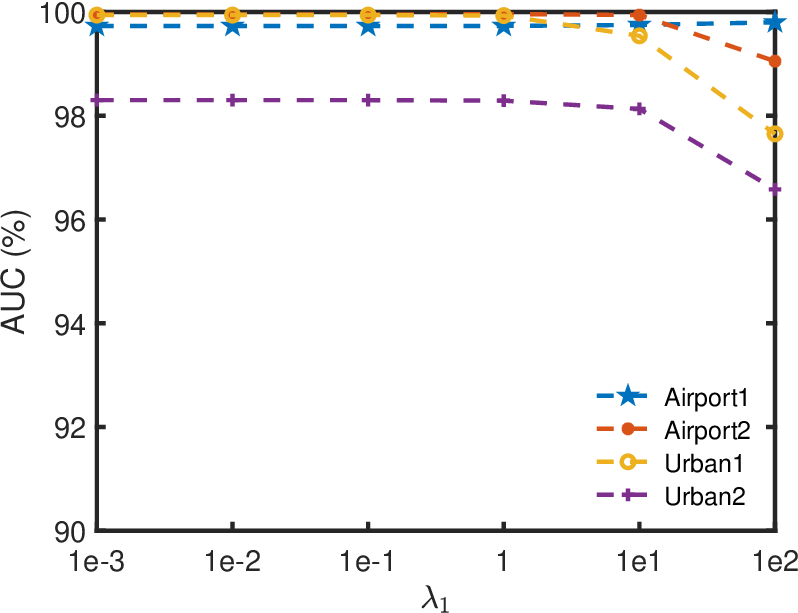}
		\end{subfigure}   	
		\begin{subfigure}[b]{0.322\linewidth}
			\centering
			\includegraphics[width=\linewidth]{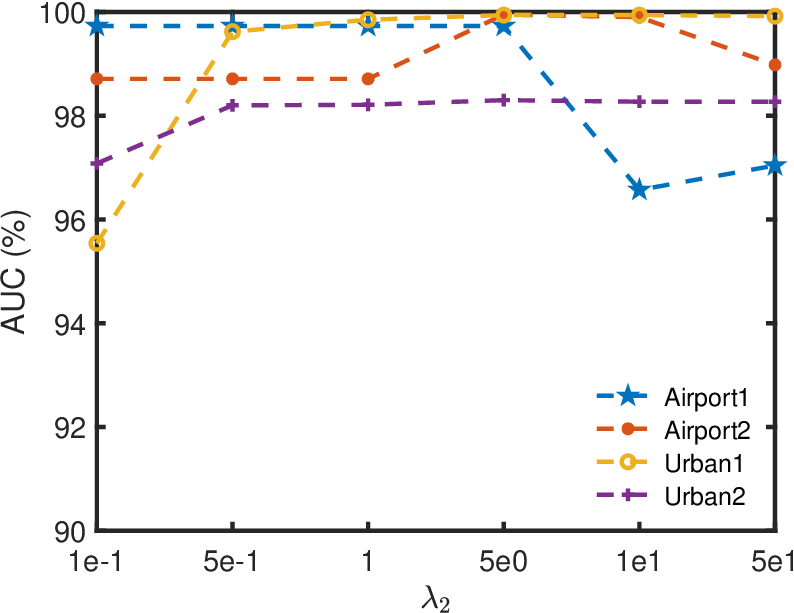}
		\end{subfigure}
		\begin{subfigure}[b]{0.322\linewidth}
			\centering
			\includegraphics[width=\linewidth]{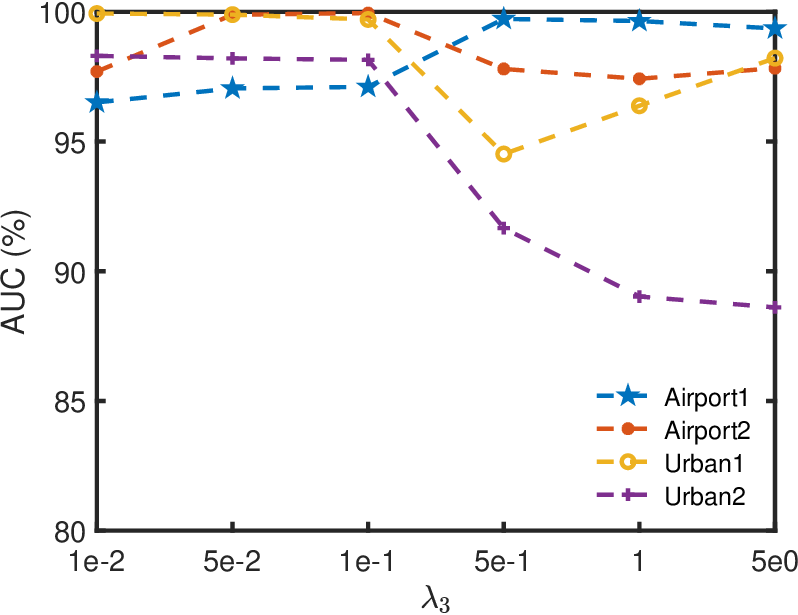}
		\end{subfigure}
		
		\begin{subfigure}[b]{0.322\linewidth}
			\centering
			\includegraphics[width=\linewidth]{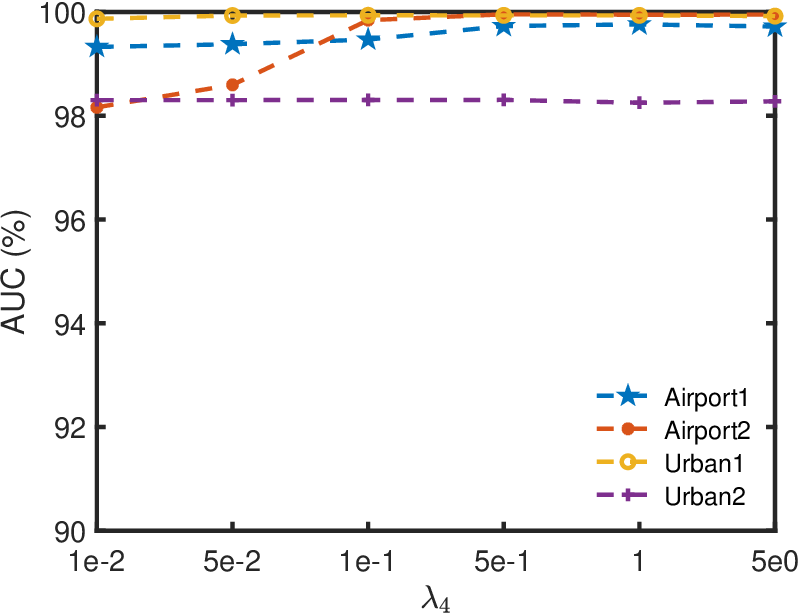}
		\end{subfigure}
		\begin{subfigure}[b]{0.322\linewidth}
			\centering
			\includegraphics[width=\linewidth]{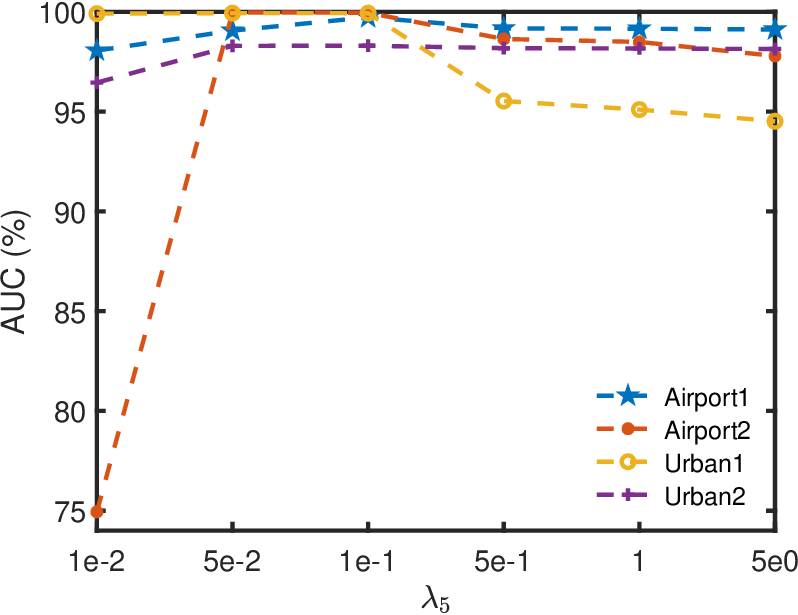}
		\end{subfigure}
		\begin{subfigure}[b]{0.322\linewidth}
			\centering
			\includegraphics[width=\linewidth]{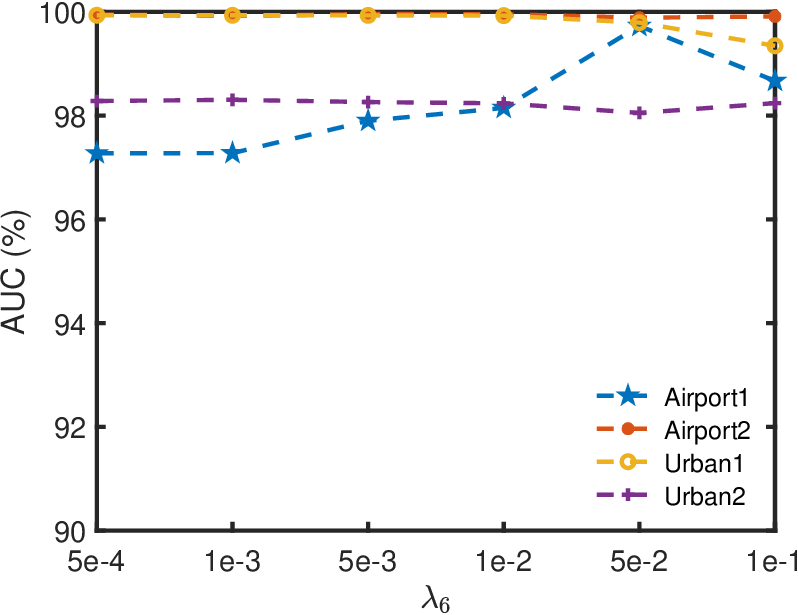}
		\end{subfigure}
	\end{subfigure}
	\vfill
	\caption{AUC values (\%) of LTD with different $\lambda_i~(i\in[6])$ for ABU dataset.}
	\label{fig:Lambda}
\end{figure}


\subsection{Detection performance}

\subsubsection{ABU dataset}
For the ABU dataset, the reference map and detection results in visual are shown in Figures. \ref{fig:HSIs}-\ref{fig:2D}.
The proposed method, LTD, effectively detects all anomalies while including fewer background pixels compared to other methods. Specifically, although the anomaly detection maps produced by RX, RPCA, LRASR, and TPCA contain few background pixels, they fail to detect many anomalies or detect them weakly. Conversely, the anomaly detection maps generated by PTA, TLRSR, RGAE, and GAED can identify most anomalies but include a significant amount of background pixels, such as the red area in the lower left corner of Airport1 and the dashed lines at the top of Airport2. In contrast, the anomaly detection map produced by our proposed LTD method successfully detects nearly all anomalies while containing minimal background pixels.
\begin{figure}[htbp]
	\centering
	\begin{subfigure}[b]{1\linewidth}
		\begin{minipage}{0.245\textwidth}
			\centering
			\includegraphics[width=0.485\textwidth]{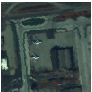}\vspace{0pt}
			\includegraphics[width=0.485\textwidth]{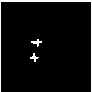}
			\caption{Airport1}
		\end{minipage}
		\begin{minipage}{0.245\textwidth}
			\centering
			\includegraphics[width=0.485\textwidth]{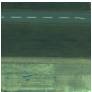}\vspace{0pt}
			\includegraphics[width=0.485\textwidth]{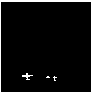}
			\caption{Airport2}
		\end{minipage}
		\begin{minipage}{0.245\textwidth}
			\centering
			\includegraphics[width=0.485\textwidth]{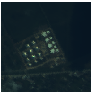}\vspace{0pt}
			\includegraphics[width=0.485\textwidth]{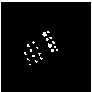}
			\caption{Urban1}
		\end{minipage}
		\begin{minipage}{0.245\textwidth}
			\centering
			\includegraphics[width=0.485\textwidth]{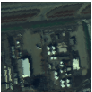}\vspace{0pt}
			\includegraphics[width=0.485\textwidth]{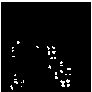}
			\caption{Urban2}
		\end{minipage}
	\end{subfigure}
	\vfill
	\caption{Pseudo-color images and ground-truth maps of ABU dataset.}
	\label{fig:HSIs}
\end{figure}

\begin{figure}[htbp]
	\centering
	\begin{subfigure}[b]{1\linewidth} 
		\begin{subfigure}[b]{1\linewidth}
			\centering
			\includegraphics[width=\linewidth]{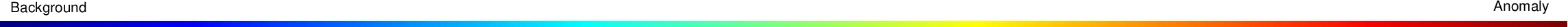}
		\end{subfigure}
		\begin{subfigure}[b]{0.106\linewidth}
			\centering
			\includegraphics[width=\linewidth]{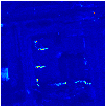}
			\includegraphics[width=\linewidth]{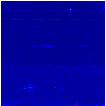}
			\includegraphics[width=\linewidth]{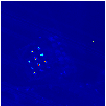}
			\includegraphics[width=\linewidth]{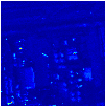}
			\caption*{RX}
		\end{subfigure}
		\begin{subfigure}[b]{0.106\linewidth}
			\centering
			\includegraphics[width=\linewidth]{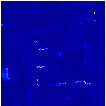}
			\includegraphics[width=\linewidth]{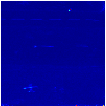}
			\includegraphics[width=\linewidth]{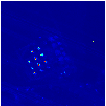}
			\includegraphics[width=\linewidth]{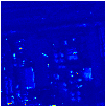}
			\caption*{RPCA}
		\end{subfigure}
		\begin{subfigure}[b]{0.106\linewidth}
			\centering
			\includegraphics[width=\linewidth]{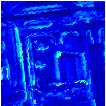}
			\includegraphics[width=\linewidth]{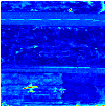}
			\includegraphics[width=\linewidth]{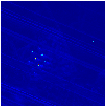}
			\includegraphics[width=\linewidth]{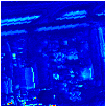}
			\caption*{LRASR}
		\end{subfigure}
		\begin{subfigure}[b]{0.106\linewidth}
			\centering
			\includegraphics[width=\linewidth]{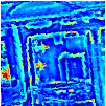}
			\includegraphics[width=\linewidth]{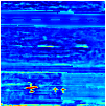}
			\includegraphics[width=\linewidth]{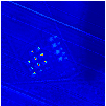}
			\includegraphics[width=\linewidth]{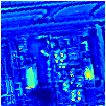}
			\caption*{PTA}
		\end{subfigure}
		\begin{subfigure}[b]{0.106\linewidth}
			\centering
			\includegraphics[width=\linewidth]{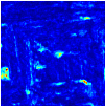}
			\includegraphics[width=\linewidth]{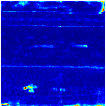}
			\includegraphics[width=\linewidth]{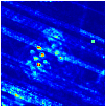}
			\includegraphics[width=\linewidth]{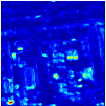}
			\caption*{TPCA}
		\end{subfigure}  	
		\begin{subfigure}[b]{0.106\linewidth}
			\centering
			\includegraphics[width=\linewidth]{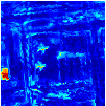}
			\includegraphics[width=\linewidth]{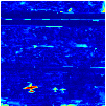}
			\includegraphics[width=\linewidth]{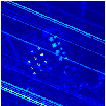}
			\includegraphics[width=\linewidth]{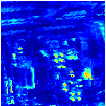}
			\caption*{TLRSR}
		\end{subfigure}
		\begin{subfigure}[b]{0.106\linewidth}
			\centering
			\includegraphics[width=\linewidth]{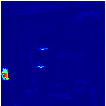}
			\includegraphics[width=\linewidth]{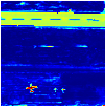}
			\includegraphics[width=\linewidth]{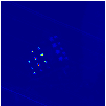}
			\includegraphics[width=\linewidth]{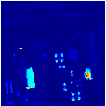}
			\caption*{RGAE}
		\end{subfigure}   	
		\begin{subfigure}[b]{0.106\linewidth}
			\centering
			\includegraphics[width=\linewidth]{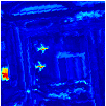}
			\includegraphics[width=\linewidth]{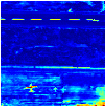}
			\includegraphics[width=\linewidth]{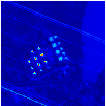}
			\includegraphics[width=\linewidth]{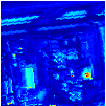}
			\caption*{GAED}
		\end{subfigure}
		\begin{subfigure}[b]{0.106\linewidth}
			\centering
			\includegraphics[width=\linewidth]{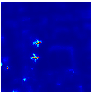}
			\includegraphics[width=\linewidth]{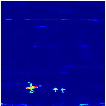}
			\includegraphics[width=\linewidth]{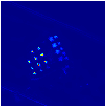}
			\includegraphics[width=\linewidth]{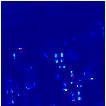}
			\caption*{LTD}
		\end{subfigure}
	\end{subfigure}
	\vfill
	\caption{Target detection results by different methods for ABU dataset.}
	\label{fig:2D}
\end{figure}

Table \ref{tab:AUC_ABU} provides a detailed comparison of the AUC values and running time for each method.
On the AUC metric, the proposed LTD method achieves the best results, consistent with the observations in Figure \ref{fig:2D}. Regarding running time, our method is the second fastest. Although RX exhibits shorter running times, its AUC values are unsatisfactory in most scenarios. In contrast, other methods have longer running times compared to ours, with the deep learning based method RGAE taking approximately 50 seconds, which is about 25 times the running time of our proposed method. Moreover, the AUC values of these methods are not as high as those achieved by the LTD method.


Figure \ref{fig:ROC} provides a detailed comparison of the ROC curves for each method. It shows that our proposed LTD method has a higher detection probability than other methods for most false alarm rates. Specifically, for the Airport dataset, our method maintains a higher detection probability across all false alarm rates.

\begin{table}[htbp]
	\centering
	\caption{Comparison of AUC values (\%) and running time (s) of different methods for ABU dataset.}\resizebox{\linewidth}{!}{
	\begin{tabular}{ccccccccccc}
		\toprule
		Algorithm & Index & RX    & RPCA  & LRASR & PTA   & TPCA  & TLRSR & RGAE  & GAED  & LTD \\
		\midrule
		\multirow{2}{*}{Airport1} & AUC   & 84.04  & 84.28  & 87.70  & 90.96  & 88.11  & 94.58  & 96.40  & 96.77  & \textbf{99.73 } \\
		& Time & 0.05  & 4.62  & 20.06  & 15.77  & 18.63  & 2.50  & 51.88  & 41.10  & 1.66  \\
		\midrule
		\multirow{2}{*}{Airport2} & AUC   & 95.26  & 96.27  & 97.95  & 99.55  & 95.26  & 99.49  & 93.25  & 96.81  & \textbf{99.95 } \\
		& Time & 0.06  & 1.93  & 19.45  & 14.68  & 17.93  & 2.28  & 50.00  & 39.15  & 1.89  \\
		\midrule
		\multirow{2}{*}{Urban1} & AUC   & 99.46  & 99.57  & 87.03  & 97.70  & 95.45  & 95.42  & 99.73  & 99.59  & \textbf{99.93 } \\
		& Time & 0.05  & 3.98  & 22.87  & 16.69  & 17.98  & 2.40  & 53.78  & 39.84  & 2.13  \\
		\midrule
		\multirow{2}{*}{Urban2} & AUC   & 96.92  & 96.58  & 89.13  & 82.58  & 93.00  & 97.11  & 94.90  & 90.18  & \textbf{98.30 } \\
		& Time & 0.06  & 3.85  & 20.30  & 15.99  & 19.32  & 2.49  & 42.26  & 37.99  & 1.89  \\
		\bottomrule
	\end{tabular}}%
	\label{tab:AUC_ABU}%
\end{table}%

\begin{figure}[htbp]
	\centering
	\begin{subfigure}[b]{1\linewidth}
		\begin{subfigure}[b]{1\linewidth}
			\centering
			\includegraphics[width=0.95\linewidth]{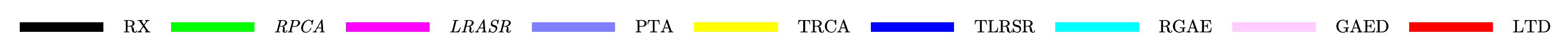}
		\end{subfigure} 
		\begin{subfigure}[b]{0.245\linewidth}
			\centering
			\includegraphics[width=\linewidth]{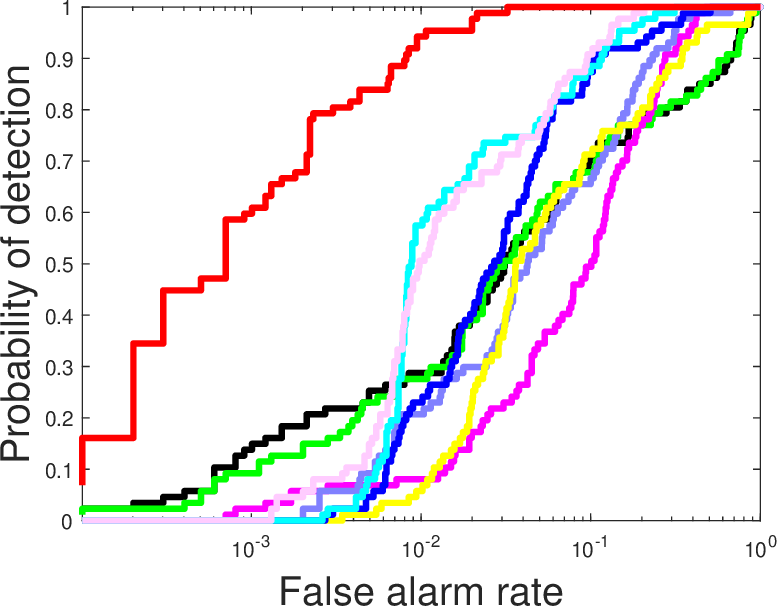}
			\caption{Airport1}
		\end{subfigure}   	
		\begin{subfigure}[b]{0.245\linewidth}
			\centering
			\includegraphics[width=\linewidth]{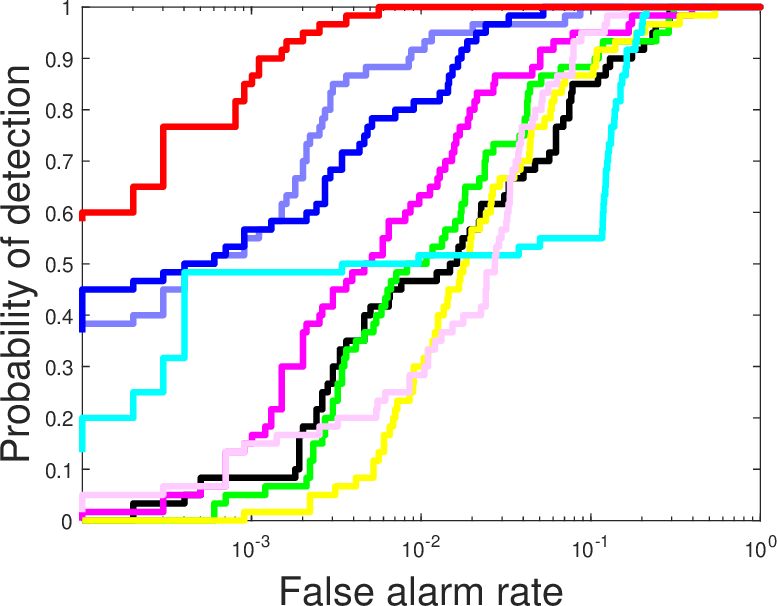}
			\caption{Airport2}
		\end{subfigure}
		\begin{subfigure}[b]{0.245\linewidth}
			\centering
			\includegraphics[width=\linewidth]{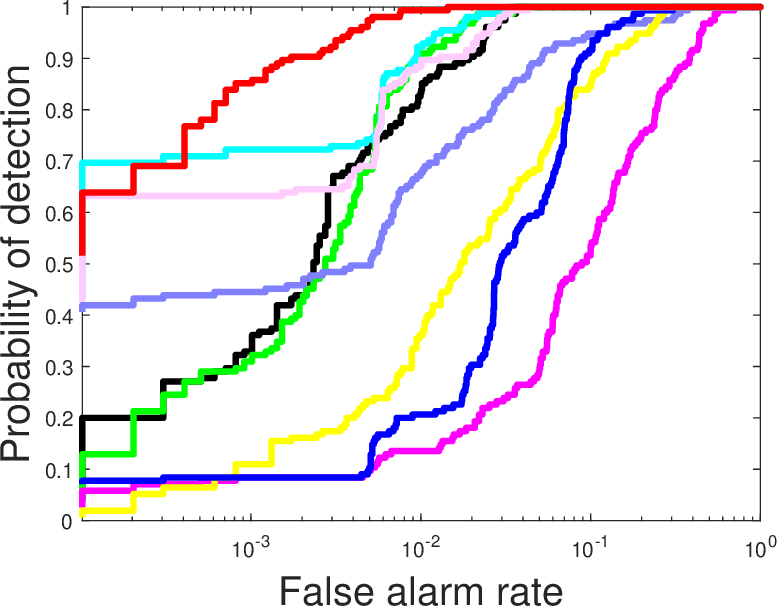}
			\caption{Urban1}
		\end{subfigure}
		\begin{subfigure}[b]{0.245\linewidth}
			\centering
			\includegraphics[width=\linewidth]{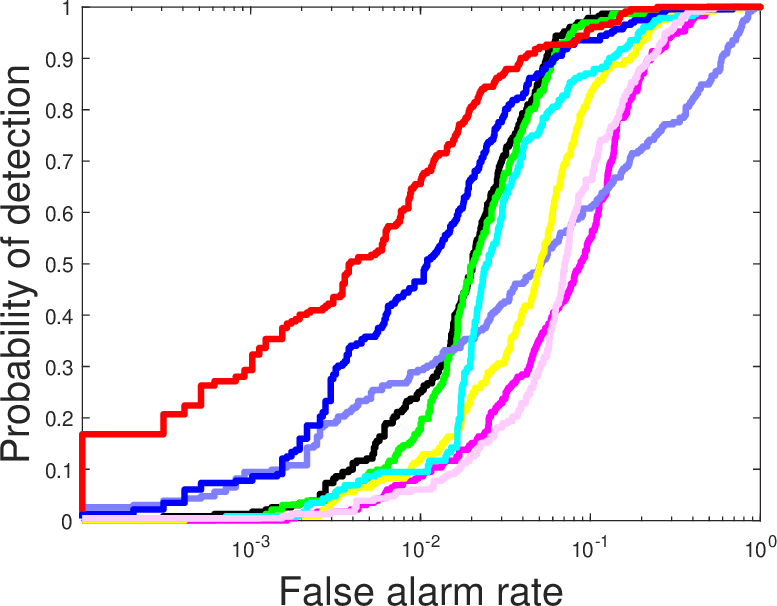}
			\caption{Urban2}
		\end{subfigure}
	\end{subfigure}
	\vfill
	\caption{ROC curves obtained by different methods for ABU dataset.}
	\label{fig:ROC}
\end{figure}

To further evaluate the performance of LTD, we utilize normalized background-anomaly separation maps. A larger distance between the anomaly and background boxes indicates superior separation performance, while a shorter background box reflects more effective suppression of background information. As illustrated in Figure \ref{fig:AB}, the LTD method consistently demonstrates both shorter background boxes and greater separation between anomaly and background boxes across all scenarios. In contrast, some comparative algorithms only exhibit shorter background boxes for specific datasets. For instance, the RX algorithm achieves shorter background boxes solely in the Urban1 dataset but not in others. Similarly, certain algorithms achieve greater separation between anomaly and background boxes only in select datasets. For example, the PTA algorithm shows greater separation exclusively in the Airport2 dataset, while in other datasets, the anomaly and background boxes remain closely spaced.

\begin{figure}[htbp]
	\centering
	\begin{minipage}{0.9\textwidth}
		\begin{subfigure}[b]{0.245\linewidth}
			\centering
			\includegraphics[width=\linewidth]{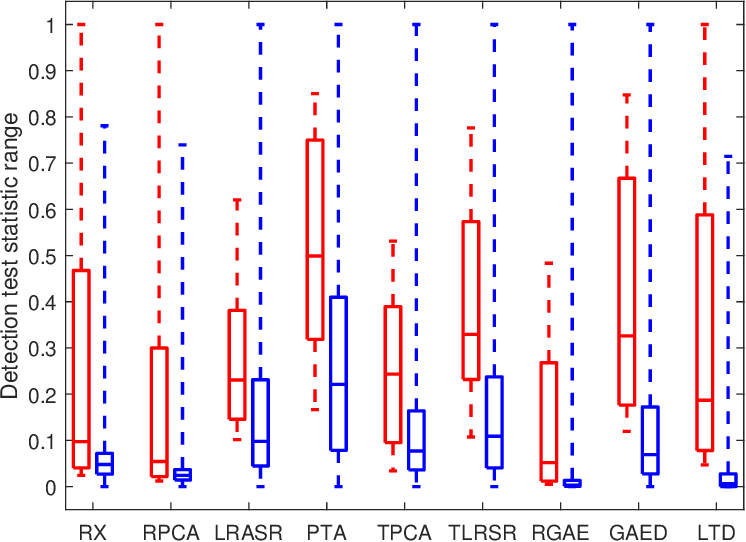}
			\caption{Airport1}
		\end{subfigure}   	
		\begin{subfigure}[b]{0.245\linewidth}
			\centering
			\includegraphics[width=\linewidth]{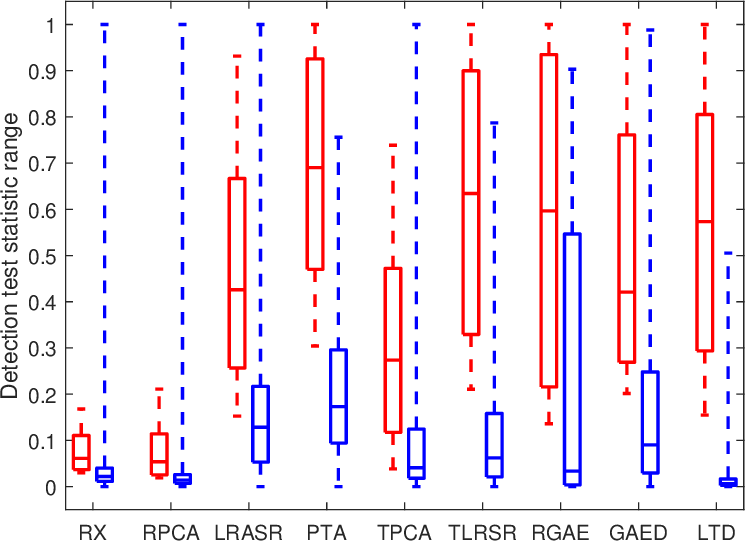}
			\caption{Airport2}
		\end{subfigure}
		\begin{subfigure}[b]{0.245\linewidth}
			\centering
			\includegraphics[width=\linewidth]{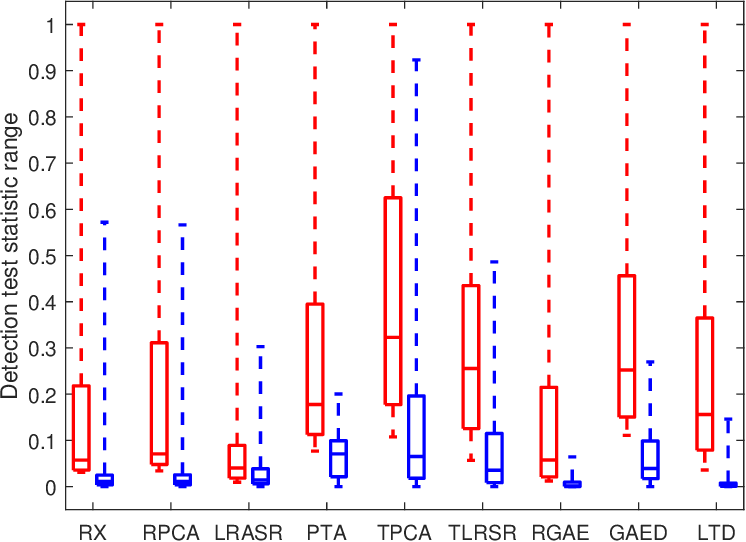}
			\caption{Urban1}
		\end{subfigure}
		\begin{subfigure}[b]{0.245\linewidth}
			\centering
			\includegraphics[width=\linewidth]{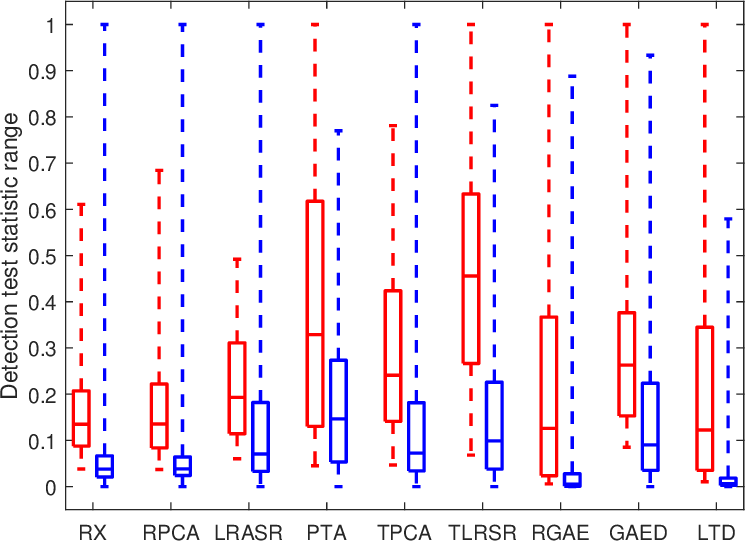}
			\caption{Urban2}
		\end{subfigure}
	\end{minipage}\hfill
	\begin{minipage}{0.1\textwidth}
		\centering
		\includegraphics[width=\linewidth]{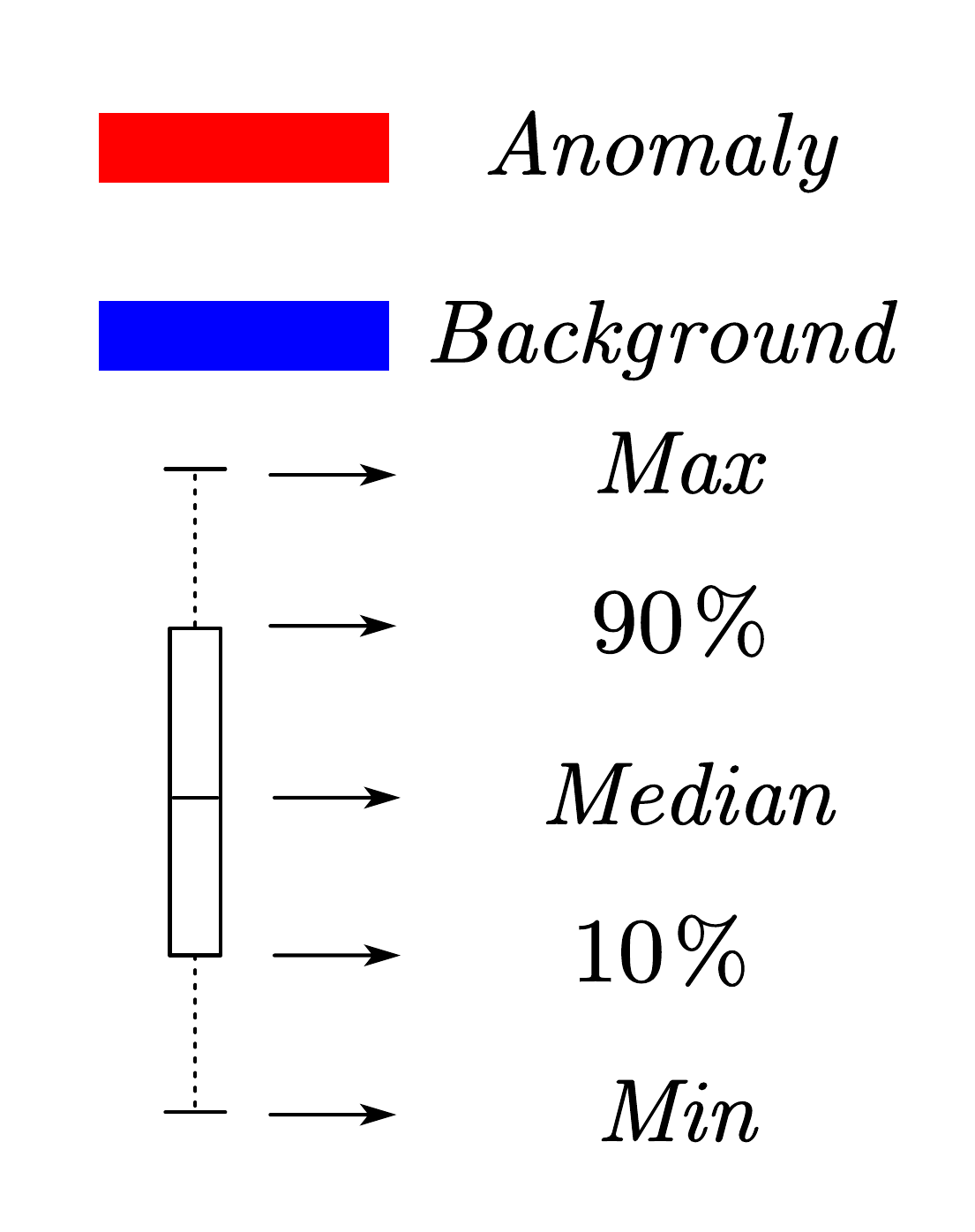}
	\end{minipage}
	\vfill
	\caption{Separability maps of different methods for ABU dataset.}
	\label{fig:AB}
\end{figure}

\subsubsection{MVTec dataset}

For the MVTec dataset, the detection results of the compared methods are illustrated in Figures \ref{fig:MVTec} and \ref{fig:2D:MVTec}. The corresponding AUC values and running time are provided in Table \ref{tab:MVTec}.
One can see that the detection maps produced by the proposed LTD method closely resemble the reference maps, outperforming other methods. Notably, the detection map contains less background noise compared to the RX method, demonstrating the effectiveness of the proposed LTD method in suppressing background.
In addition, as shown in Table \ref{tab:MVTec}, the AUC values of our proposed LTD method show significant improvements over other methods. Specifically, the AUC values are higher by 17.90\%, 9.03\%, 3.66\%, and 2.91\% across four datasets compared to the second-highest values. Regarding running time, our method, which employs a rank reduction strategy with validation mechanism, achieves an average running time of approximately $4$ seconds, significantly lower than that of deep learning based methods.

\begin{figure}[htbp]
	\centering
	\begin{subfigure}[b]{1\linewidth}
		\begin{minipage}{0.245\textwidth}
			\centering
			\includegraphics[width=0.485\textwidth]{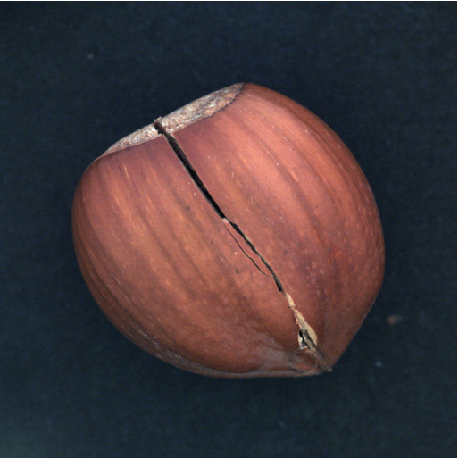}\vspace{0pt}
			\includegraphics[width=0.485\textwidth]{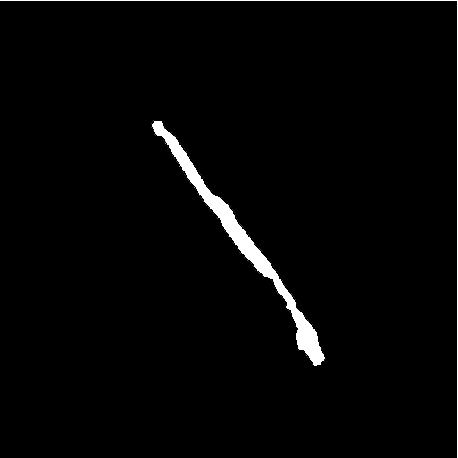}
			\caption{Crack}
		\end{minipage}
		\begin{minipage}{0.245\textwidth}
			\centering
			\includegraphics[width=0.485\textwidth]{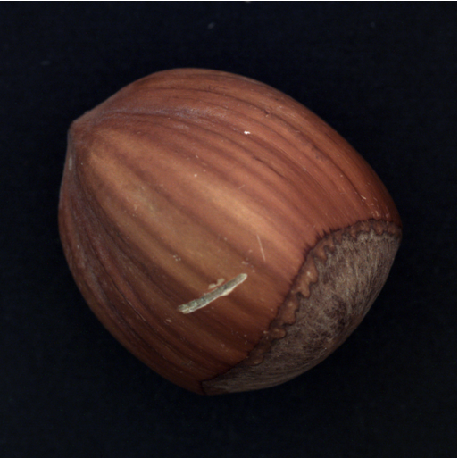}\vspace{0pt}
			\includegraphics[width=0.485\textwidth]{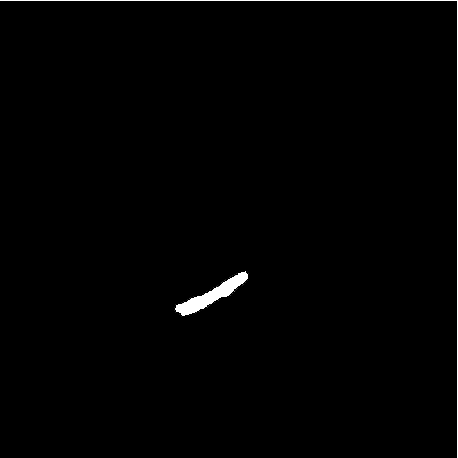}
			\caption{Cut}
		\end{minipage}
		\begin{minipage}{0.245\textwidth}
			\centering
			\includegraphics[width=0.485\textwidth]{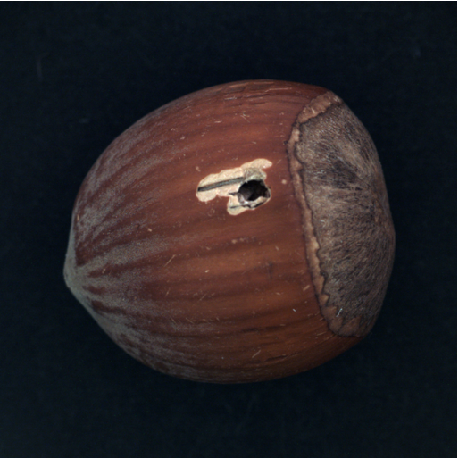}\vspace{0pt}
			\includegraphics[width=0.485\textwidth]{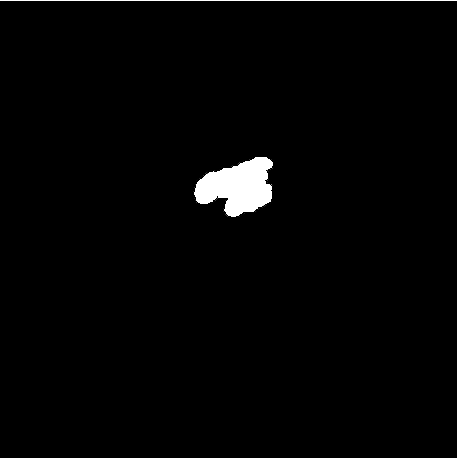}
			\caption{Hole}
		\end{minipage}
		\begin{minipage}{0.245\textwidth}
			\centering
			\includegraphics[width=0.485\textwidth]{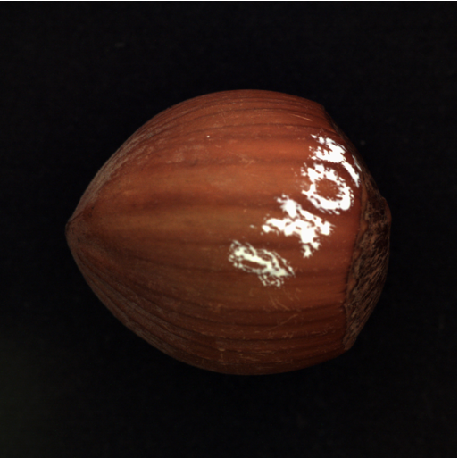}\vspace{0pt}
			\includegraphics[width=0.485\textwidth]{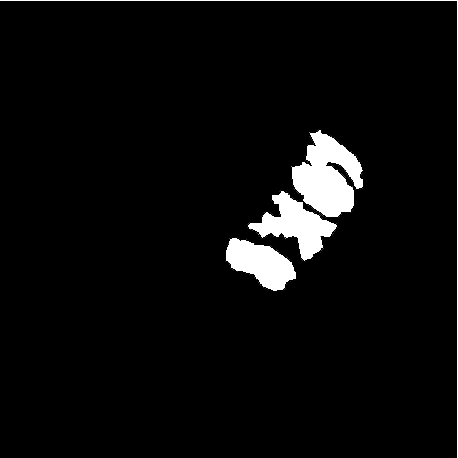}
			\caption{Print}
		\end{minipage}
	\end{subfigure}
	\vfill
	\caption{Color images and ground-truth maps of MVTec dataset.}
	\label{fig:MVTec}
\end{figure}

\begin{figure}[htbp]
	\centering
	\begin{subfigure}[b]{1\linewidth} 
		\begin{subfigure}[b]{1\linewidth}
			\centering
			\includegraphics[width=\linewidth]{AB}
		\end{subfigure}
		\begin{subfigure}[b]{0.106\linewidth}
			\centering
			\includegraphics[width=\linewidth]{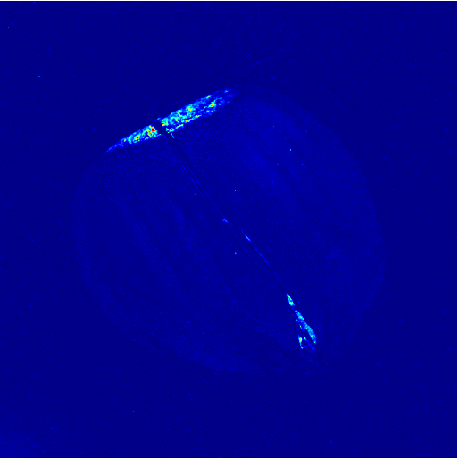}
			\includegraphics[width=\linewidth]{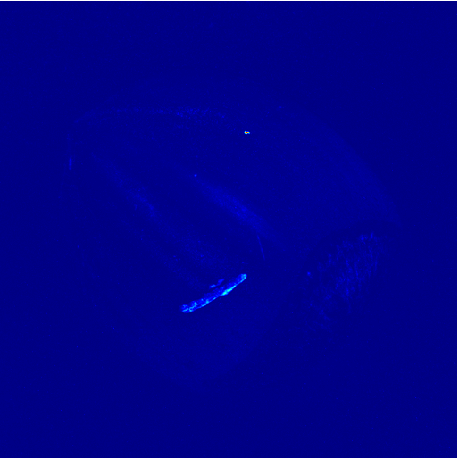}
			\includegraphics[width=\linewidth]{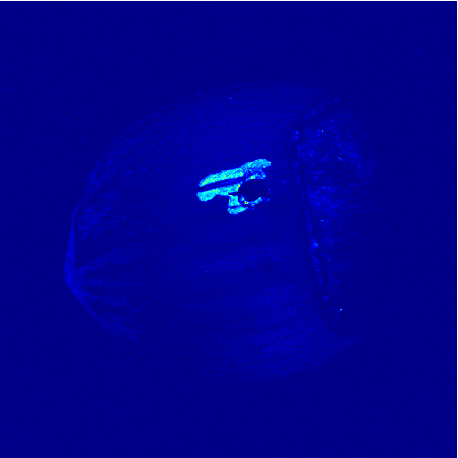}
			\includegraphics[width=\linewidth]{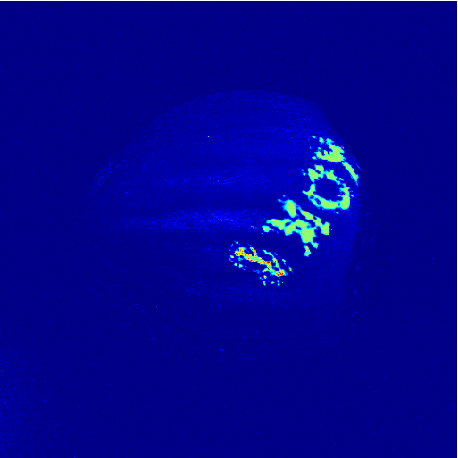}
			\caption*{RX}
		\end{subfigure}
		\begin{subfigure}[b]{0.106\linewidth}
			\centering
			\includegraphics[width=\linewidth]{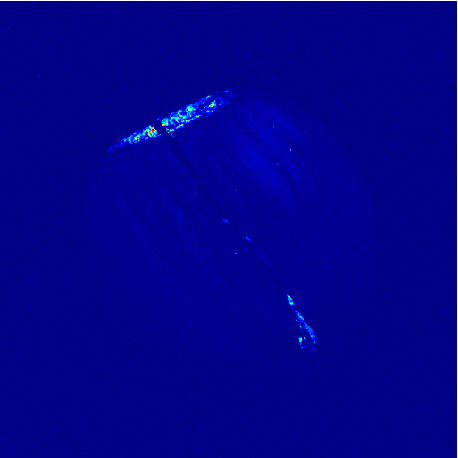}
			\includegraphics[width=\linewidth]{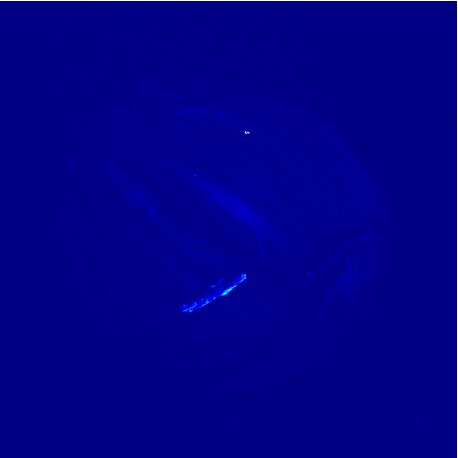}
			\includegraphics[width=\linewidth]{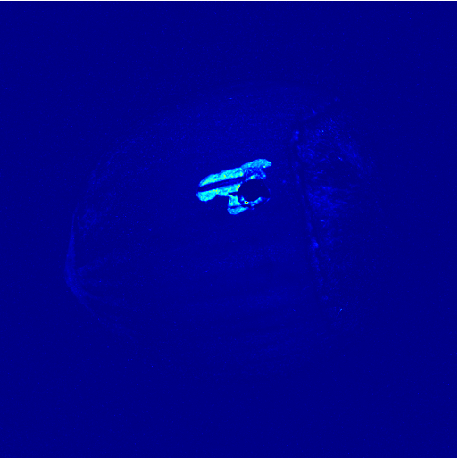}
			\includegraphics[width=\linewidth]{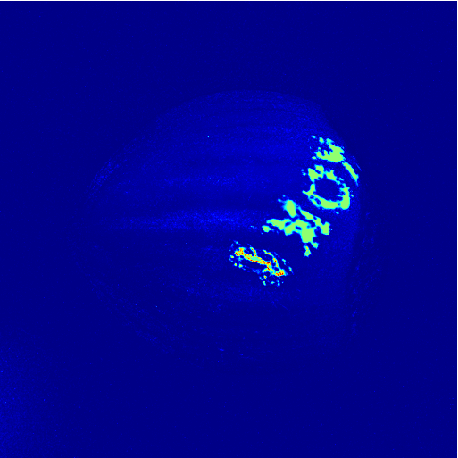}
			\caption*{RPCA}
		\end{subfigure}
		\begin{subfigure}[b]{0.106\linewidth}
			\centering
			\includegraphics[width=\linewidth]{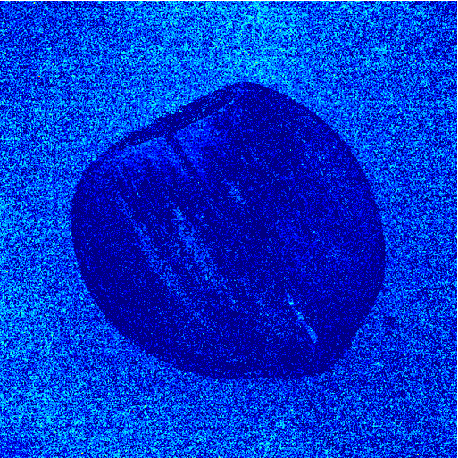}
			\includegraphics[width=\linewidth]{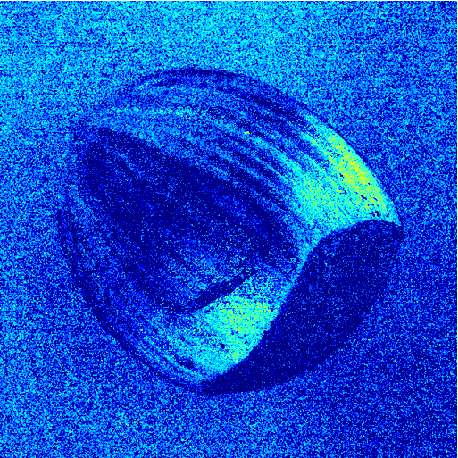}
			\includegraphics[width=\linewidth]{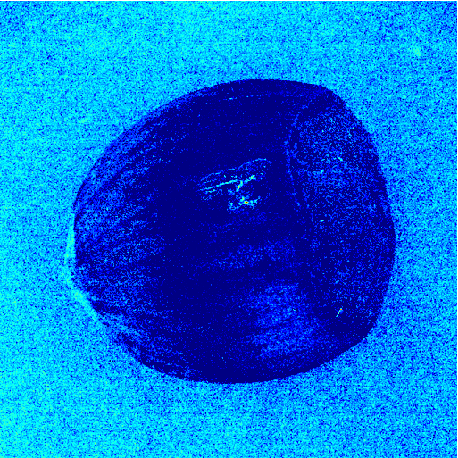}
			\includegraphics[width=\linewidth]{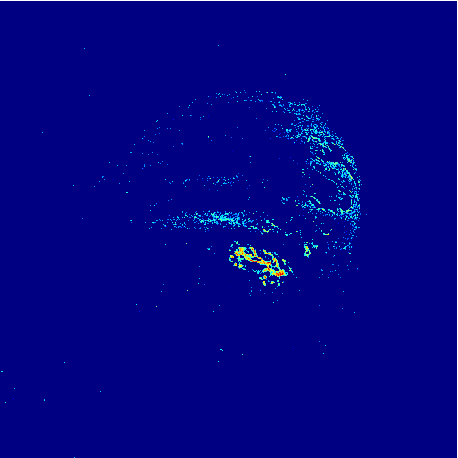}
			\caption*{LRASR}
		\end{subfigure}
		\begin{subfigure}[b]{0.106\linewidth}
			\centering
			\includegraphics[width=\linewidth]{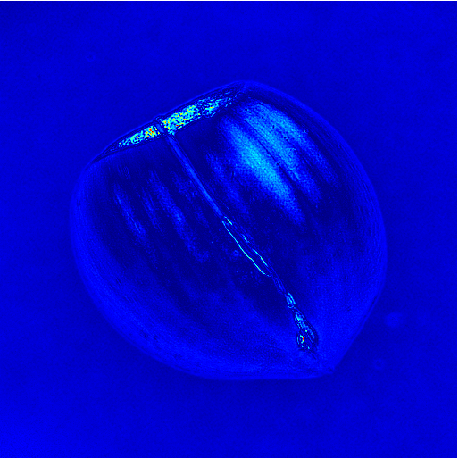}
			\includegraphics[width=\linewidth]{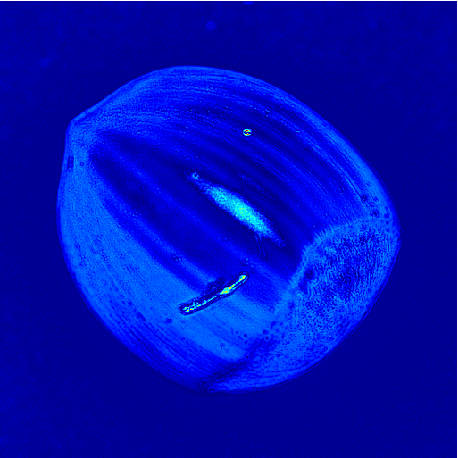}
			\includegraphics[width=\linewidth]{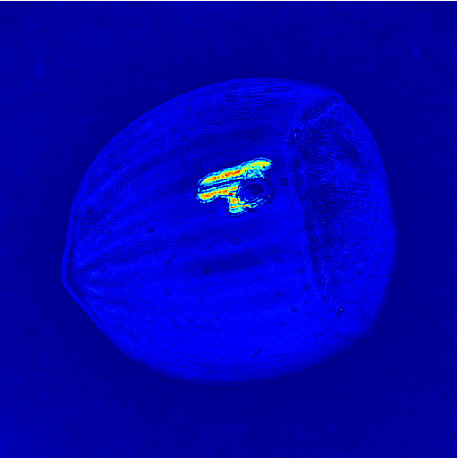}
			\includegraphics[width=\linewidth]{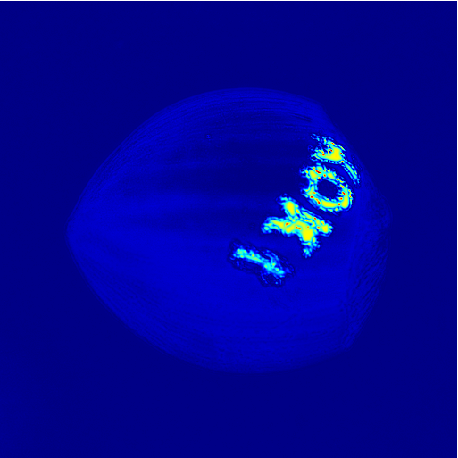}
			\caption*{PTA}
		\end{subfigure}
		\begin{subfigure}[b]{0.106\linewidth}
			\centering
			\includegraphics[width=\linewidth]{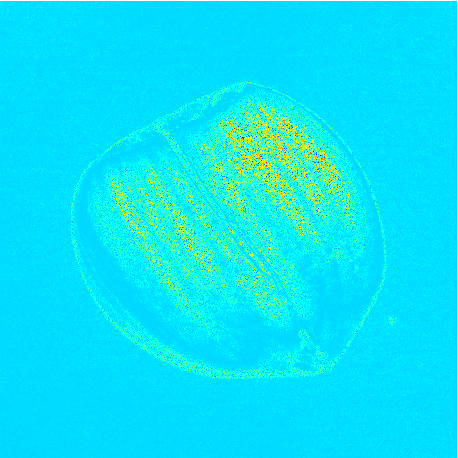}
			\includegraphics[width=\linewidth]{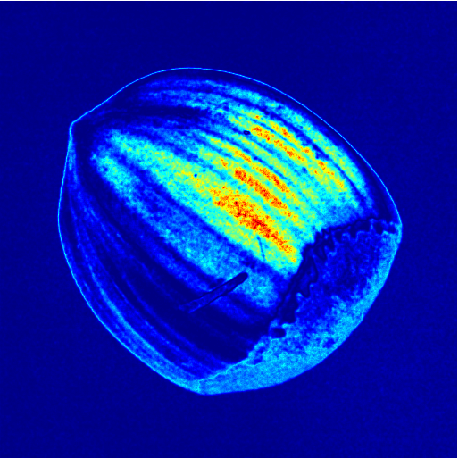}
			\includegraphics[width=\linewidth]{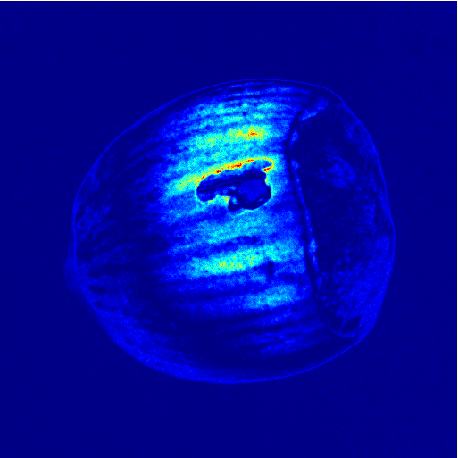}
			\includegraphics[width=\linewidth]{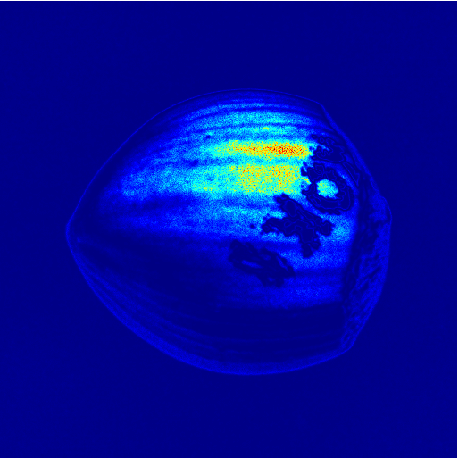}
			\caption*{TPCA}
		\end{subfigure}  	
		\begin{subfigure}[b]{0.106\linewidth}
			\centering
			\includegraphics[width=\linewidth]{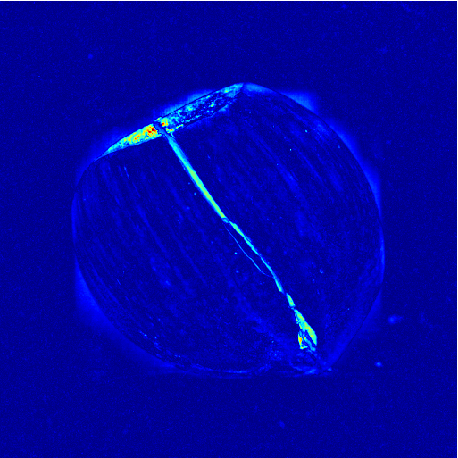}
			\includegraphics[width=\linewidth]{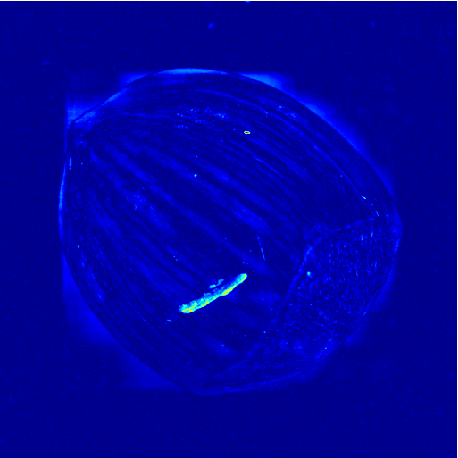}
			\includegraphics[width=\linewidth]{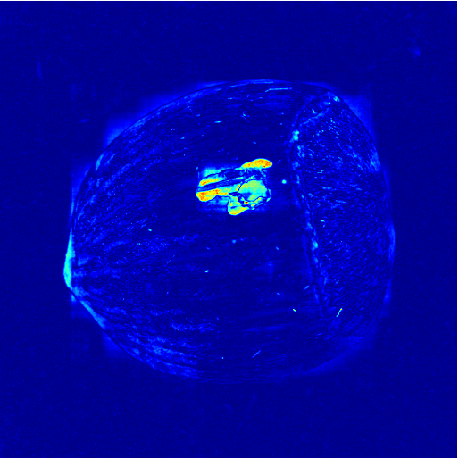}
			\includegraphics[width=\linewidth]{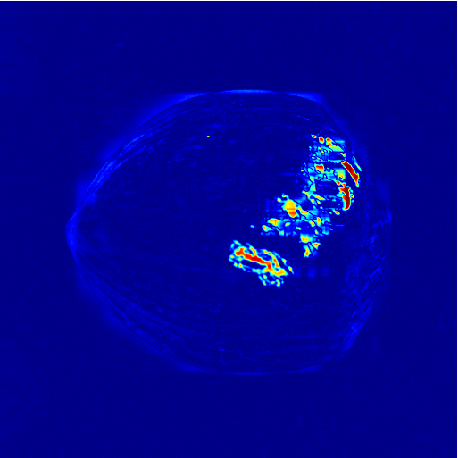}
			\caption*{TLRSR}
		\end{subfigure}
		\begin{subfigure}[b]{0.106\linewidth}
			\centering
			\includegraphics[width=\linewidth]{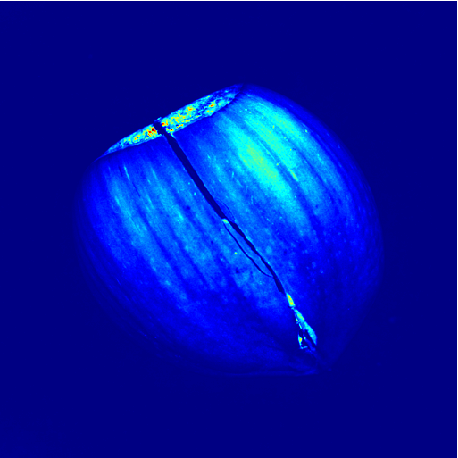}
			\includegraphics[width=\linewidth]{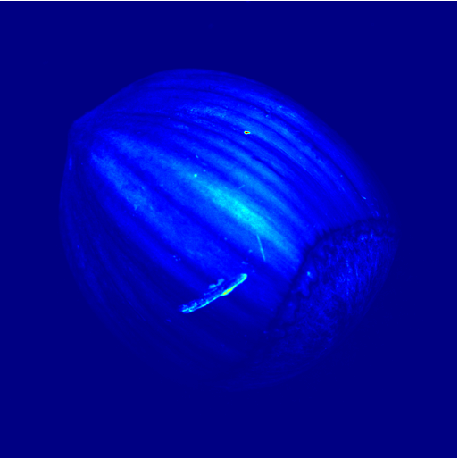}
			\includegraphics[width=\linewidth]{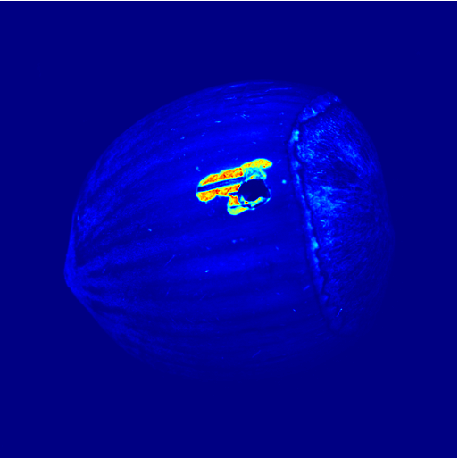}
			\includegraphics[width=\linewidth]{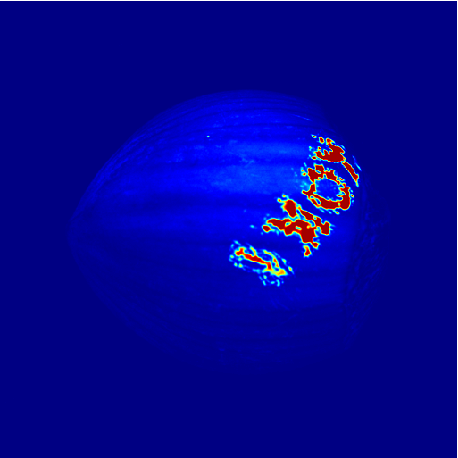}
			\caption*{RGAE}
		\end{subfigure}   	
		\begin{subfigure}[b]{0.106\linewidth}
			\centering
			\includegraphics[width=\linewidth]{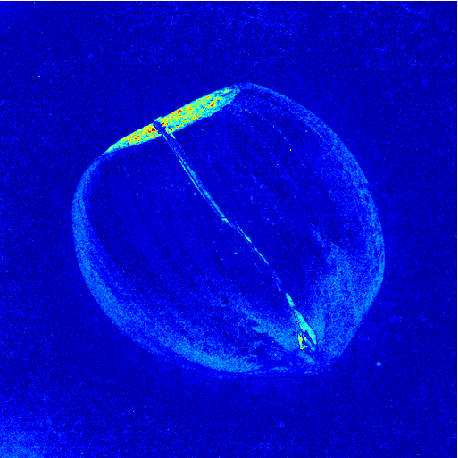}
			\includegraphics[width=\linewidth]{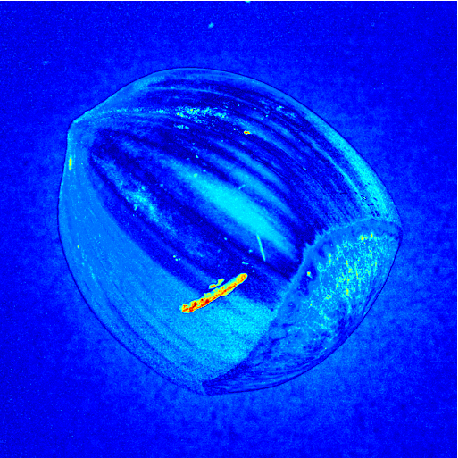}
			\includegraphics[width=\linewidth]{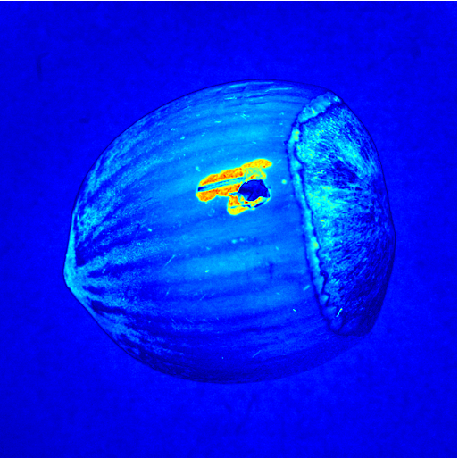}
			\includegraphics[width=\linewidth]{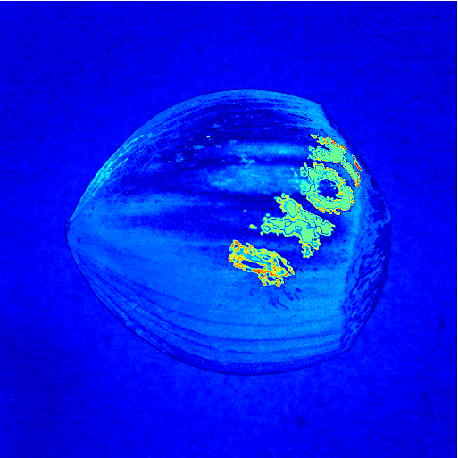}
			\caption*{GAED}
		\end{subfigure}
		\begin{subfigure}[b]{0.106\linewidth}
			\centering
			\includegraphics[width=\linewidth]{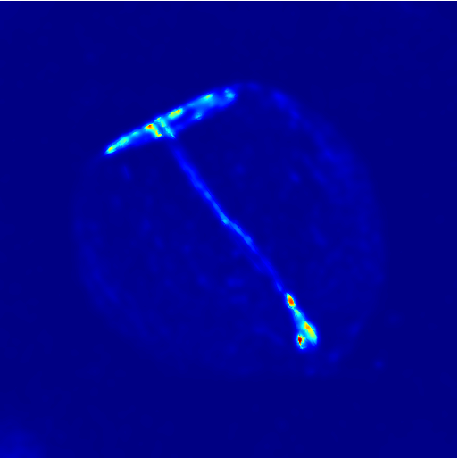}
			\includegraphics[width=\linewidth]{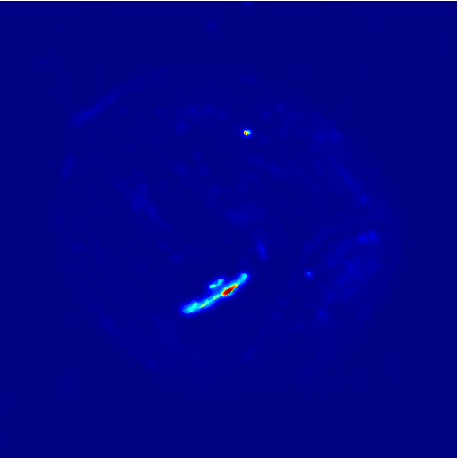}
			\includegraphics[width=\linewidth]{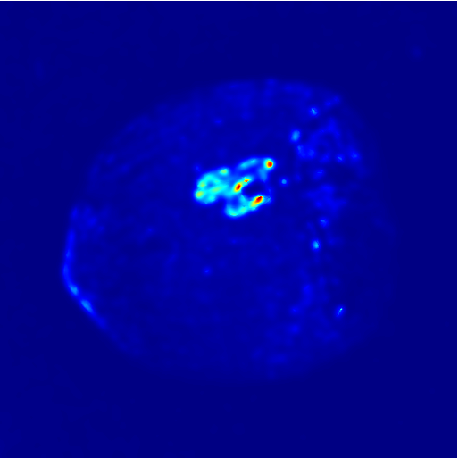}
			\includegraphics[width=\linewidth]{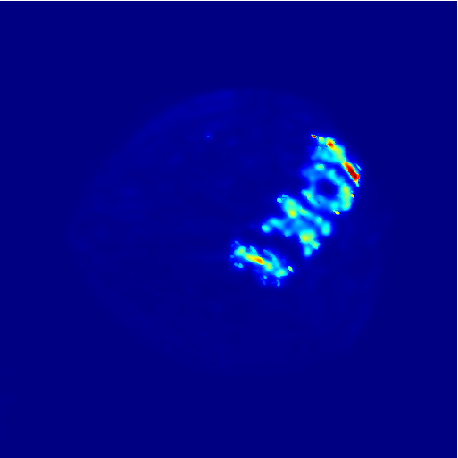}
			\caption*{LTD}
		\end{subfigure}
	\end{subfigure}
	\vfill
	\caption{Target detection results by different methods for MVTec dataset.}
	\label{fig:2D:MVTec}
\end{figure}

\begin{table}[htbp]
	\centering
	\caption{Comparison of AUC values (\%) and running time (s) of different methods for MVTec dataset.}\resizebox{\linewidth}{!}{
	\begin{tabular}{ccccccccccc}
		\toprule
		Algorithm & Index & RX    & RPCA  & LRASR & PTA   & TPCA  & TLRSR & RGAE  & GAED  & LTD \\
		\midrule
		\multirow{2}{*}{Crack} & AUC   & 77.23  & 71.12  & 29.98  & 45.66  & 65.75  & 83.34  & 82.03  & 78.37  & \textbf{98.26 } \\
		& Time & 0.05  & 0.40  & 977.33  & 16.04  & 14.33  & 29.41  & 818.99  & 669.92  & 4.99  \\
		\midrule
		\multirow{2}{*}{Cut} & AUC   & 89.46  & 91.50  & 33.48  & 72.30  & 56.40  & 91.11  & 88.54  & 86.02  & \textbf{99.76 } \\
		& Time & 0.05  & 0.40  & 991.03  & 16.81  & 13.96  & 32.88  & 837.77  & 662.67  & 4.75  \\
		\midrule
		\multirow{2}{*}{Hole} & AUC   & 91.26  & 88.50  & 22.93  & 78.77  & 82.01  & 95.75  & 93.88  & 90.34  & \textbf{99.25 } \\
		& Time & 0.05  & 0.40  & 979.06  & 16.26  & 14.62  & 46.29  & 814.96  & 677.04  & 2.62  \\
		\midrule
		\multirow{2}{*}{Print} & AUC   & 95.70  & 95.60  & 58.68  & 89.03  & 68.48  & 94.21  & 97.05  & 87.42  & \textbf{99.87 } \\
		& Time & 0.04  & 0.40  & 982.01  & 15.31  & 14.00  & 32.08  & 978.39  & 684.22  & 3.44  \\
		\bottomrule
	\end{tabular}}%
	\label{tab:MVTec}%
\end{table}%

Figure \ref{fig:ROC:MVTec} presents a detailed comparison of the ROC curves for each method. The detection probabilities of the proposed LTD method approach 1 when the false alarm rate is approximately $1e\text{-}1$. Notably, our LTD method outperforms others in detection probability when the false alarm rate exceeds $1e\text{-}2$. Figure \ref{fig:AB:MVTec} further compares the methods through box plots, confirming that our approach achieves the superior performance, aligning with the results in Figure \ref{fig:ROC:MVTec}.

The anomaly detection results from MVTec dataset are also consistent with those from ABU dataset and all these demonstrate that our method achieves superior anomaly detection accuracy and operates with enhanced efficiency.

\begin{figure}[htbp]
	\centering
	\begin{subfigure}[b]{1\linewidth}
		\begin{subfigure}[b]{1\linewidth}
			\centering
			\includegraphics[width=0.95\linewidth]{Legend}
		\end{subfigure} 
		\begin{subfigure}[b]{0.245\linewidth}
			\centering
			\includegraphics[width=\linewidth]{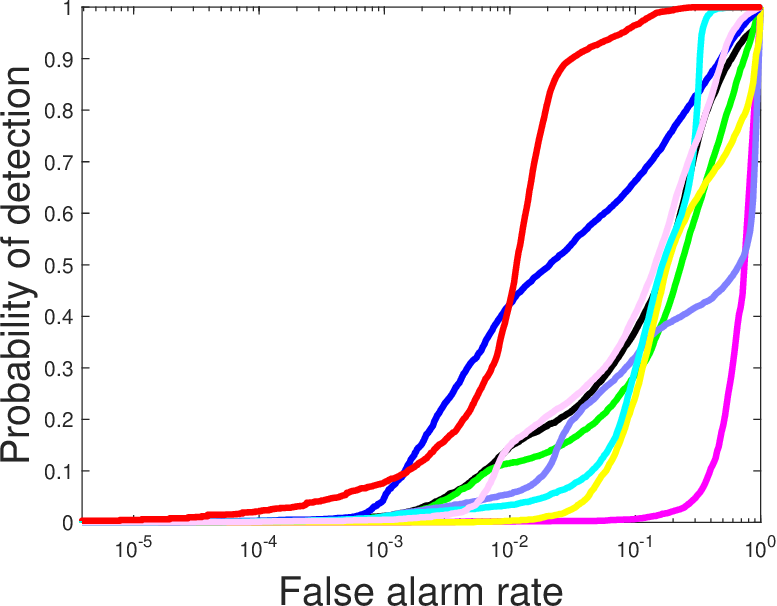}
			\caption{Crack}
		\end{subfigure}   	
		\begin{subfigure}[b]{0.245\linewidth}
			\centering
			\includegraphics[width=\linewidth]{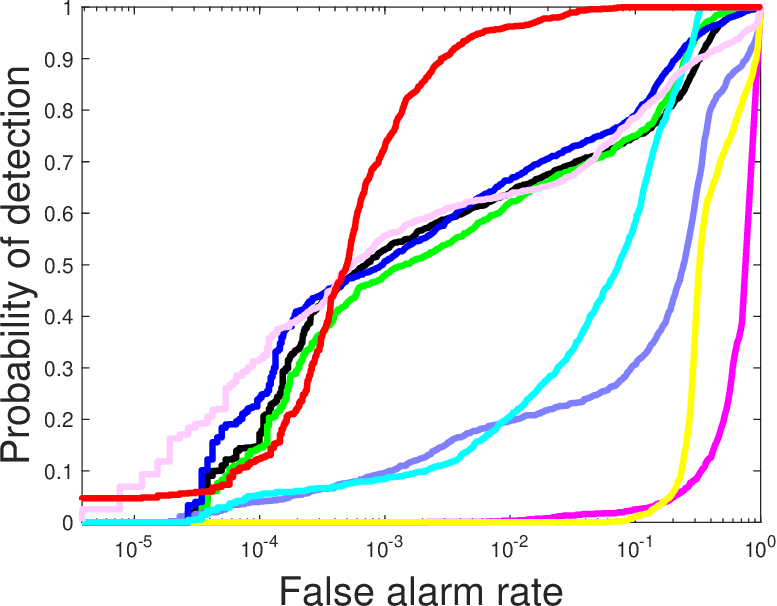}
			\caption{Cut}
		\end{subfigure}
		\begin{subfigure}[b]{0.245\linewidth}
			\centering
			\includegraphics[width=\linewidth]{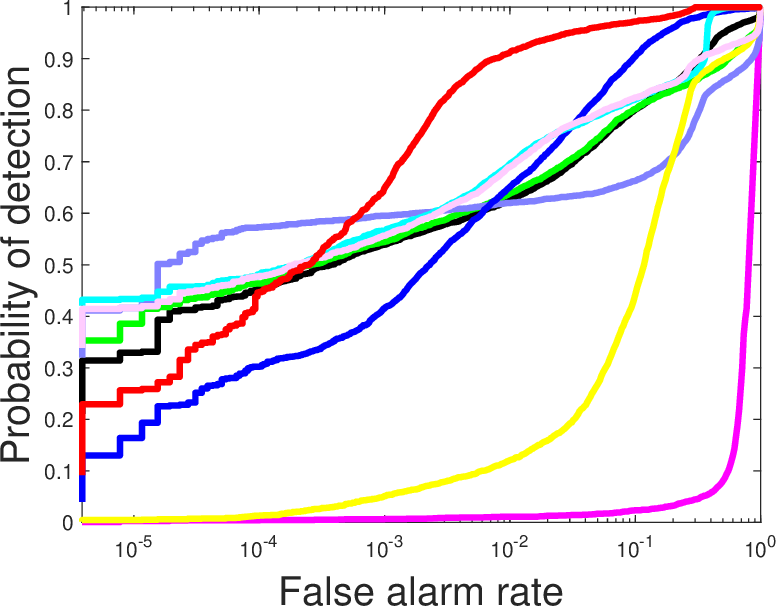}
			\caption{Hole}
		\end{subfigure}
		\begin{subfigure}[b]{0.245\linewidth}
			\centering
			\includegraphics[width=\linewidth]{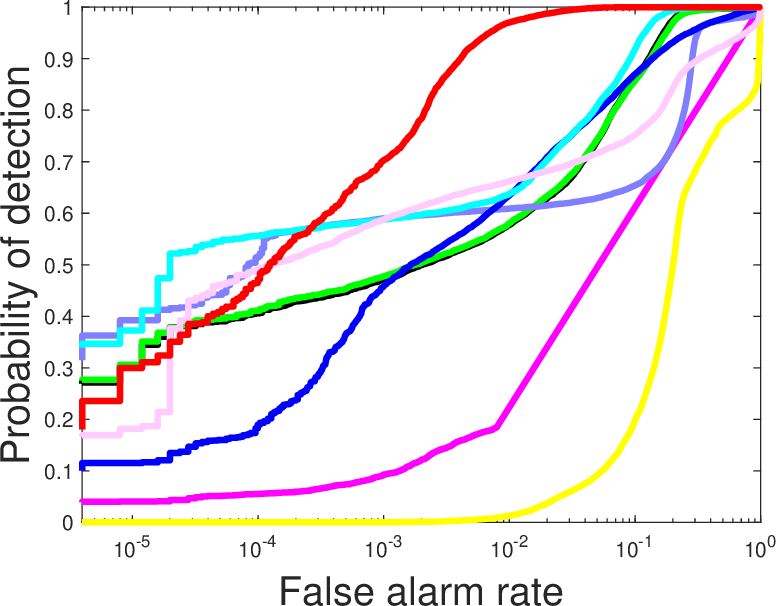}
			\caption{Print}
		\end{subfigure}
	\end{subfigure}
	\vfill
	\caption{ROC curves obtained by different methods for MVTec dataset.}
	\label{fig:ROC:MVTec}
\end{figure}

\begin{figure}[htbp]
	\centering
	\begin{minipage}{0.9\textwidth}
		\begin{subfigure}[b]{0.245\linewidth}
			\centering
			\includegraphics[width=\linewidth]{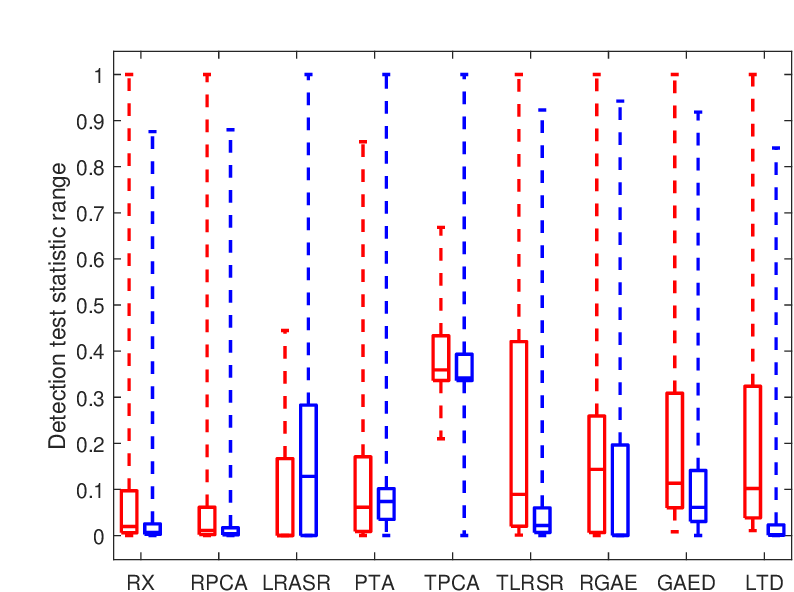}
			\caption{Crack}
		\end{subfigure}   	
		\begin{subfigure}[b]{0.245\linewidth}
			\centering
			\includegraphics[width=\linewidth]{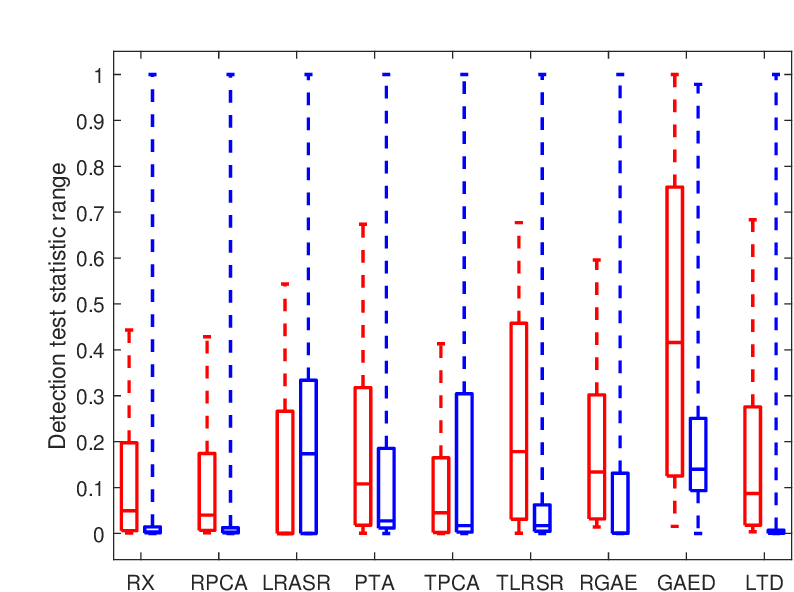}
			\caption{Cut}
		\end{subfigure}
		\begin{subfigure}[b]{0.245\linewidth}
			\centering
			\includegraphics[width=\linewidth]{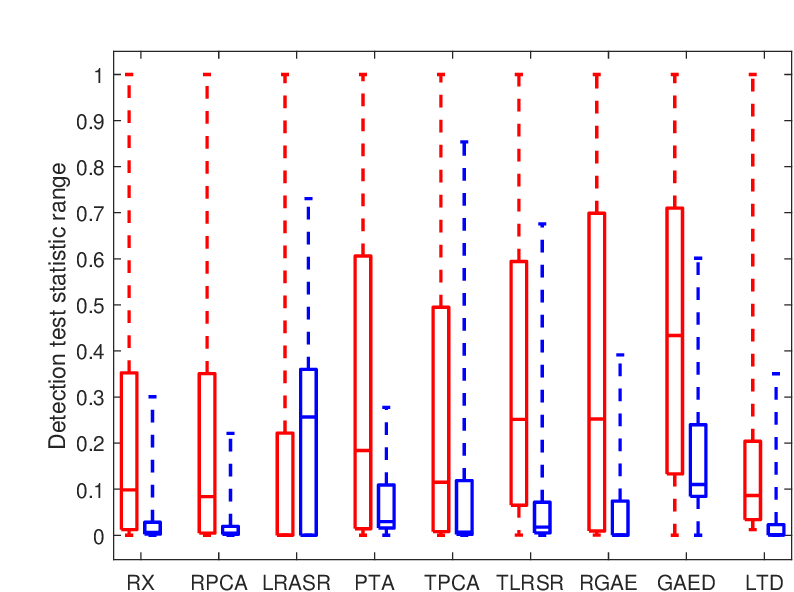}
			\caption{Hole}
		\end{subfigure}
		\begin{subfigure}[b]{0.245\linewidth}
			\centering
			\includegraphics[width=\linewidth]{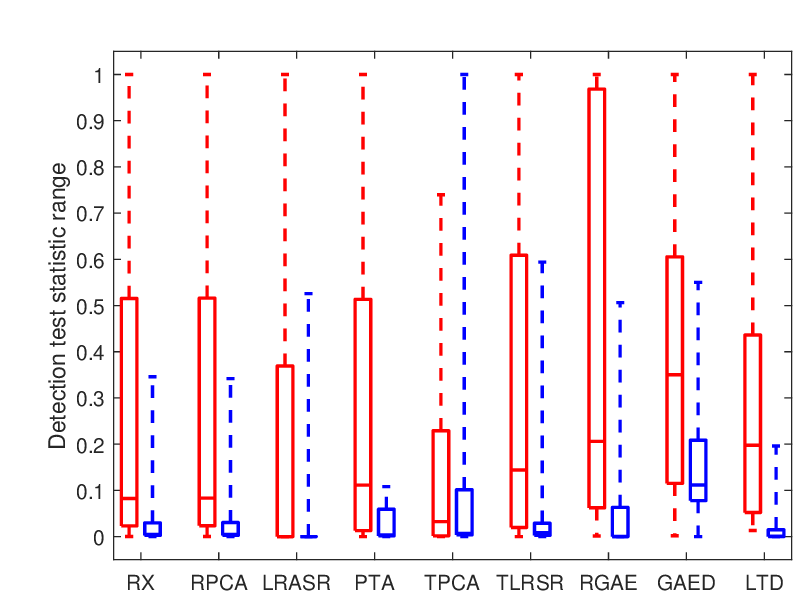}
			\caption{Print}
		\end{subfigure}
	\end{minipage}\hfill
	\begin{minipage}{0.1\textwidth}
		\centering
		\includegraphics[width=\linewidth]{Box}
	\end{minipage}
	\vfill
	\caption{Separability maps of different methods for MVTec dataset.}
	\label{fig:AB:MVTec}
\end{figure}

\section{Conclusion}
In this paper, we propose a novel LTD model for HAD. In our model, we first use NMF to reduce spectral dimension redundancy, which not only overcomes the drawbacks of undesirable statistical and geometrical properties but also enhances computational efficiency, while simultaneously generating spectral anomaly.
We then apply LRTR to capture spatial features, leading to the identification of spatial anomaly. A key innovation in LRTR is the substitution of the tensor tubal rank with group sparsity regularization, which iteratively reduces data size by eliminating groups of all-zero elements, further improving algorithmic efficiency. By integrating spectral anomaly from NMF with spatial anomaly from LRTR, we derive a spectral-spatial anomaly, leveraging the complementarity between the spectral and spatial domains to boost detection performance. Lastly, a guided image filter is applied to the spectral-spatial anomaly, yielding clearer target detection with reduced noise. We also introduce a PAM algorithm to solve the LTD model, proving that the iterative sequence converges to a critical point. Experimental results demonstrate that our method not only delivers superior performance but also operates faster than most state-of-the-art HAD methods.

\section*{Acknowledgments}
The authors thank the anonymous reviewers for their valuable suggestions. 

\section*{Funding}
 National Natural Science Foundation of China (11971159 and 12071399 to Minru Bai; 12021001 and 92473208 to Yu-Hong Dai); Hunan Provincial Key Laboratory of Intelligent Information Processing and Applied Mathematics (Minru Bai); Postgraduate Scientific Research Innovation Project of Hunan Province (CX20240363 to Quan Yu).


\bibliographystyle{abbrvnat}
\bibliography{reference}

\appendix
\numberwithin{lemma}{section}
\numberwithin{defi}{section}
\titleformat{\section}{\bfseries}{Appendix \thesection.}{0.2em}{} 

\section{Preliminaries of tensors}\label{app:pre}
In Appendix \ref{app:pre}, we present a comprehensive set of definitions for the tensor concepts that are utilized consistently throughout this paper.

\begin{defi}(f-diagonal tensor \cite{KM11}).
	If every frontal slice of a tensor constitutes a diagonal matrix, then the tensor is termed $f$-diagonal.
\end{defi}

\begin{defi}(conjugate transpose \cite{KM11}).\label{def:ct}
   The conjugate transpose of a tensor, denoted by $\mathcal{X}^T$, for $\mathcal{X} \in \mathbb{R}^{n_{1} \times n_{2} \times n_{3}}$, is the tensor obtained by taking the conjugate transpose of each frontal slice and then reversing the order of transposed frontal slices 2 through $n_{3}$.
\end{defi}

\begin{defi}(identity tensor \cite{KM11}). \label{def:it}
	The identity tensor, denoted by $ \mathcal{I} \in \mathbb{R}^{n \times n \times n_{3}} $, is the tensor whose first frontal slice is the identity matrix, with all other frontal slices being zeros.
\end{defi}

\begin{defi}(orthogonal tensor \cite{KM11}).
	A tensor $\mathcal{P} \in$ $\mathbb{R}^{n \times n \times n_{3}}$ is orthogonal if it fulfills the condition $\mathcal{P}^T * \mathcal{P}=\mathcal{P} * \mathcal{P}^T=\mathcal{I}.$
\end{defi}

\section{Proof of Theorem \ref{Thm:rank}}\label{app:thm_rank}
Initially, let us delineate the notations to be utilized.
\begin{itemize}
	\item The symbol $\mO$ denotes the tensor wherein all elements are zero;
	\item for the given tensors $\X \in \R^{n_1\times r_1 \times n_3}$ and $\Y \in \R^{n_1\times r_2 \times n_3}$, the tensor $\Z = [\X, \Y] \in \R^{n_1\times (r_1+r_2) \times n_3}$ is defined as the concatenation of $\X$ and $\Y$ along the lateral slices, i.e., $\Z(:,1:r_1,:) = \X$ and $\Z(:,r_1+1:r_1+r_2,:) = \Y$.
\end{itemize}
\begin{proof}
	 Let $\hat{r}:=\|\Z\|_{F,0} $. Without loss of generality, assume that $\Z$ can be partitioned as $\Z = [\Z_1, \mO]$, where $\Z_1 \in \R^{n_2 \times \hat{r} \times b}$ and $\mO \in \R^{n_2 \times (r-\hat{r}) \times b}$. Correspondingly, tensor $\D$ is divided into $\D = [\D_1, \D_2]$, with $\D_1 \in \R^{n_1 \times \hat{r} \times b}$ and $\D_2 \in \R^{n_1 \times (r-\hat{r}) \times b}$. Consequently, $\D * \Z^T = \D_1 * \Z_1^T$.
	Combining this with $\mL = \D * \Z^T$, one has
	$$
	\rank_t\left(\mL\right) = \rank_t\left(\D * \Z^T\right) =\rank_t\left(\D_1 * \Z_1^T\right)\le \rank_t\left( \Z_1\right) \le \hat{r} = \left\|\Z\right\|_{F,0},
	$$
	where the first inequality follows from \cite[Lemma 2]{ZLLZ18}. When $\D = \U(:,1:r,:)$ and $\Z = \V*\left(\mS(1:r,:,:)\right)^T$, equality is achieved. The tensors $\U$, $\mS$, and $\V$ are derived from the T-SVD of $\mL$, represented as $\mL = \U*\mS*\V^T$.
\end{proof}

\section{Proof of Theorem \ref{Thm:cap}}\label{app:thm_cap}
Before proving the theorem, we introduce some notations. 
We denote
$$
\Omega=\left\lbrace \Z \in \R^{n_2\times r \times b} \mid \exists\; \D\; \mbox{\rm s.t.}\; \mL = \D * \Z^T, \D^T*\D=\I \right\rbrace.
$$
We also denote the group support set of $\Z$ as
\[ \Gamma(\Z):=\left\lbrace j \mid\left\|\Z(:,j,:)\right\| \neq 0, \, j=1, \ldots, r\right\rbrace=\Gamma_{1}(\Z) \cup \Gamma_{2}(\Z), \]
\[ \Gamma_{1}\left(\Z\right):=\left\lbrace j \mid \left\|\Z(:,j,:)\right\|<\nu, j \in \Gamma\left(\Z\right)\right\rbrace, \, \Gamma_{2}\left(\Z\right):=\left\lbrace j \mid \left\|\Z(:,j,:)\right\| \geq \nu, j \in \Gamma\left(\Z\right)\right\rbrace. \]
For an integer $s$ with $0 \leq s \leq r$, denote $
Q^s=\left\{\Z \in \mathbb{R}^{n_2\times r \times b}\mid\|\Z\|_{F,0} \leq s\right\}$ and $\operatorname{dist}\left(\Omega, Q^s\right)=\inf_{\Z \in \Omega} \operatorname{dist}\left(\Z, Q^s\right)$. Here $\operatorname{dist}\left(\Z, Q^s\right)=\inf_{\tilde{\Z} \in Q^s}\|\Z-\tilde{\Z}\|$. 

For simplicity, we denote
$$
\left\{\begin{array}{l}
	(P_0) ~ \min \left\{\left\|\Z\right\|_{F,0}: \mL=\D*\Z^T, \D^T*\D=\I\right\}; \\
	(P_\psi)~  \min \left\{\left\|\Z\right\|_{F,1}^\psi: \mL=\D*\Z^T, \D^T*\D=\I\right\}.
\end{array}\right.
$$
Obviously, the optimal value of $(P_0)$ is a positive integer, which we denote as $N$. Then for any $\Z\in \Omega$, one has $\left\|\Z\right\|_{F,0} \ge N$, which implies that $ \operatorname{dist}(\Omega, Q^{N-n})>0 $ for all $n=1,2,\ldots,N$. Up to this point, we can construct $\bar{\nu}$ as follows:
$$
\bar{\nu}=\min\left\lbrace\frac{1}{n} \operatorname{dist}\left(\Omega, Q^{N-n}\right): n=1, \ldots, N\right\rbrace.
$$
\begin{proof}
	$(i)$ Let $(\D^*, \Z^*)$ be a global minimizer of problem $(P_0)$ with 
	$ \left\| \Z^* \right\|_{F,0}=N $. We demonstrate that $(\D^*, \Z^*)$ is also a global minimizer of problem $(P_\psi)$ for any $0<\nu<\bar{\nu}$. Given that the global optimality of $(P_0)$ implies $ \left\| \Z \right\|_{F,0}\ge N $ for $\Z \in \Omega$, we establish this conclusion by considering two cases.
	
	Case 1. $\|\Z\|_{F,0}=N$. It is easy to see that for any $j \in \Gamma(\Z)$,
	\begin{equation*}
		\begin{aligned}
			\left\|\Z(:,j,:)\right\| &\geq \min \left\lbrace \left\|\Z(:,l,:)\right\|>0: l=1, \ldots, r\right\rbrace\\
			&=\operatorname{dist}\left(\Z, Q^{N-1}\right) \geq \operatorname{dist}\left(\Omega, Q^{N-1}\right) \geq \bar{\nu}>\nu,	
		\end{aligned}
	\end{equation*}
	which means that $\|\Z\|_{F,1}^\psi=N=\|\Z^*\|_{F,1}^\psi$.	
	
	Case 2. $\|\Z\|_{F,0}=M>N$. Without loss of generality, assume $\|\Z(:,1,:)\|, \ldots,\|\Z(:,j_1,:)\| \in(0, \nu)$, $\|\Z(:,j_1+1,:)\|, \ldots,\|\Z(:,j_2,:)\| \in[\nu, +\infty)$ and $\left\|\Z(:,j_2+1,:)\right\|=\cdots=\left\|\Z(:,r,:)\right\|=0$. If $M^{\prime}=\Gamma_{2}(\Z)\ge N$, from $\phi(x)>0$ for $x>0$, one has $\|\Z\|_{F,1}^\psi>N$. Now assume $M^{\prime}<N$, we know that 
	$$\frac{1}{N-M^{\prime}} \operatorname{dist}\left(\Omega, Q^{M^{\prime}}\right) \geq \bar{\nu}.$$ 
	Together with
	\begin{equation*}
		\begin{aligned}
		\left\|\Z(:,1,:)\right\|+\cdots+\left\|\Z(:,j_1,:)\right\|
			&\geq \sqrt{\left\|\Z(:,1,:)\right\|^2+\cdots+\left\|\Z(:,j_1,:)\right\|^2} \\
			&\geq \operatorname{dist}\left(\Z, Q^{M^{\prime}}\right) \geq \operatorname{dist}\left(\Omega, Q^{M^{\prime}}\right),	
		\end{aligned}
	\end{equation*}
	we get 
	\begin{equation}
		\begin{aligned}
			\|\Z\|_{F,1}^\psi &= \psi\left(\left\|\Z(:,1,:)\right\|\right)+\cdots+\psi\left(\left\|\Z(:,j_1,:)\right\|\right)+\cdots+\psi\left(\left\|\Z(:,j_2,:)\right\|\right) \\
			&\geq\frac{1}{\nu}\left(\left\|\Z(:,1,:)\right\| +\cdots+\left\|\Z(:,j_1,:)\right\|\right)+M^{\prime} \\
			& \geq \frac{1}{\nu} \operatorname{dist}\left(\Omega, Q^{M^{\prime}}\right)+M^{\prime}  \\ &\geq \frac{1}{\nu}\left(N-M^{\prime}\right) \bar{\nu}+M^{\prime} \\
			&>\frac{1}{\bar{\nu}}\left(N-M^{\prime}\right) \bar{\nu}+M^{\prime}=N.
		\end{aligned}
	\end{equation}
	The aforementioned two cases indicate that $\|\Z\|_{F,1}^\psi \geq N = \|\Z^*\|_{F,1}^\psi$ for all $\Z\in\Omega$. Hence, $(\D^*, \Z^*)$ is also a global minimizer of problem $(P_\psi)$. Furthermore, for each global minimizer $(\D^*, \Z^*)$ of problem $(P_\psi)$, we have $\|\Z^*\|_{F,0}=\|\Z^*\|_{F,1}^\psi $.
	
	$(ii)$ Let $(\D^\#, \Z^\#)$ be a global minimizer of problem $(P_\psi)$ for $0<\nu<\bar{\nu}$. Assume on the contrary $(\D^\#, \Z^\#)$ is not a solution of problem $(P_0)$. Let $(\D^*, \Z^*)$ be a global minimizer of problem $(P_0)$, that is, $ \|\Z^*\|_{F,0}=N $. By $\psi(x) \leq|x|^{0}$, we have $\|\Z^*\|_{F,1}^\psi\le \left\|\Z^*\right\|_{F,0}$. Using similar ways in the proof for Case 2 above, we will obtain $\|\Z^\#\|_{F,1}^\psi>N=\| \Z^*\|_{F,0} \geq \|\Z^*\|_{F,1}^\psi$ for any $0<\nu<\bar{\nu}$. 
	This contradicts the global optimality of $(\D^\#, \Z^\#)$ for problem $(P_\psi)$. Hence $(\D^\#, \Z^\#)$ is a global minimizer of problem $(P_0)$.
	
	Therefore, whenever $0<\nu<\bar{\nu}$, problems $(P_0)$ and $(P_\psi)$ have the same global minimizers and optimal values.
\end{proof}

\section{Proof of Lemma \ref{lem:or}}\label{app:lem_or}

\begin{lemma}\cite{KM11}\label{lem:equ}
	Suppose that $\X \in \mathbb{R}^{n_1 \times r \times n_3}$ and $\Y \in \mathbb{R}^{r \times n_2 \times n_3}$ are two arbitrary tensors. Let $\Z=\X * \Y$. Then, $\Z=\X * \Y$ and ${\bar Z}={\bar X} {\bar Y}$ are equivalent to each other.
\end{lemma}

\begin{proof}
	Given that $F_{b} \in \R^{b\times b}$ is the Discrete Fourier Transform matrix, and from the property $ F_{b}^TF_{b} = bI $, we have
	$$ \left\langle \D,\G\right\rangle = \frac{1}{b}\left\langle \D\times_3\left(F_{b}^TF_{b}\right),\G\right\rangle =\frac{1}{b}\left\langle \D\times_3F_{b},  \G\times_3F_{b}\right\rangle = \frac{1}{b}\left\langle \bar{D},  \bar{G}\right\rangle = \frac{1}{b}\sum_{k=1}^{b}\left\langle \bar{D}^{(k)},  \bar{G}^{(k)}\right\rangle. $$
	We obtain from the above equality that
	\begin{equation*}
		\begin{aligned}
			\mathop{\arg\max}\limits_{\D^T*\D = \I} ~\left\langle \D,\G\right\rangle = \mathop{\arg\max}\limits_{\D^T*\D = \I} ~\sum_{k=1}^{b}\left\langle \bar{D}^{(k)},  \bar{G}^{(k)}\right\rangle = \mathop{\arg\max}\limits_{{\bar D}^{(k)^T}{\bar D}^{(k)} = I} ~\sum_{k=1}^{b}\left\langle \bar{D}^{(k)},  \bar{G}^{(k)}\right\rangle,
		\end{aligned}
	\end{equation*}
	where the last equality follows from Lemma \ref{lem:equ}. By orthogonal Procrustes problem  \cite{Sch66}, we obtain
	$$\mathop{\arg\max}\limits_{{\bar D}^{(k)^T}{\bar D}^{(k)} = I} ~\left\langle \bar{D}^{(k)},  \bar{G}^{(k)}\right\rangle = {\bar U}^{(k)}{\bar V}^{(k)^T},$$
	where ${\bar U}^{(k)}$ and $ {\bar V}^{(k)} $ are obtained by the singular value decomposition: $\bar{G}^{(k)} = {\bar U}^{(k)}{\bar S}^{(k)}{\bar V}^{(k)^T} $, which completes the proof of this statement.
\end{proof}

\end{document}